\documentclass[nonacm]{jair}
\setcopyright{none}
\renewcommand\footnotetextcopyrightpermission[1]{}

\usepackage{subcaption}
\usepackage{graphicx}
\usepackage{url}
\usepackage{amsmath}
\usepackage{booktabs}
\usepackage{multirow}
\usepackage{algorithmic}
\usepackage{algorithm}
\usepackage[switch]{lineno}
\usepackage{verbatim}
\usepackage[capitalize,noabbrev]{cleveref}

\usepackage[dvipsnames]{xcolor}
\usepackage{natbib} 
\newcommand{\rev}[1]{{#1}}

\newcommand{\Z}{\Tilde{Z}}
\newcommand{\ZH}{\hat{Z}}

\newcommand{\algoName}{Robust Interval Search - Non-Sticky}
\newcommand{\shortAlgoName}{RIS-NS}

\newcommand{\algoNameSticky}{Robust Interval Search - Sticky}
\newcommand{\shortAlgoNameSticky}{RIS-S}

\newcommand{\modifiedBisection}{modified bisection search}
\newcommand{\modifiedBisectionShort}{MBS}

\newcommand{\twoStep}{{TZ}}
\newcommand{\parameterized}{P}

\newcommand{\problemName}{policy observation robustness}
\newcommand{\ProblemName}{Policy Observation Robustness}

\newcommand{\MC}{\mathcal{C}}

\usepackage{amsthm}

\newtheoremstyle{main}
                {1em}                                                
                {1em}                                              
                {\normalfont}                                        
                {0pt}                                                
                {\scshape}                                           
                {\\*}                                                
                {2pt}                                                
                {\thmname{#1}\thmnumber{ #2}: \thmnote{\itshape #3}} 

\newtheorem{example}{Example}
\newtheorem{remark}{Remark}

\newtheorem{lemma}{Lemma}
\newtheorem{definition}{Definition}

\newtheorem{problem}{Problem} 

\usepackage[normalem]{ulem} 

\usepackage{thmtools}
\usepackage{thm-restate}

\newcommand{\threshold}{\Delta}
\newcommand{\prball}[1]{\ensuremath{\mathcal{B}^P_\infty \hspace{-0.4ex} \left(#1 \right)}}

\renewcommand{\algorithmiccomment}[1]{\bgroup\hfill//~#1\egroup}

\usepackage{array}
\newcolumntype{H}{>{\setbox0=\hbox\bgroup}c<{\egroup}@{}}

\acmVolume{}
\acmArticle{}
\acmMonth{}
\acmYear{}

\begin{document}
\pagestyle{plain}

\title[Observation Robustness]{Robustness Analysis of POMDP Policies to Observation Perturbations}

\author{Benjamin Kraske}
\authornote{Corresponding Author.}
\email{Benjamin.Kraske@colorado.edu}
\orcid{https://orcid.org/0009-0000-9618-1392}
\affiliation{%
  \institution{University of Colorado Boulder}
  \city{Boulder}
  \state{Colorado}
  \country{USA}
}

\author{Qi Heng Ho}
\orcid{https://orcid.org/0009-0006-0743-9341}
\email{qihengho@vt.edu}
\affiliation{%
  \institution{Virginia Tech}
  \city{Blacksburg}
  \state{Virginia}
  \country{USA}
}

\author{Federico Rossi}
\orcid{https://orcid.org/0000-0002-8091-881X}
\email{federico.rossi@jpl.nasa.gov}
\affiliation{%
  \institution{Jet Propulsion Laboratory, California Institute of Technology}
  \city{Pasadena}
  \state{California}
  \country{USA}
}

\author{Morteza Lahijanian}
\orcid{https://orcid.org/0000-0001-7549-4365}
\email{Morteza.Lahijanian@colorado.edu}
\affiliation{%
  \institution{University of Colorado Boulder}
  \city{Boulder}
  \state{Colorado}
  \country{USA}
}

\author{Zachary Sunberg}
\orcid{https://orcid.org/0000-0001-9707-3035}
\email{Zachary.Sunberg@colorado.edu}
\affiliation{%
  \institution{University of Colorado Boulder}
  \city{Boulder}
  \state{Colorado}
  \country{USA}
}


\renewcommand{\shortauthors}{Kraske et al.}

\begin{abstract}
    
    {\bf Background:} Policies for Partially Observable Markov Decision Processes (POMDPs) are often designed using a nominal system model. In practice, this model can deviate from the true system during deployment due to factors such as calibration drift or sensor degradation, leading to unexpected performance degradation. This work studies policy robustness against deviations in the POMDP observation model.

    {\bf Objectives:} We introduce the \ProblemName{} Problem: to determine the maximum tolerable deviation in a POMDP's observation model that guarantees the policy's value remains above a specified threshold. We analyze two variants: the sticky \rev{variant}, where deviations are dependent on state and actions, and the non-sticky \rev{variant}, where they can be history-dependent.
    
    {\bf Methods:} We show that the \ProblemName{} Problem can be formulated as a bi-level optimization problem in which the inner optimization is monotonic in the size of the observation deviation. This enables efficient solutions using root-finding algorithms in the outer optimization. For the non-sticky \rev{variant}, we show that when policies are represented with finite-state controllers (FSCs) it is sufficient to consider observations which depend on nodes \rev{in the FSC} rather than full histories. This makes the problem tractable in the non-sticky \rev{variant}. We present Robust Interval Search, an algorithm with soundness and convergence guarantees, for both the sticky and non-sticky \rev{variants}. We show this algorithm has polynomial time complexity in the non-sticky \rev{variant} and at most exponential time complexity in the sticky \rev{variant}.
    
    {\bf Results:} We provide experimental results for implementations of Robust Interval Search. These validate the ability of our algorithms to find the robustness of POMDP policies to observation model deviations. We show the algorithms can scale to \rev{POMDP} problem \rev{instances} with tens of thousands of states. Case studies from robotics and operations research demonstrate its practical utility in providing robustness analyses to model and policy designers.
    
    {\bf Conclusions:} We introduce, analyze, and provide solution methods for the \problemName{} problem. We demonstrate the scalability and relevance of these methods empirically.

\end{abstract}




\maketitle

\section{Introduction}

\rev{Policies for autonomous systems are often designed using approximate models. Consider NASA's Volatiles Investigating
Polar Exploration Rover (VIPER) \citep{colaprete2019overview,HELDMANN2016308}. Here, methods for autonomous mission planning have been proposed by \citet{VIPERMDP}. In such cases, a system must navigate autonomously using sensors whose accuracy may degrade over time or differ from pre-deployment model assumptions, including assumptions of perfect observability. A policy that performs well on the approximate model or a model assumed to be fully observable may fail at deployment if these approximations or assumptions are inaccurate. This motivates a need to verify policy performance across the space of potential models a system may encounter at deployment.}

\rev{We address this challenge within the framework of Partially observable Markov decision processes (POMDPs), which provide a principled approach to sequential decision-making under stochastic transitions and partial state observability. Applications include robotics \citep{9899480}, \rev{autonomous navigation} \citep{10.5555/3535850.3535914}, human-computer interaction \citep{10.1145/3359616}, healthcare \citep{ayer_or_2012}, and logistics \citep{ICAA}. In these applications, ensuring that a policy remains robust is a prerequisite for real-world autonomy.}

\rev{POMDP policies are often obtained through approximate means, including heuristics and direct hand-specification, treating the system as a fully observable Markov decision process despite partial observability, or using approximate models with data-fitted or manually tuned parameters.
However, these approximations can come at a cost. MDP policies in particular tend to degrade in performance when directly deployed in settings which are, in reality, only partially observable through noisy observations \citep{sharma2024risk}. Additionally, even if a precise model is available, systems are prone to change over time, such as camera calibration degradation or temperature sensor damage. Determining the sensitivity of a policy to such model error or change allows for more reliable deployment in real world applications. While many works focus on finding policies robust to model error, known as \emph{robust policy synthesis} \citep{white_parameter_1986,givan_bounded-parameter_2000,cubuktepe_robust_2021,galesloot_pessimistic_2024,robustpolsynth}, it is not always the case that model error is well-quantified or accounted for during policy design. Indeed, directly estimating the model inaccuracy may not be feasible. In such cases, a designer may want to understand how much model inaccuracy can be tolerated by a given policy before performance degrades unacceptably. Consider the Airborne Collision Avoidance System Xu (ACAS Xu) \citep{9081758}, where a policy is found via an approximate observation function and the QMDP framework of certainty equivalence. Such a policy must be validated across a range of observation models representative of real-world conditions before deployment. This motivates a convenient means of completing such validation with formal guarantees.}

\begin{figure}
    \includegraphics[width=0.8\linewidth]{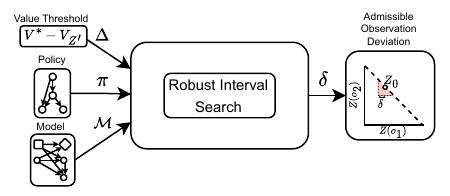}
    \caption{Solution approach to the \ProblemName{} Problem.}
    \Description{A diagram showing the solution approach to the Policy Observation Robustness Problem. It consists of inputs and outputs to Robust Interval Search. The inputs are: a value threshold, a policy, and a model. The output is an admissible observation deviation, visualized as the space of observation probabilities in a two-dimensional probability simplex.}
    \label{fig:overall}
\end{figure}

\rev{In this work, we introduce the \ProblemName{} Problem: find the maximum admissible deviation in a partially observable Markov decision process (POMDP) observation model that guarantees a policy's value remains above a specified threshold (see \cref{fig:overall}). This problem is complementary to the robust policy synthesis -- rather than finding a policy robust to a known uncertainty set, we quantify the uncertainty a given policy can tolerate. Solving this problem is crucial for quantifying the robustness of a policy to errors or misspecifications in the underlying observation model. This applies to any policy -- including heuristic policies, MDP policies, or robust policies -- providing insight into performance beyond the nominal or interval models for which they were designed.}

\rev{We show that this problem can be formulated as a bi-level optimization problem over Markov chains (MCs) with uncertain parameters. We analyze two variants of our problem which borrow notions of robustness established by \citet{bovy_imprecise_2024}:
\begin{enumerate}
      \item The \textit{sticky} variant where observation deviations 
    depend only on state and action, and are fixed for all timesteps, corresponding to a static POMDP model which minimizes the value of the policy. This can be thought of as an adversarial player, often referred to as nature \citep{bovy_imprecise_2024}, selecting a single POMDP model;
    \item The \textit{non-sticky} variant, where deviations may depend on history, in addition to the state and action, and therefore vary across timesteps. This captures adversarial perturbations that exploit the history dependence of the policy. This can be thought of as dynamic, time-varying POMDP model selection by nature \citep{bovy_imprecise_2024}.
\end{enumerate}}

\rev{Our key structural insight is that the worst-case policy value is monotonic in the size of the observation deviation, enabling efficient and sound bisection search over a single scalar $\delta$ using root-finding algorithms. For the non-sticky variant, we further show that it is sufficient to consider deviations depending on finite state controller nodes rather than the full history, reducing a potentially infinite dimensional problem to a tractable interval Markov chain evaluation. }

Building on these theoretical results, we present Robust Interval Search, a novel algorithm with convergence and soundness guarantees. The algorithm performs an efficient search for the maximum deviation in the outer optimization, while the solving an inner optimization which determines whether all models within the set defined by the deviation satisfy the value threshold. This approach yields a polynomial-time solution for the non-sticky variant and a solution with at most exponential time complexity for the more challenging sticky variant. Finally, we provide empirical validation of the algorithm implementations which use efficient near-optimal approximations for value calculations, demonstrate their scalability to \rev{POMDP} problems with tens of thousands of states, and provide relevant case studies from robotics, medicine, and operations research.

\rev{Our contributions are the following:}
\begin{enumerate}
    \item \rev{We introduce and solve a new problem, the \ProblemName{} problem for both the sticky and non-sticky variants.}

    \item \rev{We show that the structure of these problems enables bi-level optimization, which we leverage in algorithms that have soundness and convergence guarantees when exact value calculations are used. We also provide complexity analysis of these algorithms.}
    
    \item \rev{For the non-sticky variant, we prove that it is sufficient to consider deviations which depend on nodes in a finite state controller rather than histories, reducing the potentially infinite dimensionality of the problem and enabling solutions to infinite horizon problems. We also show that worst-case value under a deviation in this case can be calculated using a novel Markov chain structure, a \textit{two-step Interval Markov chain}, in polynomial time using existing methods.}
    
    \item \rev{We provide empirical results validating and demonstrating the scalability of implementations of the algorithms with approximate value calculations to POMDP problems with tens of thousands of states. We also provide case studies demonstrating applications of the problem and algorithms.}
\end{enumerate}
\subsection{Related Work}
\noindent \textit{MDP and POMDP Solutions}

\rev{Markov decision processes (MDPs) provide a principled framework for sequential decision-making problems with uncertainty due to stochastic state transitions.} POMDPs extend this by incorporating state uncertainty and partial observability. While finding optimal POMDP policies is NP-hard in the finite-horizon case \citep{c785ceef-c26f-38c4-b878-e08088df3cc3} and undecidable in the worst case \citep{MADANI20035}, various solution methods exist to find near-optimal solutions to POMDPs. These are generally classified into offline \citep{smith2004heuristic, kurniawati_sarsop_2008} and online \citep{silver_monte-carlo_2010,ye_despot_2017} methods. Additionally, a significant subset of solution methods generate Finite-State Controller (FSC) policies with desirable properties, such as satisfying reach-avoid specifications, near-optimality, etc. \cite{meuleau_solving_2013,junges_finite-state_nodate,chatterjee_qualitative_2015,ijcai2025p958,Bai2011}. Likewise, several works are concerned with the verification and synthesis of POMDP policies to meet given specifications \citep{norman2017verification,bork_verification_2020}.
In this work, we focus on analysis of the robustness of policies represented as FSCs to changes in the POMDP observation model.

\noindent \textit{Robust POMDP Solutions}

In many scenarios, there is uncertainty over (PO)MDP model parameters. A substantial literature addresses computing policies which are robust to uncertainty over model parameters \cite{white_parameter_1986,givan_bounded-parameter_2000,cubuktepe_robust_2021,galesloot_pessimistic_2024,robustpolsynth}. \citet{givan_bounded-parameter_2000} introduce Bounded-Parameter MDPs, also referred to as Interval MDPs (iMDPs), where uncertainty in transition and reward functions are specified with intervals. They introduce Interval Policy Evaluation (IPE) and Interval Value Iteration as solution methods. More recently, \citet{cubuktepe_robust_2021} presented a method to compute robust finite-memory policies for uncertain POMDPs while \citet{galesloot_pessimistic_2024} developed a technique to find policies robust to intervals over transitions, which can be generalized to observations, through iterative selection of pessimistic models and FSC solutions to these models. 

These "forward" robust planning methods assume the uncertainty parameter intervals or sets are known a-priori. However, in some applications, designers often lack precise intervals over model parameters and may assume or estimate nominal values that are inherently inaccurate. For example, characterizing precisely the calibration of a perception system or the degree to which a state may be observed is nontrivial. In such situations, a designer computes a policy for a nominal (PO)MDP via a solver or by hand-specifying a policy to ensure interpretability or simplicity. Given such policies, it is desirable to have a notion of the model accuracy required to achieve some value or performance threshold. This paper addresses an "inverse problem", the \ProblemName{} Problem: given a policy and a performance threshold, we determine the maximum allowable deviation for the observation function that guarantees the threshold is met.

\citet{konsta_what_2024} solve a problem involving both model design and policy synthesis: converting an MDP to a POMDP by designing an observation function such that a strategy can be found that meets some minimum value requirement while respecting a budget on the number of available observations. In contrast, our work finds admissible bounds on the deviation from a nominal observation function in an existing POMDP model given a fixed policy.

\noindent \textit{Robustness Analysis}

A related set of works analyze the robustness of MDP and POMDP policies to modeling error \cite{kara_robustness_2021, fard_variance_2008,mastin_loss_2012,tan_sensitivity_2011,doi:10.1137/100808976,demirci2025sensitivity}. \citet{ou2025sequential} survey the robust MDP literature in particular, discussing formulations of uncertainty and their tradeoffs.
\citet{mastin_loss_2012} provide an analysis on the loss in reward given a bound on the error of estimated transition probabilities in MDPs.
\citet{tan_sensitivity_2011} analyze the range of reward parameters over which an MDP policy remains optimal. 
\citet{fard_variance_2008} develop a method to approximate the bias and variance introduced into the value of a POMDP policy when this value is calculated using empirical POMDP models.
\rev{\citet{doi:10.1137/100808976} study the optimization of observation models, leveraging results on the continuity of optimal cost under different notions of convergence and the compactness of classes of observation models, with application to quantization of models.} 
\citet{kara_robustness_2021} establish convergence results and error bounds on the value of policies derived from approximate transition models, with application to learned MDP and POMDP models.
\rev{\citet{demirci2025sensitivity} provide computable bounds on the suboptimality of policies derived from  approximate transition and observation models. These bounds are expressed in terms of distances between the true and approximate transition and observation kernels. They derive analogous bounds for policies obtained from finite quantizations of the state and measurement spaces.}
In general, these works focus on the analysis of policy performance as approximate models converge to the true model or provide value bounds for a known approximate model. In contrast, our work is focused on finding the maximum allowable deviation in observation model under which a policy meets some performance threshold. This can be thought of as finding the space of models, as defined by the deviation, over which the policy is robust.

Of note are works which discuss notions of robustness for MDPs and POMDPs. \rev{\citet{iyengar2005robust,nilim2005robust} analyze robust dynamic programming over uncertainty sets and provide results for different notions of uncertainty. \citet{iyengar2005robust} introduce two notions of uncertainty, static and dynamic, for robust MDPs.}
\citet{bovy_imprecise_2024} extend these notions to POMDP policy robustness, highlighting how differing worst-case assumptions (e.g., stationary vs. dynamic adversaries) impact resulting policies in the POMDP case. These concepts of robustness inform the two variants of our problem previously introduced. 

Related to our work are parameter synthesis works \citep{10.1007/978-3-319-46520-3_4,cubuktepe2021convex,junges_parameter_2024,kroening_prophesy_2015,spel_finding_2021}. \citet{spel_finding_2021} propose a method for finding near-optimal parameters for parametric Markov chains (pMCs) that provide a near-optimal upper bound on reachablility probabilities. Most closely related is the work of \citet{junges_parameter_2024}, which addresses parameter synthesis by characterizing regions of a general parameter space satisfying a given specification for pMCs, parametric MDPs, and POMDP policies. Our work instead seeks the largest uniform ball around a nominal observation function guaranteeing a value threshold — a structurally different problem that reduces to a univariate search over a single scalar $\delta$. We prove and exploit the monotonicity of the worst-case value in $\delta$ to solve this via bisection, using the parameter lifting method in \citep{junges_parameter_2024} as a subroutine. While the method in \citet{junges_parameter_2024} could in principle be adapted to directly solve the problem in this paper, our approach yields favorable computational guarantees: in the non-sticky variant, we use interval Markov chain evaluation rather than parameter synthesis, yielding polynomial complexity in problem size; the sticky variant uses parameter lifting \citep{junges_parameter_2024} as a subroutine within a bisection over a single scalar rather than iterative refinement over the full parameter space.

\subsection{Outline}
The rest of the paper is organized as follows. We begin with relevant preliminaries in \cref{sec:preliminaries}. We provide a formal definition for sticky and non-sticky variants of the \problemName{} problem, as well as additional intuition on the problem structure, in \cref{sec:problem}. In \cref{sec:overall method}, we introduce a general bi-level approach to solving the problem, Robust Interval Search (RIS). In \cref{sec:solution_sticky,sec:nonsticky_solution}, we tailor this approach for the sticky and non-sticky variants of the problem, respectively, and analyze the resulting algorithms, \algoNameSticky{} (\shortAlgoNameSticky{}) and \algoName{} (\shortAlgoName{}). Finally, we provide validation and scalability results in addition to illustrative case studies for implementations of these algorithms in \cref{sec:Evaluation}.

\section{Preliminaries} \label{sec:preliminaries}
Before giving a formal problem statement, we provide definitions for POMDPs, finite-state controllers as POMDP policies, and Markov chains as a representation of the execution of FSCs on POMDPs.

\subsection{POMDPs}
POMDPs provide a model for sequential decision-making under both transition (aleatory) and state (epistemic) uncertainty.
\begin{definition}[POMDP]\label{def:POMDP}
    A \emph{partially observable Markov decision process} (POMDP) is a tuple\\ $\mathcal{M} = (S,A,O,T,Z,R,\gamma,b_0)$, where $S$ is a finite state space, $A$ is a finite action space, $O$ is a finite observations space, $T: S\times A \times S \rightarrow [0,1]$ is a transition probability function, $Z: A\times S \times O \rightarrow [0,1]$ is an observation probability function, $R: S\times A \rightarrow [R_{min},R_{max}]$ with $R_{min}, R_{max} \in \mathbb{R}$ 
    is the immediate reward function, $\gamma \in [0,1]$
    is a discount factor, and $b_0 \in \mathbf{D}(S)$ is the initial belief where $\mathbf{D}(S)$ is the probability simplex (the set of all probability distributions) over $S$.
\end{definition}

A history $h_t=(b_0,a_1,o_1,...,a_{t-1},o_{t-1})$ denotes the set of actions taken and observations received from some initial belief $b_0$ up until timestep $t$. A POMDP policy $\pi$ is a mapping from history $h_t$ to action $a_t$. The belief induced by history $h$ can be recursively computed according to the expression

\begin{equation}
    \begin{aligned}
         b_{h_t}(s') \propto \: & Z(o_{t-1}|a_{t-1},s') \cdot \sum_{s\in S}T(s'|s,a_{t-1})b_{h_{t-1}}(s),
    \end{aligned}
\end{equation}
where $b_{h_0}=b_0$.

Denote the set of reachable histories in POMDP $\mathcal{M}$ under policy $\pi$ as $H(\pi)$. Let $V_{Z}^{\pi, d}$ denote the $d$-step discounted total reward for a history $h_{\tau-d}$ under policy $\pi$ on a POMDP $\mathcal{M}$ for a horizon $\tau$. Then, the value can be computed recursively as

\begin{align}\label{eq:fh_value}
\begin{aligned}
&V_Z^{\pi, 0}(h_{\tau}) = 0 && \quad \forall h_{\tau} \in H(\pi),\\
&V_Z^{\pi, d}(h_{\tau - d}) = \sum_{s \in S} b_{h_{\tau - d}}(s) \bigg[ R(s, \pi(h_{\tau - d})) \\
&\hspace{4em} + \gamma \sum_{s' \in S} T(s' \mid s, \pi(h_{\tau - d})) \sum_{o \in O} Z(o \mid \pi(h_{\tau - d}), s') \cdot V_Z^{\pi, d - 1}(h_{\tau - d + 1}) \bigg] && \quad \forall d \le \tau.
\end{aligned}
\end{align}

\noindent
We refer to this as the finite-horizon case. When $d = \infty$, for any $t \geq 0$, we have that
\begin{align}\label{eq:ih_value}
\begin{aligned}
V_Z^{\pi, \infty}(h_t) &= \sum_{s \in S} b_{h_t}(s) \left[ R(s, \pi(h_t)) + \gamma \sum_{s' \in S} T(s' \mid s, \pi(h_t)) \sum_{o \in O} Z(o \mid \pi(h_t), s') \cdot V_Z^{\pi, \infty}(h_{t+1}) \right] \qquad \forall h_t \in H(\pi),
\end{aligned}
\end{align}

\noindent where $h_{t+1}$ denotes the history resulting from receiving observation $o$ after taking action $a$. We refer to this as the infinite-horizon case.

\begin{remark}
    In this paper, for $\tau \in \mathbb{N}$, we consider $\gamma \in [0,1]$. When $\tau = \infty$, we restrict $\gamma \in [0,1)$ to ensure that the value function is a contraction mapping  \citep{hauskrecht1997planning}.
\end{remark}

\subsection{Finite-State Controllers}
In order to promote tractability, in this work we consider policies with finite memory. Given that (potentially countably-infinite-node) state controllers can be used to represent any policy \citep{meuleau_solving_2013}, finite-state controllers (FSCs) are able to represent arbitrary finite memory history-based policies. 

Here we overload $\pi$ to denote both a functional policy which maps histories to actions and an FSC which represents such a policy. 
Denote $H(n)$ as the set of histories which reach some node $n$ in $\pi$. Here we define \raisebox{-0.25em}{$\overset{Z_0}{\rightharpoonup}$} as a partial function such that it is only defined for observations with non-zero probabilities of being received from the beliefs which correspond to nodes (as any history maps to some belief), i.e., overloading notation $O(n) = \{o|\sum_{s\in b_{h}} Z(o|a(n),s) > 0 \; \forall h \in H(n)\}$.

\begin{definition}[FSC]
    A finite-state controller (FSC) for a POMDP is defined by a tuple $\pi = (N,n_0,A,\alpha,O,T_{FSC})$ where $N$ is a finite set of memory nodes, $n_0$ is the initial memory node, A is as in \cref{def:POMDP}, $\alpha: N \rightarrow A$ is a function which maps each node to an action, O is as in \cref{def:POMDP}, and $T_{FSC}: N \times O$\raisebox{-0.25em}{$\overset{Z_0}{\rightharpoonup}$}$N$ is the memory update. Furthermore, we define $\mathcal{O}: N \times N \to 2^O$ to be a function that maps a memory transition (update) $(n,n')$ in the FSC to the set of observations that enable it:, i.e., $\mathcal{O}(n,n') = \{o 
    \in O
    \mid T_{FSC}(n,o) = n'\}$.
\end{definition}

The definition of $T_{FSC}$ is motivated by the assumption that policy designers will only specify actions for observations which can be received under the nominal model. Likewise, many solvers \cite{kurniawati_sarsop_2008,ye_despot_2017} only provide $\epsilon$-optimal solutions for a nominal model, not considering off-nominal observations.

The performance of an FSC can be evaluated via Monte Carlo simulation. 
However, achieving accurate value estimates is computationally expensive, with a convergence rate of $\frac{1}{\sqrt{n}}$. Instead, we consider a Markov chain representation. 

\subsection{Evaluation of FSCs on POMDPs: Product Markov Chains}
The execution of an FSC policy $\pi$ on a POMDP $\mathcal{M}$ can be represented as an induced product Markov chain \citep{meuleau_solving_2013}. We restate this representation
for clarity.

\label{sec:pomdpevolution}
\begin{definition}[Product Markov Chain]
\label{def:MC}
Given a POMDP $\mathcal{M}$ and an FSC policy $\pi$, the induced Markov chain is defined as $\MC = \mathcal{M} \times \pi = (q_I, Q_{\pi}, T_{\pi}, R_{\pi})$, where
\begin{itemize}
    \item $q_I = (s_I, n_0)$ is the initial Markov chain state, where $s_I$ is an augmented initial state used to capture the initial belief distribution in the MC transition function;
    \item $Q_{\pi} = (S \times N)  \cup \{q_I\}$ 
    is the set of product states;
    \item The transition function $T_\pi: Q \times Q \rightarrow [0,1]$ is defined for states $q = (s,n)$ and $q' = (s',n')$ as 
    \begin{align*}
        &T_{\pi}((s', n') | (s,n))= 
        \begin{cases}
            T(s' | s, \alpha(n)) \sum_{o \in \mathcal{O}(n,n')} Z(o|\alpha(n),s') & \text{ if } s \neq s_I \\
            b_0(s') & \text{ if } s = s_I, \; n' = n_0\\
            0 & \text{ o.w.}
        \end{cases}
    \end{align*}
    \item The reward function $R_{\pi}: (S \times N) \rightarrow \mathbb{R}$ is defined as $R_{\pi}(s, n) = R(s, \alpha(n))$.
\end{itemize}
    
\end{definition}
Note that in this definition, $T_\pi$ is a composite transition kernel that captures both the transition and observation function into one transition.

The $d$-step (discounted) total reward value of an FSC at node $n$ on POMDP from state $s$ is
\begin{equation}\label{eq:FSCVal}
    \begin{aligned}
         &V^{\pi,d}(s,n) = R(s,\alpha(n))+ \gamma \sum_{(s',n')\in Q}T_{\pi}((s',n')|(s,n)) V^{\pi,d-1}(s',n'),
    \end{aligned}
\end{equation}
with $d=\infty$ as the infinite-horizon case as in \cref{eq:ih_value}.

Likewise for an FSC on a POMDP:
\begin{equation}\label{eq:FSCValInit}
    V^{\pi,d}(h_0) = \sum_{s\in S}\left[b_0(s) V^{\pi,d}(s,n_0(\pi)) \right].
\end{equation}

Thus, we now have a means of obtaining the value of a given FSC on a POMDP, which is parameterized by the observation function $Z$.
\section{Concepts of Policy Robustness and Problem Statement} \label{sec:problem}

Characterizing how changes to POMDP observation models, whether model error or drift, affect policy performance is both non-trivial and essential for safe deployment. This work analyzes the robustness of POMDP policies to perturbations in the observation function. Specifically, we seek to identify the admissible deviations in the observation model under which a policy can be safely deployed. This motivates our problem definition.
We begin by introducing the first variant of our problem.

\subsection{The Policy Observation Robustness Problem (Sticky)}\label{sec:real_problem}

In this work, we consider history-dependent policies represented as finite-state controllers (FSCs). FSCs constitute a broad class of POMDP policies, encompassing piecewise-linear and convex $\alpha$-vector representations that can approximate the optimal value function arbitrarily well \citep{sondik1978}. 

Given a nominal POMDP model $\mathcal{M}$ and an FSC policy $\pi$, let the nominal value of the policy be $V^{\pi,d}_{Z_0}(h_0)$, where $Z_0$ is the nominal observation function. When the true observation function deviates from the nominal model, denoted by $Z'$, the value of the policy $\pi$ changes. We define the \textit{value degradation} as the decrease in value caused by this deviation:
\begin{equation}
    \threshold_D^{Z} = V^{\pi,d}_{Z_0}(h_0)- V^{\pi,d}_{Z'}(h_0).
\end{equation}

Our primary objective is to find the maximum tolerable perturbation to the observation function such that the worst-case value degradation remains below a specified threshold $\threshold$. We define a deviation of magnitude $\delta$ using a probability ball. The set of possible deviated observation functions within distance $\delta$ is:
\begin{equation}\label{eq:pball_sticky}
    \begin{aligned}
        \prball{Z_0, \delta} = &\left\{ Z'(o|a,s') \; \Big| \; |Z_0(o|a,s')-Z'(o|a,s')| \leq \delta,\right.\\
        &Z_0(o|a,s')>0 \implies Z'(o|a,s') \geq \epsilon_p,\\
        &Z_0(o|a,s')=0 \implies Z'(o|a,s')=0,\\
        &\left. \sum_{o\in O}Z'(o|a,s') = 1 \; \forall o \in O, a \in A, s'\in S  \right\}.
    \end{aligned}
\end{equation}

This definition allows for changes of magnitude $\delta$ to one or more probabilities in $Z$.
We use this definition as it is a stronger worst-case than considering a single parameter change. Additionally, this definition assumes no new observations may be added to the support while existing observations in the support may have their probabilities reduced to $\epsilon_p$. 
This structure-preserving assumption is common in the literature \citep{spel_finding_2021,junges_parameter_2024,cubuktepe2021convex,kroening_prophesy_2015}.

Formally, we define the policy observation robustness problem as follows. We consider deviations in the observation function that are applied uniformly across the state and action spaces and do not depend on the internal history or memory of the policy. This aligns with the notion of \emph{full stickiness} of an adversary's perturbations introduced by \citet{bovy_imprecise_2024}, i.e. an adversary's choice of a perturbation must "stick" after it is selected. For clarity and consistency with our setting, we refer to this as the \emph{sticky} observation robustness model.

\begin{problem}[\emph{Sticky} Observation Robustness of POMDP Policies]\label{prob:sticky}
Given a policy $\pi$ and a POMDP model $\mathcal{M}$, with nominal observation function $Z_0$, value degradation threshold $\threshold$, and horizon $d \in \mathbb{N}^0 \cup \{\infty\}$, find the maximum admissible observation function deviation $\delta_{S}$, i.e., 
\begin{subequations}\label{eq:prob_sticky}
    \begin{align}
            \delta_S = \max_{\delta\in[0,1]} & \: \delta\label{eq:prob_sticky_delta}\\
            \textnormal{s.t. } & V^{\pi,d}_{Z_0}(h_0)- V^{\pi,d}_{Z'}(h_0)\leq \threshold \qquad \forall \: Z' \in \prball{Z_0, \delta}.\label{eq:prob_sticky_constraint}
\end{align}
\end{subequations}
\end{problem}

In contrast to forward policy robustness problems, which evaluate a policy's value for a fixed perturbation level $\delta$, \cref{prob:sticky} instead seeks to compute the largest admissible perturbation radius. That is, we determine the maximum $\delta$ such that all observation models within the corresponding $\delta$-ball satisfy a prescribed value constraint for a fixed policy.

\subsection{Non-Sticky Problem}\label{sec:nonstickyproblem}

\cref{prob:sticky} induces a bi-level optimization problem, in which the inner optimization is constrained by the requirement that the adversary (nature) applies a uniform observation distribution $Z'(o \mid a, s')$ across all controller nodes. As we show in \cref{sec:AlgorithmSticky}, this problem is computationally challenging due to the structure of the inner optimization.

Motivated both by this computational difficulty and by alternative modeling assumptions on how observation uncertainty may manifest, we additionally consider a different notion of robustness. The \textit{zero stickiness} setting introduced by \citet{bovy_imprecise_2024} removes the requirement that observation deviations be shared across controller nodes. In our setting, this corresponds to allowing the observation distribution to depend not only on $a$ and $s'$, but also on the policy history $h$. We refer to this as the \textit{non-sticky} variant.

The history-dependent observation distribution function is defined as:
\begin{equation}
    \ZH: A \times S \times H \times O \rightarrow [0,1].
\end{equation}

By extending the probability ball definition from \cref{eq:pball_sticky} to include history, we construct $\prball{\ZH_0, \delta}$, which allows deviations to vary based on $h \in H(\pi)$. Formally,
\begin{equation}\label{eq:pball}
    \begin{aligned}
        \prball{\ZH_0, \delta} = &\left\{ \ZH'(o|a,s',h) \; \Big| \; |\ZH_0(o|a,s',h)-\ZH'(o|a,s',h)| \leq \delta,\right.\\
        &\ZH_0(o|a,s',h)>0 \implies \ZH'(o|a,s',h) \geq \epsilon_p,\\
        &\ZH_0(o|a,s',h)=0 \implies \ZH'(o|a,s',h)=0,\\
        &\left. \sum_{o\in O}\ZH'(o|a,s',h) = 1 \; \forall o \in O, a \in A, s'\in S, h \in H(\pi)  \right\}.
    \end{aligned}
\end{equation}

We \rev{then} define value degradation \rev{in this variant} as the decrease in value caused by a change in $\ZH$,
\begin{equation}
    \threshold_D^{\ZH} = V^{\pi,d}_{\ZH_0}(h_0)- V^{\pi,d}_{\ZH'}(h_0),
\end{equation}

\noindent where $\ZH_0$ refers to the nominal observation function, and $\ZH'$ refers to the deviated observation function. Note that while the original nominal POMDP model is inherently sticky ($\ZH_0(o|a,s',h)=Z_0(o|a,s') \; \forall h \in H(\pi)$), the adversarial deviation $\ZH'$ is free to be non-sticky.

\begin{problem}[Non-Sticky Observation Robustness Problem]\label{prob:non-sticky}
Given a policy $\pi$ and a POMDP model $\mathcal{M}$, with nominal observation function $\hat{Z}_0$, value degradation threshold $\threshold$, and horizon $d \in \mathbb{N}^0 \cup \{\infty\}$, find the maximum admissible history-dependent observation function deviation, $\delta_{NS}$, i.e.,
\begin{subequations}\label{eq:prob_nonsticky}
    \begin{align}
       \delta_{NS} =  \max_{\delta\in[0,1]} & \: \delta\label{eq:prob_nonsticky_delta}\\
        \textnormal{s.t. } & V^{\pi,d}_{\ZH_0}(h_0)- V^{\pi,d}_{\ZH'}(h_0)\leq \threshold \qquad \forall \: \ZH' \in \prball{\ZH_0, \delta}\label{eq:prob_nonsticky_constraints}.
    \end{align}
\end{subequations}
\end{problem}

This non-sticky formulation represents a distinct modeling assumption, allowing observation uncertainty to vary across information states induced by the policy. Such variability may arise, for example, from context-dependent sensing pipelines or history-dependent estimation procedures. At the same time, the non-sticky formulation strictly enlarges the adversary’s feasible set relative to the sticky case. Consequently, it yields a worst-case value that lower bounds that of the sticky formulation, and thus provides a conservative robustness guarantee. As we show in subsequent sections, it also admits more tractable solution methods.
\Cref{tab:stickyvsnonsticky} provides an overview of the differences in the problem variants.

\begin{table}[ht]
    \centering
    \caption{Summary Comparison of the Sticky and Non-Sticky Problem Variants}
    \small
    \renewcommand{\arraystretch}{1.3}
    \begin{tabular}{lcc}
        \toprule
        & \multicolumn{2}{c}{\textbf{Problem Variant}} \\
        \cmidrule(lr){2-3}
        & \textbf{Sticky} (\Cref{prob:sticky}) & \textbf{Non-Sticky} (\Cref{prob:non-sticky}) \\
        \midrule
        Notion of Robustness & Static Adversarial Perturbations & Dynamic Adversarial Perturbations \\
        Perturbation Dependencies & $Z$: State, Action & $\ZH$: State, Action, History \\
        Solution Complexity Upper Bound & Exponential [\Cref{thm:PLAComplex2}] & Polynomial [\Cref{thm:IPEComplexLeq,thm:IPEComplexEq}] \\
        Admissible Deviation Relationship & \multicolumn{2}{c}{Non-Sticky is Lower Bound on Sticky [\cref{prop:deltastickybetter}]}  \\
        \bottomrule
    \end{tabular}
    \label{tab:stickyvsnonsticky}
\end{table}

Now, we illustrate the distinction between the sticky and non-sticky formulations with the following example, which highlights how relaxing the stickiness constraint enables the adversary to induce strictly lower worst-case values.

\begin{figure}[ht]
    \begin{subfigure}[t]{0.45\textwidth}
        \centering
        \includegraphics[width=0.5\linewidth]{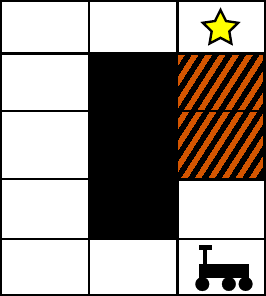}
        \caption{The Rover model environment. The striped region may or may not contain impassable sand.}
        \label{fig:rover_env}
    \end{subfigure}
    \begin{subfigure}[t]{0.49\textwidth}
        \centering
        \includegraphics[width=0.7\linewidth]{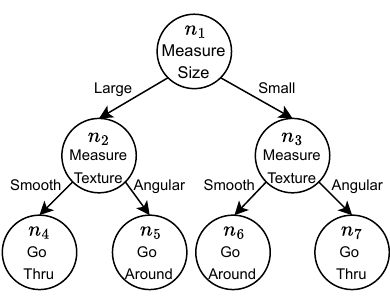}
        \caption{An FSC for the toy rover POMDP problem.}
        \label{fig:roverpg}
    \end{subfigure}
    \caption{Rover environment and FSC}
    \label{fig:rovers}
    \Description{A diagram showing the rover POMDP problem environment and the finite-state controller policy used for this example. The environment consists of two corridors which both lead to a goal location, one of which is longer than the other. The shorter corridor has striping to indicate it may or may not be traversable due to impassable sand. The policy first measures to determine the size of the sand. If it observes large sand, it then measures the texture of the sand. It takes the go through (the sandy region) action if smooth is observed and the go around action if angular is observed. If it observes small sand, it measures the texture of the sand in this case as well. It takes the go around (the sandy region) action if smooth is observed and the go through action if angular is observed.}
\end{figure}

\begin{table}[ht]
    \centering
    \caption{Traversability of the sand corridor as a function of sand properties for the rover POMDP problem.}
    \small
    \renewcommand{\arraystretch}{1.3}
    \begin{tabular}{lcc}
        \toprule
         & Large & Small \\
        \midrule
        Smooth & Yes & No \\
        Angular & No & Yes \\
        \bottomrule
    \end{tabular}
    \label{tab:traversability}
\end{table}

\begin{example}\label{ex:rover}
    Consider a rover which must autonomously navigate the environment shown in \cref{fig:rover_env}. Here, the rover may navigate through a potentially sandy corridor, which may cause the rover to become immobilized, or take a longer route, which does not have this risk, in order to reach a target location. Using the policy shown in \cref{fig:roverpg}, the rover imperfectly observes two properties of the sand, texture and size, in order to determine if the first corridor is likely to be successfully traversed.\footnote{This FSC represents a history dependent policy. We will show later by \cref{thm:history-node} that it is sufficient to analyze nodes rather than histories for such policies.} \Cref{tab:traversability} shows when the corridor is traversable. 
    Sand is traversable when it is small and angular, or when it is large and smooth. Otherwise it is not traversable. Reward of $1$ is assigned for reaching the target and $-0.1$ for taking the left corridor, to encourage efficient traversal when possible. 

    To achieve a worst-case value for this policy, nature would generally maximize the probability of an incorrect observation, e.g., an angular observation in smooth sand. However, in the sticky \rev{variant}, the observation distribution is constrained to be the same at both nodes 2 and 3, i.e. $Z'(o|a,s',n_2)=Z'(o|a,s',n_3)$. Assume there is some probability that both nodes 2 and 3 are reached from node 1 such that both nodes impact the value of the policy. Consider the case where the state is angular, large sand, causing the corridor to not be traversable. Nature then wants to increase the probability of reaching node 4, which has 0 value, from node 2 as much as possible, i.e. $\min_{Z'} V(n_2)=P_{\max}0+(1-P_{\max})1$.
    Nature then assigns more probability to smooth sand at node 2, i.e. $Z'(o_1|a,s',n_2)= Z(\text{smooth}|\text{measure texture},(\text{angular},\text{large}))=P_{\max}$. However, the sticky constraint will force the same probability at node 3, i.e. $Z'(o_1|a,s',n_3) = Z(\text{smooth}|\text{measure texture},(\text{angular},\text{large}))=P_{\max}$. This will increase the probability of reaching node 6, which has value 1, i.e. $V(n_3)=(1-P_{\max})0+P_{\max}1\neq \min_{Z'} V(n_3)$.
    In contrast, in the non-sticky \rev{variant}, nature can optimize both nodes independently such that $\mathrm{argmin}_{Z'} V(n_2) \neq \mathrm{argmin}_{Z'} V(n_3)$. This allows both to have values near 0, reducing the value at the root node. We elaborate on this example in greater detail in \cref{sec:toyrover}.
\end{example}

\subsection{The Non-Sticky \rev{Variant} as a Lower Bound on the Sticky \rev{Variant}}

An immediate observation is that the solution to the non-sticky \rev{variant} serves as a lower bound for the solution to the sticky \rev{variant}, i.e. for any given problem $\delta_{NS} \leq \delta_{S}$. This is a result of nature exploiting not only $s$ and $a$, but also the node $n$ as discussed in \cref{ex:rover}. Here, there is no constraint requiring that, for a given state, all nodes which share an action have the same probability assigned. Without this constraint, a lower worst-case bound may be achieved.

We begin by showing that the worst-case value over $\prball{\ZH_0,\delta}$ is a lower bound on the worst-case value over $\prball{Z_0,\delta}$. \rev{Similar results have been established in the literature \citep{spel_finding_2021,bovy_imprecise_2024,galesloot_pessimistic_2024}. \citet{bovy_imprecise_2024} and \citet{galesloot_pessimistic_2024} in particular address the impact of stickiness in robust POMDPs. Here, we present a brief proof for our specific formulation.}

\begin{restatable}[]{proposition}{deltastickybetter}\label{prop:deltastickybetter}
    The non-sticky admissible observation function deviation $\delta_{NS}$ is a lower bound on its sticky counterpart $\delta_S$, i.e. $\delta_{NS} \leq \delta_{S}$.
\end{restatable}

\begin{proof}
    Consider that any observation function $Z' \in \prball{Z_0, \delta}$ can be represented with a history-dependent observation function $\ZH'$, i.e. there is a $\ZH'$ such that $\ZH'(o|a,s',h)=Z'(o|a,s')\; \forall h \in H(\pi)$. Formally, the set of such $\ZH'$ is defined

    \begin{equation}
        \mathbf{\ZH_S}\left(Z_0,\delta\right) = \left\{\ZH(o|a,s',h) \middle| \ZH'(o|a,s',h)=Z'(o|a,s')\; \forall o \in O, a \in A, s'\in S, h \in H(\pi), Z' \in \prball{Z_0,\delta}\right\}.
    \end{equation}
    
    Due to the constraint in $\mathbf{\ZH_S}$, this set is a subset of all possible history-dependent observation functions, i.e. $\mathbf{\ZH_S}\left( Z_0,\delta \right) \subseteq \prball{\ZH_0, \delta}$. Then it follows for a given $\delta$, that if the non-sticky value constraint is satisfied and every non-sticky observation function satisfies, then the sticky value constraint is also satisfied as it effectively considers a subset of the non-sticky observation functions. Formally,
    \begin{equation}
        V^{\pi,d}_{\ZH^0}(h_0)- V^{\pi,d}_{\ZH'}(h_0)\leq \threshold \quad \forall \: \ZH' \in \prball{\ZH_0,\delta}
        \implies 
        V^{\pi,d}_{Z^0}(h_0)- V^{\pi,d}_{Z'}(h_0)\leq \threshold \quad \forall \: Z' \in \prball{Z_0,\delta}.
    \end{equation}

    Then every $\delta$ which satisfies \cref{eq:prob_nonsticky_constraints} also satisfies \cref{eq:prob_sticky_constraint}. However, the opposite need not be true.
    This implies that the feasible set of perturbations under the non-sticky constraint, $\boldsymbol{\delta}_{NS}$,
    is a subset of the feasible set of perturbations under the sticky constraint, $\boldsymbol{\delta}_{S}$,
    i.e. $\boldsymbol{\delta}_{NS} \subseteq \boldsymbol{\delta}_{S}$. Then the proposition follows from this set relationship.
\end{proof}

In \cref{sec:overall method}, we introduce a general search architecture to optimize $\delta$. In \cref{sec:solution_sticky}, we show how to solve the exact, sticky inner-optimization problem using parametric Markov Chains, while noting its exponential time complexity upper bound. Finally, in \cref{sec:nonsticky_solution}, we demonstrate how the non-sticky variant can be solved in polynomial time using interval Markov Chains, yielding a highly scalable policy robustness value that also acts as a lower bound on \cref{prob:sticky}.
\section{General Solution Framework: Robust Interval Search} \label{sec:overall method}

In this section, we present a general algorithmic framework, called \emph{Robust Interval Search} (RIS), for solving the policy observation robustness problems introduced in \cref{sec:problem}. The key idea is to separate the optimization over the admissible perturbation radius $\delta$ from the evaluation of the robustness constraint for a fixed perturbation set.

Rather than searching directly over the high-dimensional space of admissible observation models, RIS exploits a monotonicity property of the value constraints with respect to $\delta$. This observation allows us to reduce the outer optimization to a one-dimensional search over $\delta \in [0,1]$. For any fixed $\delta$, we query an inner feasibility evaluator that determines whether the corresponding uncertainty set satisfies the prescribed value threshold. 

This outer-loop structure is shared by both the sticky and non-sticky formulations; the distinction between the two variants lies  in how the feasibility evaluator, and thus the underlying worst-case evaluation, is realized.

\subsection{Monotonicity of Value Constraints}\label{sec:monotonicity}

To enable efficient search over $\delta$, we first establish that the worst-case value is monotonic in the size of the uncertainty set. Intuitively, increasing $\delta$ enlarges the set of admissible observation perturbations available to the adversary, and therefore cannot improve the worst-case policy value.

We formalize this property for both the sticky and non-sticky formulations.

\begin{restatable}[Monotonicity of Value Constraints]{lemma}{Monosticky}
    \label{lm:monotonicity} The constraints in
    \cref{eq:prob_sticky_constraint,eq:prob_nonsticky_constraints} exhibit monotone non-decreasing violation in $\delta$. That is, if the constraints are violated at $\delta_b$, they remain violated for all $\delta_c > \delta_b$; if they are satisfied at $\delta_b$, they remain satisfied for all $\delta_a < \delta_b$.
\end{restatable}

\begin{proof}
    We consider the sticky variant. The same argument follows for the non-sticky variant. We first consider $\delta_b < \delta_c$. If \cref{eq:prob_sticky_constraint} is violated for $\delta_b$, by definition $\exists Z' \in \prball{Z_0,\delta_b}$ such that  $V^{\pi,d}_{Z_0}(h_0)- V^{\pi,d}_{Z'}(h_0) > \threshold$. For $\delta_b<\delta_c$, $\prball{Z_0,\delta_b} \subseteq \prball{Z_0,\delta_c}$, i.e. every observation model in the set defined by $\delta_b$ is also in the set defined by $\delta_c$. Then, it follows that  $Z' \in  \prball{Z_0,\delta_c}$ and $V^{\pi,d}_{Z_0}(h_0)- V^{\pi,d}_{Z'}(h_0) > \threshold$.

    A similar argument follows for $\delta_a < \delta_b$. If \cref{eq:prob_sticky_constraint} is satisfied for $\delta_b$, by definition $\forall Z' \in \prball{Z_0,\delta_b}$, $V^{\pi,d}_{Z_0}(h_0)- V^{\pi,d}_{Z'}(h_0) \leq \threshold$. For $\delta_a<\delta_b$, $\prball{Z_0,\delta_a} \subseteq \prball{Z_0,\delta_b}$, i.e. every observation model in the set defined by $\delta_a$ is also in the set defined by $\delta_b$. Then, it follows that  $\forall Z' \in \prball{Z_0,\delta_a}$, $V^{\pi,d}_{Z_0}(h_0)- V^{\pi,d}_{Z'}(h_0) \leq \threshold$.
\end{proof}

Given this monotonicity in the value constraints, the maximal admissible perturbation radius can be computed via interval search rather than exhaustive exploration of the perturbation space.

\subsection{Robust Interval Search}

Algorithm~\ref{alg:overall} summarizes the complete RIS framework. It 
first evaluates the nominal policy value under $Z_0$, then defines a 
feasibility evaluator $f(\delta)$ over candidate perturbation radii, and 
finally returns the largest $\delta$ satisfying the robustness constraint 
via a modified bisection search (described in 
\cref{sec:bisection}). All model-specific reasoning is encapsulated in 
\texttt{feasibility\_evaluation}; the outer search is entirely agnostic 
to whether the sticky or non-sticky formulation is in use. 

\begin{algorithm}[ht]
    \caption{Robust Interval Search (RIS)}
    \label{alg:overall}
    \begin{algorithmic}[1]
        \STATE $V^d_0 \gets \texttt{evaluate}(\mathcal{M}, \pi, 0)$
        \STATE $\delta_a \gets 0$
        \STATE $\delta_b \gets 1$
        \STATE $f(\delta) \gets \texttt{feasibility\_evaluation}(\mathcal{M}, \pi, V^d_0, \threshold, \delta)$
        \RETURN \texttt{\modifiedBisectionShort{}}$(f,\delta_a,\delta_b,\epsilon_{\modifiedBisectionShort{}})$
    \end{algorithmic}
\end{algorithm}

\subsubsection{Outer Optimization via Modified Bisection Search}\label{sec:bisection}

\begin{algorithm}[ht]
    \caption{Modified Bisection Search (\modifiedBisectionShort{})}
    \label{alg:modifiedbisection}
    \hspace{-7.9cm} \texttt{\modifiedBisectionShort{}}($f,\delta_a,\delta_b,\epsilon_{\modifiedBisectionShort{}}$)
    \begin{algorithmic}[1]\label{alg:MBS}
        \STATE $f_a \gets f(\delta_a)$, $f_b \gets f(\delta_b)$
        
        \IF{$sign(f_a)=sign(f_b)$}
            \RETURN $\delta_b$
        \ENDIF
        \WHILE{$\delta_b-\delta_a>\epsilon_{\modifiedBisectionShort{}}$}
            \STATE $\delta_c \gets \frac{\delta_a+\delta_b}{2}$, $f_c \gets f(\delta_c)$
            
            \IF{$f_c=0$}
             \color{brown}\STATE {$\delta_a \gets \delta_c$, $f_a \gets f_c$}
                \COMMENT Assign as lower bound instead of terminating
            \ELSIF{$sign(f_c)=sign(f_a)$}
                \STATE $\delta_a \gets \delta_c$, $f_a \gets f_c$
                
            \ELSE
                \STATE $\delta_b \gets \delta_c$, $f_b \gets f_c$
                
            \ENDIF
        \ENDWHILE
        \RETURN $\delta_a$
    \end{algorithmic}
\end{algorithm}

To find $\delta^*$, we evaluate the feasibility evaluator $f(\delta)$
at candidate radii and apply a 
modified bisection procedure (Algorithm~\ref{alg:modifiedbisection}). Broadly, $f(\delta)$ indicates whether the constraints of \cref{eq:prob_sticky_constraint,eq:prob_nonsticky_constraints} are satisfied. We provide more detail on the specific forms of these functions below.
At $\delta = 0$ the uncertainty set is a singleton containing only $Z_0$ 
itself, so the worst-case value equals the nominal value and the 
constraint is satisfied with zero margin. At $\delta = 1$ we evaluate 
the feasibility evaluator at the widest possible uncertainty set. If the 
constraint holds at both endpoints, any $\delta \in [0,1]$ is admissible 
and we return $\delta = 1$. Otherwise, we perform bisection over the 
interval to locate $\delta^*$ within tolerance $\epsilon_\text{MBS}$. Line 8 in \cref{alg:MBS} highlights the key modification relative to standard bisection: when 
$f(\delta_c) = 0$, rather than terminating, we assign $\delta_c$ as the 
new lower bound. This ensures we return the \emph{largest} admissible 
radius rather than merely any root of $f$. This is formalized in the following propositions:

\begin{restatable}[Soundness of MBS]{proposition}{bisectsound}
\label{lm:bisect_sound}
    Given a function $f$ which is well-defined on $\delta \in [0,1]$, an initial lower bound $\delta_a=0$ such that $f(\delta_a) \leq 0$, and upper bound $0 \leq \delta_b \leq 1$, the solution returned by \modifiedBisection{} (\modifiedBisectionShort{}), $\delta_a$, will give $f(\delta_a) \leq 0$.
\end{restatable}

\begin{restatable}[Convergence of MBS]{proposition}{bisectconverge}
\label{lm:bisectconverge}
    Given a function $f$ which is monotone non-decreasing in $\delta$, the solution to the optimization problem
    \begin{equation}\label{eq:bisectminmax}
    \begin{aligned}
        \delta^* = \max_{\delta \in [0,1]}&\delta\\
        \textnormal{s.t. }& f(\delta) \leq 0\\
    \end{aligned}
    \end{equation}
    an initial lower bound $\delta_a=0$ such that $f(\delta_a) \leq 0$, and an initial upper bound $\delta_b = 1$, \modifiedBisection{} (\modifiedBisectionShort{}) is guaranteed to converge to $\delta^* - \delta < \epsilon_{\modifiedBisectionShort{}}$ for any $\epsilon_{\modifiedBisectionShort{}} > 0$.
\end{restatable}

These results follow from the construction of \modifiedBisectionShort{} and the structure of $f$. Proofs are provided in \cref{sec:bisect_sound,sec:bisect_converge}. 

The remainder of the methodology instantiates this framework for the two robustness notions introduced earlier. In \cref{sec:solution_sticky}, we present a feasibility evaluator for the sticky formulation, where the adversary must apply a history-independent observation perturbation. In \cref{sec:nonsticky_solution}, we then present the corresponding feasibility evaluation for the non-sticky formulation, where perturbations may vary across policy histories. Both inner evaluations slot 
directly into the outer bisection loop described here.
\section{Exact Evaluation for the General Sticky Problem Variant} \label{sec:solution_sticky}

Having established the general framework for optimizing the admissible deviation, $\delta$, we now solve the sticky variant of the inner optimization.
While the MBS outer optimization efficiently searches for $\delta$, an inner feasibility evaluation is necessary to \rev{determine if all sticky observation models in the probability ball with radius $\delta$} satisfy the value threshold constraint, i.e., $V^{\pi,d}_{Z_0}(h_0)- V^{\pi,d}_{Z'}(h_0)\leq \threshold \; \forall \: Z' \in \prball{Z_0, \delta}$.

\subsection{Inner Optimization: Parametric Markov Chain Evaluation} \label{sec:bilevels}

In order to evaluate \rev{\cref{eq:prob_sticky}}, we use a product Markov chain to represent an FSC operating on a POMDP. However, since in this case observation distributions are only a function of action and state, probabilities must be shared between product Markov chain nodes. \rev{To address this, we leverage parametric Markov chains (pMCs) \citep{daws2004symbolic}, which have been applied in prior work to parameterize finite-state POMDP policies \citep{junges_finite-state_nodate}. Here, we use them to parameterize the observation function.}

We represent observation probabilities parametrically with the set of parameters denoted $p=\{p_1,...,p_n\} \in P \subseteq [0,1]^n$.
We denote the support of Z as $\text{supp}(Z(a,s')) = \{o |Z(o|a,s')>0\}$ and the size of the support as $k = \lvert\text{supp}(Z_0(a,s'))\rvert$. Let $o_i$ be the $i^{th}$ observation in the support and $F: P \rightarrow \{ \sum_{p_i\in p} a_i p_i + c \mid a_i,c \in \mathbb{R}, p \in P\}$ be a function mapping a set of parameters to the polynomials of degree 1 in $P$.
Define the parametric observation function $Z_P: A\times S \times O \rightarrow F(P)$ as follows
\begin{equation}
    Z_P(o_i|a,s') =
    \begin{aligned}
        \begin{cases}
            p_{j+i} &\text{if } o_i <k \\
            1- \sum_{l=1}^{k-1}p_{j+l} &\text{if } o_i = k\\
            0 &\text{o.w.}
        \end{cases},
    \end{aligned}
\end{equation}
where $j = j(a, s')$ is an offset indexing the parameters associated with each $(a, s')$ pair.

\begin{definition}[Parametric Markov Chain]
    Given a POMDP model $\mathcal{M}$, an FSC $\pi$, an induced parametric Markov chain is defined $\MC_{\parameterized} = (q_I,Q_{\parameterized}, P, T_{\parameterized}, R_{\parameterized}, \gamma_{\parameterized})$, where 
    \begin{itemize}
        \item $q_I = (s_I, n_0)$ is the initial Markov chain state, where $s_I$ is an augmented initial state used to capture the initial belief distribution in the pMC transition function;
        \item $Q_{\parameterized} = (S \times N ) \cup \{q_I\}$
        is the cross product of the POMDP states and controller nodes;
        \item $P$ is the space of possible parameter values which determine observation probabilities and therefore Markov chain transition probabilities;
        \item The transition function $T_{\parameterized}: Q_{\parameterized} \times Q_{\parameterized} \rightarrow F(P)$ 
        is defined as 
             \begin{align*}
            &T_{\parameterized}((s', n') | (s, n))= 
            \begin{cases}
                T(s' | s, \alpha(n)) \cdot Z_P(\mathcal{O}(n,n')|\alpha(n),s') &\text{if } s \neq s_I \\
                b_0(s') &\text{if } s=s_I, n' = n_0\\
                0 &\text{o.w.}
            \end{cases};
        \end{align*}
        \item The reward function $R_{\parameterized}: (S \times N) \rightarrow \mathbb{R}$ is defined as $R_{\parameterized}(s, n) = R(s, \alpha(n))$;
        \item  The discount factor $\gamma_{\parameterized}$.
    \end{itemize}
\end{definition}

Let $P(\delta)$, the region generated by a given $\delta$, be defined

\begin{equation}
    \begin{aligned}
         P(Z,\delta) &= [p^{\downarrow}(Z,\delta),p^{\uparrow}(Z,\delta)] \text{ where }\\
         & p^{\downarrow}(Z,\delta) = \{\max(Z(o|a,s')-\delta,\epsilon_p) \; \forall s \in S, a \in A, o \in \{o_1,...,o_{k-1}\}\},\\ 
         &p^{\uparrow}(Z,\delta) = \{\min(Z(o|a,s')+\delta,1-\epsilon_p) \; \forall s \in S, a \in A, o \in \{o_1,...,o_{k-1}\} \}.\\
    \end{aligned}
\end{equation}

A pMC can be used to evaluate an FSC on a POMDP. This is formalized in the following statement. 

\begin{restatable}[]{lemma}{pMCequiv}
    \label{prop:pMCequiv}
    Let $\mathcal{M}$ be a POMDP with a finite-state controller policy $\pi$. Let $\MC_{\parameterized}$ be the induced parametric Markov chain constructed with $\mathcal{M},\pi$ and instantiated with $\epsilon_p=0$ and $p=p^{\uparrow}(Z,0)=p^{\downarrow}(Z,0)$ 
    such that $Z_P=Z$. For any horizon $d$, we have that
    \begin{equation}
        V^{\pi,d}_{Z}(h_0) = V^{\MC_{\parameterized},d}_{p}(q_I).
    \end{equation} 
\end{restatable}

\begin{proof}[Proof Sketch]
    Given the same initial state, by construction any state-augmented history reached in the POMDP $\mathcal{M}$ with the FSC $\pi$ can be reached in $\MC_\parameterized$ constructed with $\mathcal{M},\pi$ and equipped with $p$ and \textit{vice versa}. Likewise, given the same initial state, the probability of reaching these state-augmented histories is the same. Then, the discounted reward assigned to these histories will be the same in both models. This equivalence of models follows in a similar manner to existing proofs for parametric Markov chains \cite{junges_finite-state_nodate}. 
\end{proof}

We now show that the feasibility constraint over the parametric Markov chain with region $P(Z,\delta)$ is equivalent to the feasibility constraint in \cref{eq:prob_sticky}.

\begin{restatable}[]{theorem}{plainner}
\label{thm:pla_inner}
    Given a POMDP $\mathcal{M}$, FSC policy $\pi$, horizon $d$, and $\delta$, construct $\MC_{\parameterized}$. Then, 
    \begin{equation}
        V^{\MC_{\parameterized},d}_{p^0}(q_I)- V^{\MC_{\parameterized},d}_{p'}(q_I)\leq \threshold \quad \forall \: p' \in P(Z_0,\delta) 
        \iff
        V^{\pi,d}_{Z_0}(h_0)- V^{\pi,d}_{Z'}(h_0)\leq \threshold \quad \forall \: Z' \in \prball{Z_0, \delta}
    \end{equation}
\end{restatable}

\begin{proof}[Proof]
    By \cref{prop:pMCequiv}, any valid pMC realized on the interval $P(Z,\delta)$ is equivalent in value to the composition of a POMDP $\mathcal{M}$ and an FSC $\pi$. 
    Then if $P(Z_0,\delta)$ and $\prball{Z_0, \delta}$ describe the same set of possible probability distributions for transitioning from FSC node $n$ to node $n'$, for any $p'$ satisfying $V^{\MC_{\parameterized},d}_{p^0}(q_I)- V^{\MC_{\parameterized},d}_{p'}(q_I)\leq \threshold$, there exists a $Z'$ such that $V^{\pi,d}_{Z_0}(h_0)- V^{\pi,d}_{Z'}(h_0)\leq \threshold$ and vice versa.
    
    If the bounds on $P(Z_0,\delta)$ and $\prball{Z_0, \delta}$ for any $o,a,s'$ are the same it follows that their sets are the same. The bounds on $\mathbf{Z}=\prball{Z_0,\delta}$, the probability ball with radius delta around $Z_0$, for any given $o,a,s'$ are defined
    \begin{equation}
        Z^{\downarrow}(o|a,s')=\max(Z_0(o | a, s') - \delta, \epsilon_p)
    \end{equation}
    and 
    \begin{equation}
        Z^{\uparrow}(o|a,s')=
        \begin{aligned}
            \begin{cases}
                \min(Z_0(o | a, s') + \delta, 1) &\text{if } Z_0(o | a, s') > 0\\
                0 &\text{o.w.}
            \end{cases} .
        \end{aligned}
    \end{equation} 

    $Z^{\downarrow}, Z^{\uparrow}$ are bounds on all possible valid observation probability functions in $\mathbf{Z}$, although they themselves are not valid distributions. By construction $p^{\downarrow}(Z,\delta)$ and $p^{\uparrow}(Z,\delta)$ induce $Z^{\downarrow}$ and $Z^{\uparrow}$. We note here that the $1-\epsilon_p$ upper bound for $p^{\uparrow}(Z,\delta)$ is necessary to ensure that $o_k$ is observed with at least $\epsilon_p$ probability, i.e. that $Z^{\downarrow}$ is respected. This implies that $Z^{\uparrow}=1$ cannot be attained unless $\epsilon_p=0$. Thus, both problems have the same bounds on transitioning to new node $n'$ from node $n$ as a function of an observation $o$.
\end{proof}

Having established this equivalence, \rev{\cref{eq:prob_sticky}} can be written as
\begin{subequations}\label{eq:mcprob_sticky}
    \begin{align}
        \delta_S = \max_{\delta\in[0,1]} & \: \delta \label{eq:mcproba_sticky_obj}\\
        \text{s.t. } &  V^{\MC_{\parameterized},d}_{p^0}(q_I)-V^{\MC_{\parameterized},d}_{p'}(q_I) \leq \threshold \quad \forall p' \in P(Z_0,\delta) \label{eq:mcproba_sticky_con}
    \end{align}
\end{subequations}

where $V^{\MC_{\parameterized},d}_{p}$ is the value of pMC $\MC_{\parameterized}$ with parameter instantiation $p$ at the initial node $q_I$.

As shown in \cref{sec:monotonicity}, the feasibility constraint is monotonic as the uncertainty ball grows.
This monotonicity allows us to use an efficient modified bisection method to find a solution to the outer \cref{eq:mcprob_sticky} while guaranteeing convergence and soundness. We now turn our attention to the algorithm and its guarantees.

\subsection{\algoNameSticky{}}\label{sec:AlgorithmSticky}

In this section, we instantiate the RIS framework for the sticky formulation by specifying the corresponding feasibility evaluation. The resulting algorithm, \algoNameSticky{} (\shortAlgoNameSticky{}) is shown in \cref{alg:overall_sticky}.

At each iteration, the algorithm evaluates whether the robustness constraint is satisfied for a candidate $\delta$ by solving the corresponding inner optimization problem. When this inner evaluation (cf. \cref{eq:mcproba_sticky_con}) is performed exactly, \algoNameSticky{} returns a value of $\delta$ within $\epsilon_{\modifiedBisectionShort{}}$ of the optimal admissible deviation $\delta^*$ defined in \cref{eq:mcprob_sticky}.

The feasibility evaluation constructs a parametric Markov chain $\MC_{\parameterized}$ 
from the POMDP--FSC pair and evaluates the nominal policy value 
$V^{\MC_{\parameterized},d}_{p^0}$. For a candidate radius $\delta$, 
the feasibility evaluator
\begin{equation}\label{eq:f_p}
    f_{p}(\delta) = \begin{cases}
        -1 &\text{ if } V^{\MC_{\parameterized},d}_{p^0}(q_I)-V^{\MC_{\parameterized},d}_{p'}(q_I) \leq \threshold \quad p' \in P(Z_0,\delta) \\
         1 &\text{ if } V^{\MC_{\parameterized},d}_{p^0}(q_I)-V^{\MC_{\parameterized},d}_{p'}(q_I) > \threshold \quad p' \in P(Z_0,\delta) 
    \end{cases}
\end{equation}
returns $-1$ when the constraint is satisfied and 1 when it is violated.
This function, together with the search interval $[0,1]$ and tolerance $\epsilon_{\modifiedBisectionShort{}}$, 
is passed to \modifiedBisectionShort{} (\cref{alg:MBS}).

\begin{algorithm}
    \caption{\algoNameSticky{}}
    \label{alg:overall_sticky}
    \begin{algorithmic}[1]
        \STATE  $\MC_{\parameterized} \gets $\texttt{make\_pMC} $(\mathcal{M,\pi})$
        \STATE $V^{\MC_{\parameterized},d}_{p^0} \gets $\texttt{evaluate} (\texttt{make\_pMC}$(\MC_{\parameterized},0) )$ \COMMENT{Using the parametric extension of \cref{eq:FSCVal}}
        \STATE $\delta_a \gets 0$
        \STATE $\delta_b \gets 1$
        \STATE $f_{p}(\delta) \gets \texttt{feasibility}(V^{\MC_{\parameterized},d}_{p^0},\threshold,\delta)$ \COMMENT{\cref{eq:f_p}}
        \RETURN \texttt{\modifiedBisectionShort{}}($f_p,\delta_a,\delta_b,\epsilon_{\modifiedBisectionShort{}}$) \COMMENT{\cref{alg:MBS}}
    \end{algorithmic}
\end{algorithm}

\subsection{Analysis of \shortAlgoNameSticky{}}
We now address the soundness and convergence of  \shortAlgoNameSticky{}. We show \shortAlgoNameSticky{} converges to within $\epsilon_{\modifiedBisectionShort{}}$ of $\delta$ and does so in at most exponential time. We first provide a lemma on the monotonicity of $f_p$, which underpins the convergence of \shortAlgoNameSticky{}.

\begin{lemma}\label{prop:fp}
    The sticky feasibility evaluator $f_p$, \cref{eq:f_p}, is monotone non-decreasing in $\delta$.
\end{lemma}

This lemma follows from the equivalence of \cref{eq:mcproba_sticky_con} and \cref{eq:prob_sticky_constraint} (\cref{thm:pla_inner}), the monotonicity of \cref{eq:prob_sticky_constraint} in $\delta$ (\cref{lm:monotonicity}), and the construction of \cref{eq:f_p}. We now provide the key result on the convergence and soundness of \shortAlgoNameSticky{} with exact evaluation.

\begin{restatable}[]{theorem}{AlgoThmSticky}
    \label{thm:AlgoThmSticky}
    Let $\delta^*$ be the unique optimal solution to the optimization problem in \cref{eq:mcprob_sticky}. Given an exact \rev{evaluation} of \cref{eq:mcproba_sticky_con}, \algoNameSticky{} (\Cref{alg:overall_sticky}) returns a solution $\delta$ such that $V^{\MC_{\parameterized},d}_{p^0}(q_I)-V^{\MC_{\parameterized},d}_{p'}(q_I) \leq \threshold$ and $|\delta^*-\delta| \leq \epsilon_{\modifiedBisectionShort{}}$.
\end{restatable}

\begin{proof}\label{proof:AlgoThm_sticky}
    By \cref{prop:fp}, $f_p$ as defined in \cref{eq:f_p} is \rev{monotone non-decreasing} in $\delta$. \rev{It is well-defined, returns an exact evaluation of \cref{eq:mcproba_sticky_con} by assumption, and, since $Z_0$ trivially satisfies \cref{eq:mcproba_sticky_con}, also has $f_{p}(0) \leq 0$.} Then, by \cref{lm:bisect_sound} and \cref{lm:bisectconverge} modified bisection search on $f$ will converge to $\delta$ within $\epsilon_{MBS}$ of $\delta^*$ which ensures $V^{\MC_{\parameterized},d}_{p^0}(q_I)-V^{\MC_{\parameterized},d}_{p'}(q_I) \leq \threshold$.
\end{proof}

These guarantees extend to \cref{prob:sticky} as a result of \cref{thm:pla_inner}.

\begin{restatable}[]{corollary}{SolveOriginalNS}
     Let $\delta^*$ be the unique optimal solution to the optimization problem in \cref{eq:prob_sticky}. Given an exact \rev{evaluation of} \cref{eq:mcproba_sticky_con}, \algoNameSticky{} (\Cref{alg:overall_sticky}) returns a solution $\delta$ such that $V^{\pi,d}_{Z_0}(h_0)-V^{\pi,d}_{Z'}(h_0) \leq \threshold$ and $|\delta^*-\delta| \leq \epsilon_{\modifiedBisectionShort{}}$.
\end{restatable}

Thus, \shortAlgoNameSticky{} with exact feasibility evaluation solves the sticky variant of \problemName{} problem with soundness and convergence guarantees. Next, we provide an analysis of the time complexity of \shortAlgoNameSticky{}.

\begin{restatable}[]{proposition}{PLAComplex2}
\label{thm:PLAComplex2}
    \rev{Unless P=NP, \algoNameSticky{} (\Cref{alg:overall_sticky}) is not polynomial and is at most exponential in the number of parameters in the parametric Markov chain.}  
\end{restatable}
\begin{proof}[Proof Sketch]
    \rev{The feasibility problem for a  parametric Markov chain with a single Probabilistic Computation Tree Logic with expectation and comparison (PCTL+EC) operator is NP-hard \cite{hutschenreiter_parametric_2017,BAIER2020104504}. Since the verification problem, i.e. checking if all parameter valuations in a region satisfy, is logically equivalent to the feasibility problem with a negated operator, i.e. checking if any instance fails to satisfy,
    the verification problem is also NP-hard.}
    \rev{For the upper bound, \citet{hutschenreiter_parametric_2017,BAIER2020104504} establish a PSPACE upper bound on the general feasibility problem with PCTL+EC via a reduction to an ETR problem. This {implies an upper bound which has exponential complexity in the number of parameters}. Here again, this verification problem is logically equivalent to a feasibility problem with a negated operator, implying that the verification problem is also in PSPACE and has an upper bound which has exponential complexity in the number of parameters.}
    Bisection search is known to be logarithmic in the inverse of the tolerance $\frac{1}{\epsilon_{\modifiedBisectionShort{}}}$ and the interval width \citep{burden38216numerical}. Then, since a logarithmic number of inner optimization calls is required, \shortAlgoNameSticky{} will \rev{not be polynomial in the number of parameters and will have an upper bound which is} exponential in the number of parameters in the pMC.
\end{proof}

\rev{Having established the soundness, convergence, and exponential complexity of \shortAlgoNameSticky{}, we now turn our attention to solving the non-sticky variant of the \problemName{} problem.}
\section{Exact Evaluation for the Non-Sticky Problem Variant} \label{sec:nonsticky_solution}

In this section, we present the RIS feasibility oracle for the non-sticky formulation of \cref{prob:non-sticky}. In contrast to the sticky case, the non-sticky setting allows observation perturbations to depend on the policy history, which leads to a more expressive but tractable inner optimization problem. Our key insight is that this additional flexibility enables the inner optimization of \cref{eq:prob_nonsticky_constraints} to be reformulated as a structured problem over interval-valued transitions. Specifically, we show that the feasibility evaluation can be carried out by a worst-case evaluation of a Markov chain representation with intervals defined over FSC node-dependent observation distributions. This reformulation yields a tractable procedure for evaluating the feasibility oracle at a fixed $\delta$. Combined with the monotonicity of the feasibility condition in $\delta$, this allows RIS to efficiently compute the admissible robustness radius for the non-sticky variant.

We next present the construction of this interval-based model and show how the inner optimization can be solved efficiently.

\subsection{Bi-Level Optimization for the Non-Sticky \rev{Variant}} \label{sec:bilevelns}

To find $\delta_{NS}$, we reformulate \cref{eq:prob_nonsticky} as a bi-level optimization problem. \rev{As we will show in \cref{sec:NSAnalysis}, this approach has computational advantages, enabling a tractable feasibility check via dynamic programming. We first observe that if the worst-case $\ZH'$ (that is the $\ZH'$ which results in the lowest value) in  $\mathbf{\ZH}=\prball{\ZH_0,\delta}$ satisfies the value constraint $\threshold$, then all other $\ZH' \in \mathbf{\ZH}$ also satisfy $\threshold$.}

Then, determining whether the set of $\mathbf{\ZH}$ satisfies the value constraint $\threshold$ can be simplified to ensuring that the worst-case $\ZH'$ in $\mathbf{\ZH}$ satisfies the value constraint. \Cref{eq:prob_nonsticky} can then be written as
\begin{subequations}\label{eq:nested2}
    \begin{align}
        &\begin{aligned}\label{eq:nested2pt1}
           \delta_{NS} = \max_{\delta\in[0,1]} & \: \delta\\
            \text{s.t. } & V^{\pi,d}_{\ZH_0}(h_0)- V^{\pi,d}_{\ZH'}(h_0)\leq \threshold \\ 
        \end{aligned}\\
        & \qquad \quad \begin{aligned} \label{eq:inner}
            & \qquad V^{\pi,d}_{\ZH'}(h_0) = \min_{\ZH \in \prball{\ZH_0, \delta}} \: V^{\pi,d}_{\ZH}(h_0),\\
        \end{aligned}
    \end{align}
\end{subequations}

Here, the inner optimization problem finds the value-minimizing $\ZH$ for a given $\delta$, while the outer finds the $\delta$ such that $\threshold$ is not violated. We now describe our solution to the inner problem, \cref{eq:inner}.

\subsection{Inner Optimization: Interval Markov Chain Policy Evaluation}

Assessing the worst-case value under some $\ZH$ is, in the general case, intractable due to the curse of history \citep{pineau2003point}. That is, the number of reachable histories grows exponentially with the horizon of the problem and is infinite in the infinite horizon case. This then necessitates a means of representing reachable histories in a tractable manner. To address this need, we consider perturbations on FSC node-dependent observation functions. We show that this representation is equivalent to \cref{eq:nested2} and present a solution to this node-dependent problem.

Consider an FSC node-dependent observation function $\Z$:
\begin{equation}
    \Z: A \times S \times N \times O \rightarrow [0,1].
\end{equation}
We begin by formulating the FSC node-dependent robustness problem as  
\begin{subequations}\label{eq:nested2node}
    \begin{align}
        &\begin{aligned}\label{eq:nested2pt1node}
            \delta_{NS} = \max_{\delta\in[0,1]} & \: \delta\\
            \text{s.t. } & V^{\pi,d}_{\Z_0}(h_0)- V^{\pi,d}_{\Z'}(h_0)\leq \threshold \\ 
        \end{aligned}\\
        & \qquad \quad \begin{aligned} \label{eq:innernode} 
            & V^{\pi,d}_{\Z'}(h_0) = \min_{D^d_{\Z}} \: V^{\pi,d}_{\Z}(h_0),\\
        \end{aligned}
    \end{align}
\end{subequations}
where $D^d_{\Z}$ is the set of observation function(s) being minimized over.

In the finite-horizon case which may have non-stationary value, i.e. value which depends on horizon $d$,
\begin{equation}\label{eq:fhZdomain}
    D^d_{\Z} = \left\{\Z^d,  ...\ , \Z^1\right\}
\end{equation}
where each $\Z^i$ lies in the probability ball $\prball{\Z_0, \delta}$ and $V_{\Z_0}^{\pi,d}(h)$ and $V_{\Z'}^{\pi,d}(h)$ are as in \cref{eq:fh_value} except that they operate on a unique $\Z$ at each $d$ instead of $Z$. 

In the infinite-horizon case which has a stationary value, i.e. a value which does not depend on the horizon $d$,
\begin{equation}\label{eq:ihZdomain}
    D^{\infty}_{\Z} = \left\{\Z \right\}
\end{equation}
where $\Z$ lies in the probability ball $\prball{\Z_0, \delta}$ and $V_{\Z_0}^{\pi,\infty}(h)$ and $V_{\Z'}^{\pi,\infty}(h)$ are as in \cref{eq:ih_value} except that they operate on $\Z$.

\begin{restatable}[]{theorem}{equivalence}
    \label{thm:history-node}
    For a history-dependent policy represented as an FSC $\pi$, the history-dependent observation function problem, \cref{eq:nested2}, and the node-dependent observation function problem, \cref{eq:nested2node}, achieve the same $\delta$.
\end{restatable} 

\begin{proof}[Proof Sketch]
    Here, we provide a proof sketch. A full proof is provided in \cref{sec:history-node}. It is sufficient to show that \cref{eq:inner} achieves the same objective value as \cref{eq:innernode}. We show that \cref{eq:inner,eq:innernode} can be represented as recursive Bellman updates on $s,h$ which take the minimum value over all $\ZH$ or $\Z$ at a given step. We then show inductively for the finite-horizon case that, given any history $h$ and state $s$, both the $\ZH$ and $\Z$ cases have the same value at horizon $d=1$ and minimize over the same domain, therefore achieving the same value. In the infinite-horizon case, we show both Bellman updates converge to fixed points. Given the finite-horizon equivalence of the two updates for any horizon length, the two converge to the same infinite-horizon value.
\end{proof}

A consequence of the equivalence of the history-dependent and node-dependent observation function problems is that it enables a mathematical representation of the minimization problem \cref{eq:innernode} that is computationally efficient to solve. \rev{As with other works \citep{galesloot_pessimistic_2024,robustpolsynth,cubuktepe_robust_2021}, we leverage uncertain Markov chains to determine the worst-case value. Specifically, we use an interval Markov chain (iMC) \cite{jonsson_specification_1991}.} The FSC evaluation described in \cref{def:MC} can be used to compute the value given a singleton node-augmented observation function $\Z$. The transition function for the Markov chain is a product of the POMDP transition and observation functions. For $\delta > 0$, the node-augmented observation function is not a singleton but the set $\mathbf{\Z}_{\delta}$ with radius $\delta$ defined by $\prball{\Z_0, \delta}$. Specifically, each element of the observation function lies in an interval of width $\delta$. \rev{ This set may be captured as an interval Markov chain where $\delta$ forms the interval on the observation function and consequently the Markov chain transition interval.}

Given an iMC, established methods \cite[e.g.][]{givan_bounded-parameter_2000} can be used to find the minimum value over the interval.
The most straightforward MC model is shown in \cref{fig:MC_product}. However, this formulation has a key drawback -- the probability simplex ($\sum_{(s',n') \in Q_{\pi}} \left[ T_{\pi}((s',n')|(s,n))\right] = 1$) is only enforced on the product transition function $T_\pi$, not specifically on $\Z$ and $T$ which compose it.
\rev{Then, given a transition function with fixed probabilities, the observation function recovered from the Markov chain transition function may no longer obey the probability simplex requirement. This is despite the product of the two (the Markov chain transition function) obeying this requirement.} 
Thus, in order to leverage iMC methods to find a valid $\Z$, a different representation is needed.

\begin{figure}
    \centering
    \begin{subfigure}[b]{0.49\linewidth}
        \centering
        \includegraphics[width=0.8\linewidth]{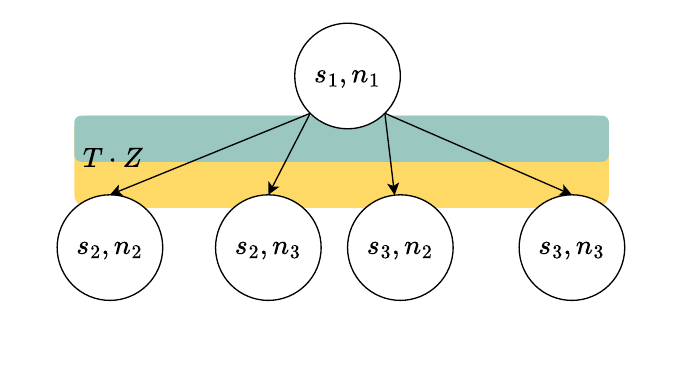}
        \caption{A product Markov chain}
        \label{fig:MC_product}
    \end{subfigure}
    \begin{subfigure}[b]{0.49\linewidth}
        \centering
        \includegraphics[width=0.8\linewidth]{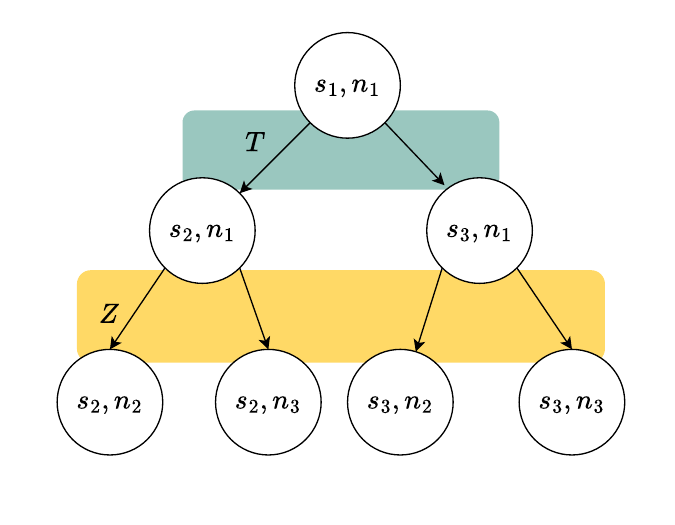}
        \caption{A two-step Markov chain}
        \label{fig:MC_twostep}
    \end{subfigure}
    \caption{Markov chain representations}
    \label{fig:MCs}
    \Description{A figure showing two Markov chain representations considered in this work. The first is a product Markov chain in which each transition is dependent on both a state transition and node transition (which is a function of the observation model). The second is a two-step Markov chain in which state transitions and node transitions are separated. Here, state transitions occur first followed at the next step by node transitions.}
\end{figure}

As an alternative to product Markov chains, we propose a two-step interval Markov chain (tsiMC) shown in \cref{fig:MC_twostep} that considers state transitions independently from observations, maintaining the probability simplex requirement without additional constraints. Since $T$ and $Z$ are treated as individual transitions in this representation, each is constrained to respect the probability simplex requirement. We now formally define a two-step Markov chain (tsMC) and a two-step interval MC and show how to solve the inner problem~\cref{eq:inner} using a tsiMC.

\begin{definition}[Two-Step Markov Chain]
\label{def:tsMC}
Given a POMDP model $\mathcal{M}$, with a node-augmented observation function $\Z$, an FSC $\pi$, and an evaluation horizon $d$, an induced two-step Markov chain is defined $\MC_{\twoStep} = (q_I, Q_{\twoStep}, T_{\twoStep}, R_{\twoStep}, \gamma_{\twoStep},d_{\twoStep})$, where

\begin{itemize}
    \item $q_I = (s_I, n_0)$ is the initial Markov chain state, where $s_I$ is an augmented initial state used to capture the initial belief distribution in the tsMC transition function;
    \item $Q_{\twoStep} = ( S \times N \times \{0,1\}) \cup \{q_I\}$ is the product of the POMDP states, controller nodes, and an indicator set;
    \item The transition function $T_{\twoStep}: Q_{\twoStep} \times Q_{\twoStep} \rightarrow [0,1]$ is defined
        \begin{align*}
        &T_{\twoStep}((s', n', i') | (s,n, i)) = \begin{cases}
            T(s' | s, \alpha(n))  &  \text{if } s \neq s_I, n' = n, i = 0, i' = 1 \\
            \Z(\mathcal{O}(n,n') | \alpha(n), s', n) & \text{if }s \neq s_I, s = s', i = 1, i' = 0\\
            b_0(s') & \text{if } s = s_I, n' = n_0, i' = 0\\ 
            0 & \text{o.w.}
        \end{cases};
    \end{align*}
    \item The reward function $R_{\twoStep}: (S \times N \times \{0,1\}) \rightarrow \mathbb{R}$ is defined as $R_{\twoStep}((s, n, i)) = 
    \begin{aligned}
        \begin{cases}
                R(s, \alpha(n)) & \text{ if } i = 0\\
                0 & \text{o.w.}
        \end{cases}
    \end{aligned}$;
    \item The discount factor $\gamma_{\twoStep} = \gamma^{0.5}$;
    \item The horizon $d_{\twoStep} = 2d$;
\end{itemize}
\end{definition}
Note the discount factor and horizon definitions account for the fact that only every other transition in the two-step MC is representative of a state transition.

Define the value of the two-step Markov chain as:
\begin{equation}\label{eq:tsMCval}
    \begin{aligned}
         V^{\MC_{\twoStep},d_{\twoStep}}&(q) = R_{\twoStep}(q) + \gamma_{{\twoStep}} \sum_{q'\in Q}T_{{\twoStep}}(q'|q)V^{\MC_{\twoStep},d_{\twoStep}-1}(q') \; \forall q\in Q.
    \end{aligned}
\end{equation}

We now show that this two-step Markov chain can be used to find the value of an FSC on a POMDP. This result allows us to show that tsiMCs can be used to solve \cref{eq:innernode}.

\begin{restatable}[]{lemma}{mcequivalence}
    \label{prop:mcequivalence}
    Let $\mathcal{M}$ be a POMDP with a finite-state controller policy $\pi$. Let $\MC_{\twoStep}$ be the induced two-step Markov chain constructed with $\mathcal{M}, \pi$. For any horizon $d$, we have that
    \begin{equation}
        V^{\pi,d}(h_0) = V^{\MC_{\twoStep},d_{\twoStep}}(q_I).
    \end{equation} 
\end{restatable} 

\begin{proof}[Proof Sketch]
    Given the same initial state, by construction any state-augmented history
    $h_t=(b_0,s_0,a_1,s_1,o_1,\allowbreak ...,a_{t-1},s_{t-1},o_{t-1})$ reached in the POMDP $\mathcal{M}$ by following FSC $\pi$ can be reached in the two-step Markov chain $\MC_\twoStep$ constructed with $\mathcal{M},\pi$ and vice versa. Likewise, given the same initial state, the probability of reaching these state-augmented histories is the same. Then, the discounted reward assigned to these histories will be the same in both models. This equivalence of models follows in a similar manner to existing theorems for product Markov chains \cite{meuleau_solving_2013}. 
\end{proof}

To analyze a policy performance over $\delta$, we now consider a generalization of the two-step Markov chain to intervals. By equipping the observation function $\Z$ with intervals, we obtain a two-step interval Markov chain. 

\begin{definition}[Two-Step Interval Markov Chain]\label{def:tsimc}
Given a POMDP $\mathcal{M}$ model, with a node-augmented observation function $\Z$, an FSC $\pi$, a horizon $d$, and an admissible observation deviation $\delta$, an induced two-step interval Markov chain is defined as 

\noindent$\MC^{\updownarrow}_{\twoStep} = (q_I,Q_{\twoStep}, T^{\downarrow}_{\twoStep},  T^{\uparrow}_{\twoStep}, R_{\twoStep}, \gamma_{\twoStep},d_{\twoStep})$, where
\begin{itemize}
     \item  $T^{\downarrow}_{\twoStep}: Q_{\twoStep} \times Q_{\twoStep} \rightarrow [0,1]$ defines the lower bound of the transition probability function, such that
    \begin{align*}
        &T^{\downarrow}_{\twoStep}((s', n', i') | (s,n, i)) = \begin{cases}
            T(s' | s, \alpha(n)) & \text{if } s\neq s_I, n' = n, i = 0, i' = 1\\
            \max(\Z(\mathcal{O}(n,n') | \alpha(n), s', n) - \delta, \epsilon_p) & \text{if } s\neq s_I, s = s', i = 1, i' = 0\\
            b_0(s') & \text{if } s = s_I, n' = n_0, i' = 0\\
            0 & \text{o.w.}
        \end{cases},
    \end{align*}
    where $\epsilon_p$ is the minimum probability, ensuring graph-preservation as in \citep{spel_finding_2021};
     \item $T^{\uparrow}_{\twoStep}: Q_{\twoStep} \times Q_{\twoStep} \rightarrow [0,1]$ defines the upper bound of the transition probability function, such that
     \begin{align*}
        &T^{\uparrow}_{\twoStep}((s', n', i') | (s,n, i)) = \begin{cases}
            T(s' | s, \alpha(n)) &\text{if } s\neq s_I, n' = n, i = 0, i' = 1\\
            \min(\Z(\mathcal{O}(n,n') | \alpha(n), s', n) + \delta, 1) &\text{if } s\neq s_I, s = s', i = 1, i' = 0, \Z(o | \alpha(n), s') > 0\\
            b_0(s') & \text{if } s = s_I, n' = n_0, i' = 0\\
            0 &\text{o.w.}
        \end{cases};
    \end{align*}
\end{itemize}
\noindent and all other terms are as in \cref{def:tsMC}.
\end{definition}

Given this tsiMC, we desire to find the worst-case (minimum value) tsMC on the interval as a means of finding the worst-case observation deviation. We will show in \cref{sec:AlgorithmNS} that this formulation is tractable.
\rev{Let $D^d_{\twoStep}$ denote the tsMC extension of \cref{eq:fhZdomain,eq:ihZdomain}, i.e. the set of two-step Markov chain transition probabilities being minimized over in either the finite or infinite horizon case. In this case, each element, $T^d_{\twoStep}$, must be in the interval $\left[T_{\twoStep}^{\downarrow},T_{\twoStep}^{\uparrow}\right]$.} 

\begin{restatable}[]{theorem}{ipeinner}
\label{thm:ipe_inner}
    Given a POMDP $\mathcal{M}$, FSC policy $\pi$, a horizon $d$, and $\delta$, construct $\MC^{\updownarrow}_{\twoStep}$. Then, 
    \begin{equation}
       \min_{D^d_{\Z}} \: V^{\pi,d}_{\Z}(h_0) = \min_{D^d_{\twoStep} 
       } \: V^{\MC_{\twoStep},d_{\twoStep}}_{T_{\twoStep}}(q_I).
    \end{equation}
\end{restatable}

\begin{proof}[Proof]
    By \cref{prop:mcequivalence}, any realized $\MC_{\twoStep}  \in \MC^{\updownarrow}_{\twoStep}$ from the interval is equivalent in value to the composition of a POMDP $\mathcal{M}$ and an FSC $\pi$. Then if the bounds on the two problems are the same, the minimum value will be the same.
    
    For ease of notation and without loss of generality, consider the case where $d=\infty$. The bounds on $\mathbf{\Z}=\prball{\Z_0,\delta}$, the probability ball with radius $\delta$ around $\Z$, \rev{for any given $o,a,s',n$} are
    \begin{equation}
        \Z^{\downarrow}(o|a,s',n)=\max(\Z_0(o | a, s', n) - \delta, \epsilon_p)
    \end{equation}
    and 
    \begin{equation}
        \Z^{\uparrow}(o|a,s',n)=\
        \begin{aligned}
            \begin{cases}
                \min(\Z_0(o | a, s', n) + \delta, 1) & \text{if } \Z(o | a, s', n) > 0\\
                0 &\text{o.w.}
            \end{cases} .
        \end{aligned}
    \end{equation}
    
    By construction, $\Z^{\downarrow}, \Z^{\uparrow}$ are bounds on all possible valid node-dependent observation probability functions in $\mathbf{\Z}$, although they themselves are not valid distributions. The transition distributions as defined in $T^{\downarrow}_{\twoStep}$ from the MC states with $i=1$ are equivalent to observation distributions in $\Z^{\downarrow}$ and likewise for $T^{\uparrow}_{\twoStep}$ and $\Z^{\uparrow}$. Thus, both problems have the same bounds on transitioning to new node $n'$ from node $n$ as a function of an observation $o$.
\end{proof}

Given this result, we can now formulate \cref{eq:nested2node} with a tsiMC in the inner optimization problem. 
Having established the equivalence of the value of the two-step Markov chain and the FSC on a POMDP, \cref{eq:nested2} can be written as
 \begin{subequations}\label{eq:mcprob}
    \begin{align}
        & \begin{aligned}\label{eq:mcproba}
            \delta_{NS} = \max_{\delta\in[0,1]} & \: \delta\\
            \text{s.t. } & V^{\MC_{\twoStep},d_{\twoStep}}_{T_{\twoStep_0}}(q_I)-V^{\MC_{\twoStep},d_{\twoStep}}_{T_{\twoStep}'}(q_I) \leq \threshold\\ 
        \end{aligned}\\
        &\begin{aligned}\label{eq:mcprobb}
              & \qquad \qquad \quad V^{\MC_{\twoStep},d_{\twoStep}}_{T'_{\twoStep}}(q_I) = \min_{ D^d_{T_{\twoStep}}
              } \: V^{\MC_{\twoStep},d_{\twoStep}}_{T_{\twoStep}}(q_I),\\
        \end{aligned}
    \end{align}
\end{subequations}

where $V^{\MC_{\twoStep}}_{T_{\twoStep}}(q_I)$ denotes the value of the two-step Markov chain, \cref{eq:tsMCval}, constructed with some $T_{\twoStep}$. To solve \cref{eq:mcprob}, we must search over $\delta$ while respecting the threshold constraints. As established in \cref{sec:monotonicity}, this inner minimization is monotonic in $\delta$. We state this formally in the following lemma.

\begin{lemma}\label{prop:minmono}
    The minimum value over the two-step interval Markov chain defined by $\delta$, \cref{eq:mcprobb}, is monotonically non-increasing in $\delta$.
\end{lemma}

Intuitively, as the size of the permissible uncertainty set expands with $\delta$, the worst-case minimum can only decrease.
The monotonicity of \cref{eq:mcprobb} in $\delta$ allows us to use the efficient modified bisection method to find a solution to \cref{eq:mcprob} while guaranteeing convergence and soundness.

\subsection{\algoName{}}\label{sec:AlgorithmNS}
In this section, we instantiate RIS for the non-sticky formulation by specifying the corresponding feasibility evaluator. The resulting algorithm, \algoName{} (\shortAlgoName{}), shown in \cref{alg:overall_nonsticky}, computes the maximal admissible deviation $\delta$ for \cref{prob:non-sticky} via interval search. When the inner 
evaluation (cf.\ \cref{eq:mcprobb}) is performed exactly, \shortAlgoName{} 
returns a $\delta$ within $\epsilon_{\modifiedBisectionShort{}}$ of the 
optimal admissible deviation $\delta^*$ defined in \cref{eq:mcprob}.

The algorithm constructs a two-step interval Markov chain $\MC_{\twoStep}$ 
from the POMDP--FSC pair and evaluates the nominal policy value 
$V^{\MC_{\twoStep},d_{\twoStep}}_{T^0_{\twoStep}}$. For a candidate 
radius $\delta$, the feasibility evaluator
\begin{equation}\label{eq:f_TZ}
    f_{TZ}(\delta) =V^{\MC_{\twoStep},d_{\twoStep}}_{T^0_{\twoStep}}(q_I)-\threshold-\min_{D^d_{\twoStep} }V^{\MC_{\twoStep},d_{\twoStep}}_{T'_{\twoStep}}(q_I).
\end{equation}
is non-positive when the robustness constraint is satisfied at radius $\delta$ 
and positive otherwise. \rev{Note $D^d_{\twoStep}$ is a function of $\delta$ (cf. \cref{def:tsimc}).} This function, together with the search interval 
$[0,1]$ and tolerance $\epsilon_{\modifiedBisectionShort{}}$, is passed 
to \modifiedBisectionShort{} (\cref{alg:MBS}).

\begin{algorithm}
    \caption{\algoName{}}
    \hspace{-7.9cm} \texttt{\shortAlgoName{}}($\mathcal{M},\pi,d,\threshold, \epsilon_{\modifiedBisectionShort{}}$)
    \begin{algorithmic}[1]\label{alg:overall_nonsticky}
        \STATE  $\MC_{\twoStep} \gets $\texttt{make\_tsMC} $(\mathcal{M,\pi})$
        \STATE $d_{\twoStep} \gets 2d$
        \STATE $V^{\MC_{\twoStep},d_{\twoStep}}_{T^0_{\twoStep}} \gets $\texttt{evaluate}(\texttt{make\_tsiMC}$(\MC_{\twoStep},0) )$ \COMMENT{\cref{eq:mcprobb}}
        \STATE $\delta_a \gets 0$
        \STATE $\delta_b \gets 1$
        \STATE $f_{TZ}(\delta) \gets V^{\MC_{\twoStep},d_{\twoStep}}_{T^0_{\twoStep}}-\threshold- $\texttt{evaluate} (\texttt{make\_tsiMC}$(\MC_{\twoStep},\delta) )$ \COMMENT{\cref{eq:f_TZ}}
        \RETURN \texttt{\modifiedBisectionShort{}}($f_{TZ},\delta_a,\delta_b,\epsilon_{\modifiedBisectionShort{}}$) \COMMENT{\cref{alg:MBS}}
    \end{algorithmic}
\end{algorithm}

We next provide analysis of \shortAlgoName{} and its convergence.

\subsection{Analysis of \shortAlgoName{}}\label{sec:NSAnalysis}
We show that \shortAlgoName{} is sound and guaranteed to converge to within $\epsilon_{\modifiedBisectionShort{}}$ of $\delta^*$ and does so in polynomial time when \cref{eq:mcprobb} is solved exactly.
We first provide a lemma on the monotonicity of $f_{\twoStep}$, which underpins the convergence of \algoName{}.

\begin{lemma}\label{prop:fz}
    The non-sticky feasibility evaluator $f_{\twoStep}$, \Cref{eq:f_TZ}, is monotone non-decreasing in $\delta$.
\end{lemma}

This lemma follows from the monotonicity of \cref{eq:mcprobb} in $\delta$ (\cref{prop:minmono}) and the construction of $f_{\twoStep}$. We now provide the key result on the convergence and soundness of \shortAlgoName{} with exact evaluation.

\begin{restatable}[]{theorem}{AlgoThm}
\label{thm:AlgoThm}
    Let $\delta^*$ be the unique optimal solution to the optimization problem in \cref{eq:mcprob}.
    
    Given exact solutions to \cref{eq:mcprobb}, \algoName{} (\Cref{alg:overall_nonsticky}) returns a solution $\delta$ such that $V^{\MC_{\twoStep},d_{\twoStep}}_{T_{\twoStep_0}}(q_I)-V^{\MC_{\twoStep},d_{\twoStep}}_{T'_{\twoStep}}(q_I) \leq \threshold$ and $|\delta^*-\delta| \leq \epsilon_{\modifiedBisectionShort{}}$.
\end{restatable}

\begin{proof}\label{proof:AlgoThm}
    By \cref{prop:fz}, $f_{TZ}$ as defined in \cref{eq:f_TZ} is monotone \rev{non-decreasing} in $\delta$. It has $f_{TZ}(0) \leq 0$, is well-defined, and, by assumption, returns an exact value. Then, by \cref{lm:bisect_sound} and \cref{lm:bisectconverge} modified bisection search on $f$ will converge to $\delta$ within $\epsilon_{MBS}$ of $\delta^*$ which ensures $V^{\MC_{\twoStep},d_{\twoStep}}_{T_{\twoStep_0}}(q_I)-V^{\MC_{\twoStep},d_{\twoStep}}_{T'_{\twoStep}}(q_I) \leq \threshold$.
\end{proof}

These guarantees extend to \cref{prob:non-sticky} as a result of \cref{thm:ipe_inner}.

\begin{restatable}[]{corollary}{SolveOriginalS}
    Let $\delta^*$ be the unique optimal solution to the optimization problem in \cref{eq:prob_nonsticky}.
     Given exact solutions to \cref{eq:mcprobb}, \algoName{} (\Cref{alg:overall_nonsticky}) returns a solution $\delta$ such that $V^{\pi,d}_{\ZH_{0}}(h_0)-V^{\pi,d}_{\ZH'}(h_0) \leq \threshold$ and $|\delta^*-\delta| \leq \epsilon_{\modifiedBisectionShort{}}$.
\end{restatable}

Thus, \shortAlgoName{} with exact feasibility evaluation solves the non-sticky variant of \problemName{} problem with soundness and convergence guarantees.

Next we provide an analysis of the time complexity of \shortAlgoName{} for both the infinite- and finite-horizon cases. We begin by establishing the complexity of \cref{eq:mcprobb} by showing it can be posed as a linear program. \Cref{eq:mcprobb} may be written as an extension of the MDP value linear program \citep{puterman2014markov} where the transition probabilities are now free variables and there are no actions. \rev{This form is similar to those discussed in works by \citet{iyengar2005robust} and \citet{nilim2005robust}.} 
Without loss of generality, for the infinite-horizon case,
\begin{equation}\label{eq:BLP}
     \begin{aligned}
     V^{\MC_{\twoStep},\infty}_{T'_{\twoStep}}(q_I) = \min_{D_{BLP}} & \: V^{\MC_{\twoStep},\infty}(q_I)\\
            \text{s.t.} & \: V^{\MC_{\twoStep},\infty}(q) \geq R_{\twoStep}(q)+\gamma\sum_{q' \in Q}T_{\twoStep}(q'|q)V^{\MC_{\twoStep},\infty}(q') \; \forall q \in Q \; \\ 
            & \:\sum_{q'\in Q} T_{\twoStep}(q'|q) = 1 \; \forall q \in Q \;.\\ 
    \end{aligned}
\end{equation}
where $D_{BLP} = \{ V^{\MC_{\twoStep},\infty}(q_1) \in \mathbb{R},...,V^{\MC_{\twoStep},\infty}(q_{|Q|})\in \mathbb{R},T_{\twoStep}\in[T^{\downarrow}_{\twoStep},T^{\uparrow}_{\twoStep}]\}$.
This bilinear problem can be converted to a linear program using a transformation of $T_{\twoStep}$ following the approach established in work by \citet{mazouz2024piecewise}.

\begin{restatable}[]{proposition}{IPEComplexLeq}
\label{thm:IPEComplexLeq}
   \algoName{} (\cref{alg:overall_nonsticky}) is polynomial in the number of states $|Q|$ and bits used to represent the two-step Markov chain for $\gamma < 1$. 
\end{restatable}
\begin{proof}[Proof Sketch]
    For the case where $\gamma < 1$, linear programs may be solved in time which is polynomial in the number of variables, constraints, and bits used to represent these. Here, the number of variables and constraints are polynomial in $|Q|$.
  
    Bisection search is known to be logarithmic in the inverse of the tolerance $\frac{1}{\epsilon_{\modifiedBisectionShort{}}}$ and the interval width \citep{burden38216numerical}. Then, since a logarithmic number of inner optimization calls are required, \shortAlgoName{} will be polynomial in the number of tsiMC states $|Q|$ and bits used to represent the parameters of \rev{the} two-step Markov chain.
\end{proof}

\begin{restatable}[]{proposition}{IPEComplexEq}
\label{thm:IPEComplexEq}
   For a horizon expressed in unary, \algoName{} (\cref{alg:overall_nonsticky}) is polynomial in the horizon, the number of states $|Q|$, and bits used to represent the two-step Markov chain for $\gamma=1$.
\end{restatable}
\begin{proof}[Proof Sketch]
    It is straightforward to extend \cref{eq:BLP} to the finite-horizon case with $\gamma=1$, where there is a value function for each state $q \in Q$ and \rev{each} timestep less than horizon $d$. Then, \cref{thm:IPEComplexEq} follows from the sketch for \cref{thm:IPEComplexLeq}.
\end{proof}

\rev{Having established the soundness, convergence, and polynomial complexity of \shortAlgoName{}, we now evaluate implementations of both algorithms, conducting validation and scalability experiments before demonstrating their use on several case studies.}

\section{Evaluation}\label{sec:Evaluation}
To the best of our knowledge, this work is the first to propose and solve the \ProblemName{} Problem.  
Although our problem is similar to the parameter space partitioning problem in \citet{junges_parameter_2024} and other works, these methods
would require additional post-processing to solve our robustness problem. Thus, there are no existing algorithms that directly solve the same problem to compare to.
As such, we first provide validation results for the approximate inner optimization implementations of our algorithms which show the interval found meets the value constraint $\threshold$, and then scalability results for \shortAlgoName{} with IPE and \shortAlgoNameSticky{} with PLA. We then provide relevant case studies which demonstrate the applicability of our algorithms and highlight the difference between \cref{prob:non-sticky} and \cref{prob:sticky}. Problem \rev{instances are implemented using} the POMDPs.jl package \citep{egorov2017pomdps} in Julia (v1.9.0), IntervalMDPs.jl \citep{mathiesen2024intervalmdp} is used for IPE, and STORM (v1.9.0) \citep{dehnert_storm_2017} is used for PLA in C++ (via Docker, limited to 23GB of RAM) with calls from Julia. All experiments were run on an Intel Core i7-11700K CPU with 32GB system RAM with Ubuntu 24.04. Unless noted, the threshold for modified bisection search convergence is $\epsilon_{\modifiedBisectionShort{}}=10^{-7}$, the precision for IPE is $\epsilon_{IPE}=10^{-7}$, and the precision for PLA is $\epsilon_{PLA}=10^{-7}$ for all experiments.

\subsection{Implementation Details}

\paragraph*{\shortAlgoName{ }} For the Non-Sticky variant, we use dynamic programming to approximate the value of \cref{eq:f_TZ} to within $\epsilon_{IPE}$. 
For this we use a well-established, efficient procedure called Interval Policy Evaluation (IPE) \citep{givan_bounded-parameter_2000}.
IPE computes the upper and lower bounds on the value of a policy on an interval Markov decision process, and by extension the upper and lower bounds on the value of an interval Markov chain. It adapts the value iteration algorithm of \citet{bellman57} by using a modified Bellman backup equation. In each iteration, this backup calculates the upper bound by optimistically (or pessimistically) selecting transition probabilities from within their prescribed intervals to maximize (or minimize) the resulting value. \rev{Due to approximate value calculations, the returned $\delta$ may exceed the $\epsilon_{MBS}$ tolerance. Nevertheless, this implementation ensures the value under the deviation exceeds the threshold by at most the inner optimization precision, i.e., $V^{\MC_{\twoStep},d_{\twoStep}}_{T_{\twoStep_0}}(q_I)-V^{\MC_{\twoStep},d_{\twoStep}}_{T'_{\twoStep}}(q_I) \leq \threshold - \epsilon_{IPE}$}.

In order to simplify construction of $\MC^{\updownarrow}_{\twoStep}$, we build a tsiMC with all possible states in $Q_{\twoStep}$, which may lead to states which are not reachable from $q_I$ having incomplete distributions over $q'$ as arbitrary nodes and states are not guaranteed to share observation support and some $q$ are never reached by another state. In order to ensure valid transition distributions for these non-reachable states and a valid MC for use in IPE with \citet{mathiesen2024intervalmdp}, we modify $T_{\twoStep}$ in the following manner: In the case where $q$ has zero outbound transition probability, we add a transition $T_{\twoStep}(q'|q)=1.0$. In cases where transition probabilities do not sum to 1, we normalize these distributions to ensure they are valid.

\paragraph*{\shortAlgoNameSticky{}} For the Sticky variant, \cref{eq:f_p} in \cref{alg:overall_sticky} (\shortAlgoNameSticky{}) -- \rev{checking if all parameters in $P(Z_0,\delta)$ satisfy the inequality $V^{\MC_{\parameterized},d}_{p^0}(q_I)- V^{\MC_{\parameterized},d}_{p'}(q_I)\leq \threshold$} -- is solved approximately by finding the parameters which minimize the value of the pMC to a threshold of $\epsilon_{PLA}$, in a manner similar to the approach used for \shortAlgoName{}. \rev{As noted in \cref{sec:bilevels}, when the inner optimization is solved exactly, this representation is equivalent to the original problem. For the problems we considered, we found this approach yielded faster performance than a feasibility check of the full inequality, \cref{eq:mcproba_sticky_con}.} We use the well-established Parameter Lifting Algorithm (PLA)  \rev{implementation} \citep{junges_parameter_2024,spel_finding_2021} in \rev{the probabilistic model checker} STORM \citep{dehnert_storm_2017} due to its comparative ability to scale to higher numbers of parameters in pMCs. 
PLA \citep{10.1007/978-3-319-46520-3_4,junges_parameter_2024} finds the worst-case parameter instantiation over a parameter region (intervals over parameter values). 
\rev{This is accomplished using parameter lifting to simplify analysis by eliminating dependencies on parameters which are shared between multiple transitions, transforming the pMC into an MDP which provides a conservative approximation, and iteratively refining the subregions over which this lifting is done to reduce the conservativeness of the approximation to within some tolerance \citep{junges_parameter_2024}.} PLA operates on pMCs which \rev{are multi-affine in} parameters \cite{10.1007/978-3-319-46520-3_4}. \rev{As in \citet{spel_finding_2021}, we consider simple pMCs, a subclass of multi-affine pMCs to which general pMCs can be converted via a process similar to that described for POMDPs in \citet{junges_finite-state_nodate}.}

\rev{Due to approximate value calculations, the returned $\delta$ may exceed the $\epsilon_{MBS}$ tolerance. Nevertheless, this implementation ensures the value under the deviation exceeds the threshold by at most the inner optimization precision, i.e., $V^{\MC_{\parameterized},d}_{p^0}(q_I)-V^{\MC_{\parameterized},d}_{p'}(q_I) \leq \threshold - \epsilon_{PLA}$.}

\subsection{\shortAlgoName{} Validation}
In order to validate \shortAlgoName{}, we consider several standard POMDP benchmarks: Tiger POMDP \citep{Kaelbling1998pomdp}, RockSample \citep{smith2004heuristic} (with a $5\times5$ grid and 3 rocks), and BabyPOMDP \citep{kochenderfer2015decision}. FSCs are constructed from SARSOP \citep{kurniawati_sarsop_2008} $\epsilon_{\pi}$-optimal policies, with $\epsilon_{\pi}=10^{-5}$. Solutions are found with \shortAlgoName{} with $d=\infty$ and $\epsilon_p=0$.

\subsubsection{\shortAlgoName{} Benchmark Models}\label{sec:NS_Benchmark}
Results are shown in \cref{tab:validation}. In all environments, we determine $\delta$ with a threshold defined: $\threshold = \eta\left|V^{\MC_{\twoStep}}_{T_{\twoStep}^0}\right|$. Given $\delta$ as found with \shortAlgoName{}, we find a worst-case two-step Markov chain $T'_{\twoStep}$ using IPE and calculate its empirical degradation $\eta_s(T'_{\twoStep}) = \frac{V^{\MC_{\twoStep}}_{T^0_{\twoStep}}-V^{\MC_{\twoStep}}_{T'_{\twoStep}}}{\left| V^{\MC_{\twoStep}}_{T^0_{\twoStep}}\right|}$ using Value Iteration \citep{bellman57}. Additionally, we randomly sample 100,000 (200,000 for BabyPOMDP) $T'_{\twoStep}$ from the set of extrema probabilities defined by $\delta$, calculate $\eta$ using Value Iteration, and show the worst-case $\eta$ from the samples ($rand(T'_{\twoStep})$). Likewise, we randomly sample 100,000 (200,000 for BabyPOMDP) $Z'$ (the sticky \rev{variant}) and shown the worst-case ($rand(Z')$) where values are calculated using Value Iteration on the tsMC constructed from $\pi$ and $Z'$. Here $T'_{\twoStep}$ are candidates to the non-sticky minimizer and $Z'$ are candidates to the sticky solution. Extrema are chosen as these are the candidate points at which value is minimized \citep{givan_bounded-parameter_2000}. 

\begin{table}[ht]
    \centering
    \caption{ \shortAlgoName{} Benchmark Models - Target $\eta$, resulting $\delta$, and empirical $\eta_s$ for the resulting minimizing matrix $T'_{\twoStep}$, the lowest value $T'_{\twoStep}$ from randomly sampled $T_{\twoStep}$ at the extrema of $\delta$, and the lowest value $Z'$ from randomly sampled $Z$ at the extrema of $\delta$.}
    \begin{tabular}{c|cc|ccc}
        \toprule
        \multirow{2}{*}{Domain} & \multirow{2}{*}{$\eta$} & \multirow{2}{*}{$\delta$}  & \multicolumn{3}{c}{$\underline{\hspace{10mm} \eta_s \hspace{10mm}}$}\\
          & &  & $T_{\twoStep}'$&  $rand(T_{\twoStep}')$ & $rand(Z')$\\
        \midrule
        \midrule
        \multirow{5}{*}{Tiger}     & 0.05 & 0.0031 & 0.05 & 0.05 & 0.05\\
             & 0.25 & 0.0149 & 0.25 & 0.25 & 0.25\\
             & 0.45 & 0.0260 & 0.45 & 0.45 & 0.45\\
             & 0.65 & 0.0367 & 0.65 & 0.65 & 0.65\\
             & 0.85 & 0.0469 & 0.85 & 0.85 & 0.85\\
        \midrule
        \multirow{5}{*}{RS(5,5)}  & 0.05 & 0.1413 & 0.0500 & 0.0500& 0.0500\\
             & 0.25 & 0.7063 &  0.2500 & 0.2500 &  0.2500\\
             & 0.45 & 1.0000 &  0.3479 & 0.3479 & 0.3479\\
             & 0.65 & 1.0000 &  0.3479 & 0.3479 & 0.3479\\
             & 0.85 & 1.0000 &  0.3479 & 0.3479 & 0.3479\\
        \midrule
        \multirow{5}{*}{Baby}    & 0.05 & 0.0268 & 0.0500 & 0.0491 & 0.0500\\
            & 0.25 & 0.1405 &  0.2500 & 0.2474 & 0.2500\\
            & 0.45 & 0.2622 &  0.4500 & 0.4492 & 0.4500\\
            & 0.65 & 0.3848 &  0.6500 & 0.6492 & 0.6500\\
            & 0.85 & 0.4980 &  0.8500 & 0.8488 & 0.8500\\
        \bottomrule
    \end{tabular}
    \label{tab:validation}
\end{table}

These results demonstrate that \shortAlgoName{} is successful in finding an interval $\delta$ for which the minimizing $T'_{\twoStep}$ found by IPE and evaluated with Value Iteration does one of two things: achieves the desired $\eta$, or achieves the worst-case $\eta$ for the problem \rev{instance} if it is less than the desired $\eta$, as in RockSample. Likewise, the same $\eta$ is achieved by the worst-case $T'_{\twoStep}$ randomly sampled from the same $\delta$, showing that IPE does select a valid minimizer. Thus, the $\delta$ selected does achieve the specified or problem \rev{instance} worst-case $\eta$. Additionally, the worst-case randomly sampled $Z'$ (which is strictly dependent on $a$, $s'$) achieves the desired or problem \rev{instance} worst-case $\eta$. This indicates that, in these cases, the non-sticky solution is not conservative, but does find a worst-case POMDP model for the specified $\eta$.

\subsubsection{\shortAlgoName{} Scalability}
To demonstrate the scalability of \shortAlgoName{}, we find $\delta$ using $\eta=0.1$ for RockSample domains of varying size with corresponding infinite-horizon policies and finite-horizon TigerPOMDP domains of varying horizon length with their corresponding finite-horizon policies. FSCs here are generated from a SARSOP policy generated with a 1,000 iteration limit in addition to the $\epsilon_{\pi}$ criteria listed above. Results are shown in \cref{tab:scalability}.

\begin{table}[ht]
    \centering
    \caption{\shortAlgoName{} Scalability - The number of nodes in $\pi$ $|N|$, number of states in $\mathcal{M}$ $|S|$, size of the tsMC state space $|Q_{\twoStep}|$, time to find a solution (from one run), and solution $\delta$ for $\eta=0.1$.}
    \begin{tabular}{c|cc|c|c|c}
        \toprule
        Domain  & $|N|$ & $|S|$ & $|Q_{\twoStep}|$ & time (sec) & $\delta$ \\
        \midrule
        \midrule
        RS(3,3)  & 13 & 73 & 1899 & 1.4223 &  0.2798 \\
        RS(4,4)  & 17 & 257 & 8,739 &  0.3642 & 0.3654\\
        RS(5,5)  & 21 & 801 & 33,643 & 1.6772 &  0.4678 \\
        RS(6,6)  & 25 & 2,305 & 115,251 &  8.3200 & 0.5903 \\
        RS(7,7)  & 30 & 6,273 & 376,381 & 350.7606 & 0.7361  \\
        RS(8,8)  & 47 & 16,385 & 1,540,191 &  1,843.6515 & 0.4235 \\
        \midrule
        $\text{Tiger-3}$ & 9 & 8 & 145 & 0.0021 & 0.0056\\
        $\text{Tiger-4}$ & 15 & 10 & 301 & 0.0045 & 0.0052\\
        $\text{Tiger-5}$ & 16 & 12 & 385 & 0.0061 & 0.0093\\
        $\text{Tiger-6}$ & 26 & 14 & 729 & 0.0131 & 0.0061\\
        $\text{Tiger-7}$ & 23 & 16 & 737 & 0.0151 & 0.0090\\
        \bottomrule
    \end{tabular}
    \label{tab:scalability}
\end{table}

For small problem \rev{instances}, \shortAlgoName{} finds solutions in seconds and is capable of solving problem \rev{instances} with tsMC statespaces in the millions. We note that for anticipated use cases, longer run times are acceptable, as validation is not likely to be completed in an online fashion. However, the short run time of our algorithm enables rapid design iteration for policies.

\subsection{\shortAlgoNameSticky{} Validation}
In order to validate our \shortAlgoNameSticky{}, we consider finite-horizon, undiscounted evaluation of several standard POMDP benchmarks: Tiger POMDP \citep{Kaelbling1998pomdp}, RockSample \citep{smith2004heuristic} (with a $3\times3$ grid and 3 rocks), and BabyPOMDP \citep{kochenderfer2015decision}. FSCs are constructed from SARSOP \citep{kurniawati_sarsop_2008} $\epsilon_{\pi}$-optimal policies, with $\epsilon_{\pi}=10^{-5}$, on the infinite-horizon \rev{instance of each} problem. $\epsilon_p=0.01$ in all results.

\subsubsection{\shortAlgoNameSticky{} Benchmark Models}
Results are shown in \cref{tab:validation_pla}. In all environments, we find $\delta$ given a threshold defined: $\threshold = \eta\left|V^{\MC_{\parameterized}}_{p^0}\right|$. 
Using this $\delta$, we find a pMC with worst-case $p'$ and its value using PLA and calculate the degradation $\eta_s(p') = \frac{V^{\MC_{\parameterized}}_{p^0}-V^{\MC_{\parameterized}}_{p'}}{\left|V^{\MC_{\parameterized}}_{p^0}\right|}$. 
Additionally, we sample 1,000 $Z$ from the set of extrema probabilities defined by $\delta$ and select the worst-case $\eta_s(Z')$, as in \cref{sec:NS_Benchmark}, where the value is calculated using IPE with zero interval on a tsMC constructed from $\mathcal{M}_Z'$ and $\pi$. $h$ denotes the evaluation horizon.

\begin{table}[ht]
    \centering
    \caption{\shortAlgoNameSticky{} Benchmark Models - Target $\eta$, resulting $\delta$, and empirical $\eta_s$ for the resulting minimizing parameter set $p'$ and the lowest value from $Z'$ randomly sampled from the extrema of $\delta$. $h$ denotes the evaluation horizon.}
    \begin{tabular}{c|cc|cc}
        \toprule
        \multirow{2}{*}{Domain} & \multirow{2}{*}{$\eta$} & \multirow{2}{*}{$\delta$}  & \multicolumn{2}{c}{$\underline{\hspace{10mm} \eta_s \hspace{10mm}}$}\\
          & &  & $p'$& $rand(Z')$\\
        \midrule \midrule
        \multirow{5}{*}{Tiger (h=21)}     
             & 0.05 & 0.0031 & 0.05 & 0.05 \\
             & 0.25 & 0.0149 & 0.25 & 0.25\\
             & 0.45 & 0.0261 & 0.45 & 0.45\\
             & 0.65 & 0.0368 & 0.65 & 0.65\\
             & 0.85 & 0.0470 & 0.85 & 0.85\\
        \midrule
        \multirow{5}{*}{RS(3,3) (h=21)}  
             & 0.05 & 0.1241 & 0.0500 & 0.0500\\
             & 0.25 & 0.6207 &  0.2500 & 0.2500 \\
             & 0.45 & 1.000 &  0.3919 & 0.3919 \\
             & 0.65 & 1.000 &  0.3919 & 0.3919 \\
             & 0.85 & 1.000 &  0.3919 & 0.3919 \\
        \midrule
        \multirow{5}{*}{Baby (h=31)}    
            & 0.05 & 0.0268 & 0.05 & 0.05 \\
            & 0.25 & 0.1403 & 0.25 & 0.25 \\
            & 0.45 & 0.2598 & 0.45 & 0.45 \\
            & 0.65 & 0.3766 & 0.65 & 0.65 \\
            & 0.85 & 0.4797 & 0.85 & 0.85 \\
        \bottomrule
    \end{tabular}
    \label{tab:validation_pla}
\end{table}

We note that the worst-case $\eta$ calculated from randomly sampled $Z'$ matches that of the input $\eta$ and the $\eta$ calculated using \shortAlgoNameSticky{} in cases where the input $\eta$ is achievable under the problem \rev{instance}. In cases where it is not achievable, the $\delta$ returned gives the problem \rev{instance} worst-case $\eta$. These results demonstrate that $\delta$ calculated by \shortAlgoNameSticky{} does indeed have a worst-case value which meets the desired or problem \rev{instance} worst-case $\eta$.

\subsubsection{\shortAlgoNameSticky{} Scalability}
To demonstrate the scalability of \shortAlgoNameSticky{}, we find $\delta$ using $\eta=0.1$ for RockSample domains of varying size with corresponding infinite-horizon policies and finite-horizon TigerPOMDP domains of varying horizon length with corresponding finite-horizon policies. All problem \rev{instances are evaluated to} horizon $h=35$. FSCs here are generated from a SARSOP policy generated with a 1,000 iteration limit. Results are shown in \cref{tab:scalability_pla}. Note that the time reported here is the cumulative time for calls to PLA for model processing and feasibility as timed by STORM. These times do not include outer loop model-writing time in Julia, \modifiedBisection{} time in Julia, or time for Docker start-up, and thus are a lower bound on the wall clock time.

\begin{table}[ht]
    \centering
    \caption{\shortAlgoNameSticky{} Scalability - The number of nodes in $\pi$ $|N|$, number of states in $\mathcal{M}$ $|S|$, the number of problem parameters and relevant problem parameters in the optimization after model simplification (as reported by STORM) $|P|/|P^*|$, time to find a solution (from one run, STORM run time only), and solution $\delta$ for $\eta=0.1$ \rev{evaluated to a finite} horizon \rev{of} $h=35$.} 
    \begin{tabular}{c|c c|c|c H|c}
        \toprule
        Domain  & $|N|$ & $|S|$ & $|P|/|P^*|$ & time (sec) & t J (sec) & $\delta$ \\
        \midrule \midrule
        RS(3,3)  & 13 & 73 & 192/2 & 1.507 & 8.9691 & 0.2483\\
        RS(4,4)  & 17 & 257 & 960/2 & 9.323 &  19.5960 & 0.2983  \\
        RS(5,5)  & 21 & 801 & 3840/2 & 120.676 & 164.4027 & 0.3483\\
        RS(6,6)  & 25 & 2,305 & 13440/2 & 1,296.775 & 1687.3779 & 0.3983 \\
        RS(7,7)  & 30 & 6,273 & 43008/2 & 12,487.129 & 16577.7750 & 0.4476 \\
        \midrule
         $\text{Tiger-3}$ & 9 & 8 & 10/4 & 1.240 &  & 0.0056 \\
         $\text{Tiger-4}$ & 15 & 10 & 18/6 & 14.031 &  & 0.0052 \\
         $\text{Tiger-5}$ & 16 & 12 & 24/8 & 276.678 &  & 0.0093 \\
         $\text{Tiger-6}$ & 26 & 14 & 28/10 & MO & MO & MO \\
         $\text{Tiger-7}$ & 23 & 16 & 34/12 & MO & MO & MO \\
        \bottomrule
    \end{tabular}
    \label{tab:scalability_pla}
\end{table}

For problem \rev{instances} with few relevant parameters $P^*$ as reported by PLA, \shortAlgoNameSticky{} is able to find solutions across a range of state sizes, albeit in more time than \shortAlgoName{}. As $P^*$ increases, \shortAlgoNameSticky{} scales less well. This is due in part to the need for increased memory for PLA as $P^*$ increases.

\subsection{Case Studies}

Here, we present motivating case studies which demonstrate the impact of observation error, the differences between the sticky and non-sticky \rev{variants}, and potential applications of our POMDP \ProblemName{} Problem.

\subsubsection{Rover Navigation}\label{sec:toyrover}
We first consider the motivating robotics example, \cref{ex:rover}, introduced in \cref{sec:nonstickyproblem} to demonstrate the impact of observation distribution model error on performance. Consider a Rover POMDP shown in \cref{fig:rover_env}. 
Here, a rover must choose between traversing two paths to a target cell, one which is longer but consistently traversable and another which may or may not have a sand type which causes the rover to become stuck. Noisy observations of both the sand size and texture indicate if the region can be successfully traversed. 
The full \rev{POMDP} definition is provided in \cref{sec:probdefs}. 
We calculate values with the policy shown in \cref{fig:roverpg}.

Here we suppose a decrease of $\eta=20\%$ in the undiscounted value of the policy, i.e. the frequency with which the target state is reached \rev{minus} the penalty for the left path, is admissible and \rev{use \shortAlgoName{}} to find the corresponding observation robustness interval $\delta$. \shortAlgoName{} with $\epsilon_p=0.0$ and a finite evaluation horizon of $h=50$ \rev{returns} $\delta=0.2$.

Given \rev{this} admissible deviation, we present the difference in the frequency with which the corridors are selected as a function of the underlying two-step interval Markov chain, which is generated from the POMDP model and the FSC.  \Cref{fig:not_sandy} shows the frequency of path selection for both the original model and the worst-case $T_{\twoStep}$ model across 10,000 Monte Carlo simulations for both a case where the corridor is traversable and where the corridor is not traversable.

\begin{figure}[ht]
    \begin{subfigure}[t]{0.23\textwidth}
        \centering
        \includegraphics[width=0.8\linewidth]{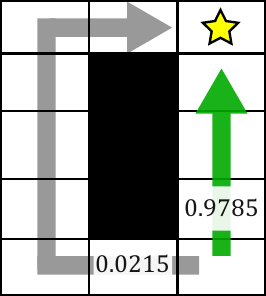}
        \caption{Nominal model tsMC}
        \label{fig:not_sandy_normal}
    \end{subfigure}
    \begin{subfigure}[t]{0.23\textwidth}
        \centering
        \includegraphics[width=0.8\linewidth]{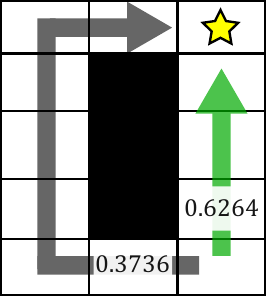}
        \caption{Min. tsMC for $\delta=0.2$}
        \label{fig:not_sandy_delta}
    \end{subfigure}
        \begin{subfigure}[t]{0.23\textwidth}
        \centering
        \includegraphics[width=0.8\linewidth]{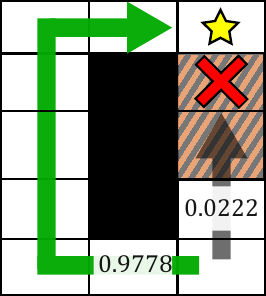}
        \caption{Nominal model tsMC}
        \label{fig:sandy_normal}
    \end{subfigure}
    \begin{subfigure}[t]{0.23\textwidth}
        \centering
        \includegraphics[width=0.8\linewidth]{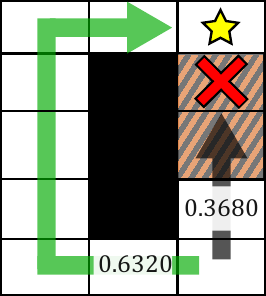}
        \caption{Min. tsMC for $\delta=0.2$}
        \label{fig:sandy_delta}
    \end{subfigure}
    \caption{Differences in path frequencies between the nominal and worst-case (minimizing) two-step Markov chains for a case where the corridor is traversable (a),(b) and not traversable (c),(d).}
    \label{fig:not_sandy}
    \Description{Four figures showing the frequencies with which the two corridors are selected in the rover navigation POMDP problem. When the shorter corridor is traversable (subfigures a and b), the rover takes the shorter corridor with frequency 0.9785 and the longer corridor the rest of the time. Under the worst-case two-step Markov chain for delta 0.2, the rover takes the shorter corridor with frequency 0.6264 and the longer corridor the rest of the time. When the shorter corridor is not traversable (subfigures c and d), the rover takes the longer corridor with frequency 0.9778 and the shorter corridor (resulting in failure) the rest of the time. Under the worst-case two-step Markov chain for delta 0.2, the rover takes the longer corridor with frequency 0.6320 and the shorter corridor (resulting in failure) the rest of the time.}
\end{figure}

In the case where the path is traversable, the increase in (incorrect) observations indicating it isn't traversable results in the longer path being taken more often. This does not impact the frequency with which the target is reached. In the case where the path is not traversable, the system receives more (incorrect) observations indicating it is traversable and consequently fails to reach the target upon entering the non-traversable path. This decrease in reaching the target roughly corresponds to the $20\%$ decrease allowed by $\eta$ when weighted by the probability of a non-traversable initial state. Thus, our approach is able to successfully find observation deviations which correspond to admissible decreases in performance.

\subsubsection{Toy Rover: Sticky and Non-Sticky Differences}
Here we demonstrate a case where the sticky and non-sticky solutions to a problem diverge.
We consider a simplified version of the Rover \rev{POMDP} problem discussed previously, specifically focusing on the corridor decision and abstracting away the rest of the navigation.
The policy shown in \cref{fig:roverpg} is used. \cref{tab:toyrover_forward} shows the worst-case values for $\delta = 0.1$, with graph-preservation $\epsilon_p=0.01$ and evaluation horizon $h=5$.

\begin{table}[ht]
    \centering
    \caption{Policy evaluation of Sticky and Non-Sticky Minimizers for Toy Rover with $\delta=0.1$, $\epsilon_p=0.01$, $\epsilon_{\modifiedBisectionShort{}}=10^{-5}$, where $SM$ is short for smooth, $AN$ is short for angular, $L$ is short for large, $S$ is short for small, and the state for all observations is large, angular sand.}
    \setlength\tabcolsep{2.25pt}
    \begin{tabular}{c|c|cc|cc|cc}
        \toprule
        Problem \rev{Variant}  & $\min(V)$ & $Z(AN|N_2)$ & $Z(SM|N_2)$ & $Z(AN|N_3)$ & $Z(SM|N_3)$ & $Z(L|N_1)$ & $Z(S|N_1)$ \\
        \midrule \midrule
        Sticky  & 0.8521 & 0.89 & 0.11 & 0.89 & 0.11 & 0.89 & 0.11 \\
        Non-Sticky  & 0.8466 & 0.89 & 0.11 & $\mathbf{0.99}$ & $\mathbf{0.01}$ & 0.89 & 0.11\\
        \bottomrule
    \end{tabular}
    \label{tab:toyrover_forward}
\end{table}

The value gap observed in \cref{tab:toyrover_forward} can be attributed to the difference in observation probabilities assigned to nodes $N_2$ and $N_3$ in the FSC. Consider the case where the sand type is large, angular. In this case, a lower value is achieved when (incorrect) observations which lead to the rover traversing the sand are more likely, i.e. "smooth" from $N_2$ and "angular" from $N_3$. In the non-sticky \rev{variant}, this minimization is fairly straightforward. At both $N_2$ and $N_3$ the probability of observing smooth and angular respectively can be maximized as there is no constraint on these nodes. However in the sticky \rev{variant}, since nature is constrained to assign the same probability to all nodes which share a state, action, and observation, it is not free to maximize the probability of observing angular at $N_3$ as it was previously. Therefore the minimum value of the sticky \rev{variant} on the interval is higher.

This likely explains the differences in the admissible observation distribution model deviation $\delta$ for the sticky and non-sticky \rev{variants} shown in \cref{tab:toyrover}. For a given $\delta$, the non-sticky worst-case value is a lower bound on the sticky worst-case value. Thus, we see a corresponding decrease $\delta$ for the non-sticky \rev{variant} compared to the sticky \rev{variant}.

\begin{table}[ht]
    \centering
    \caption{Policy robustness values for Sticky and Non-Sticky $\delta$ for Toy Rover with $\eta=0.1$ and $\epsilon_p=0.01$.}
    \begin{tabular}{c|c|c}
        \toprule
        Problem \rev{Variant} & $\eta$ & $\delta$ \\
        \midrule \midrule
        Non-Sticky  & 0.1 &  0.1006\\
        Sticky & 0.1 & 0.1078\\
        \bottomrule
    \end{tabular}
    \label{tab:toyrover}
\end{table}

\subsubsection{Cancer Diagnosis}
Consider a cancer diagnosis POMDP inspired by \citet{ayer_or_2012}, where tests are used to determine when to treat in-situ cancer or invasive cancer. Reward is measured in quality-adjusted life years (QALYs). A treatment policy (\cref{fig:cancer_pg}) is developed using a test which has 20\% probability of a false negative when in-situ (pre-invasive) cancer is present and has a 5\% probability of a false positive when no cancer is present. A full description of the model is provided in \cref{sec:probdefs}. This policy achieves 98.53 QALYs with the baseline model which is based on the test. Using this policy with \shortAlgoName{}, we can determine the acceptable decrease in accuracy of the test given an admissible percent decrease in QALYs.

\begin{figure}[ht]
    \centering
    \begin{subfigure}[t]{0.39\linewidth}
        \centering
        \includegraphics[width=0.6\linewidth]{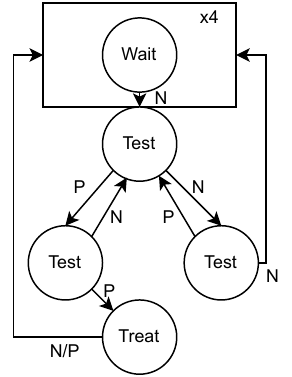}
        \caption{FSC for the cancer case study. Note N and P correspond to negative and positive results. "x4" indicates 4 wait actions are taken successively.}
        \label{fig:cancer_pg}
    \end{subfigure}
    \;
    \begin{subfigure}[t]{0.59\linewidth}
        \centering
        \includegraphics[width=0.85\linewidth]{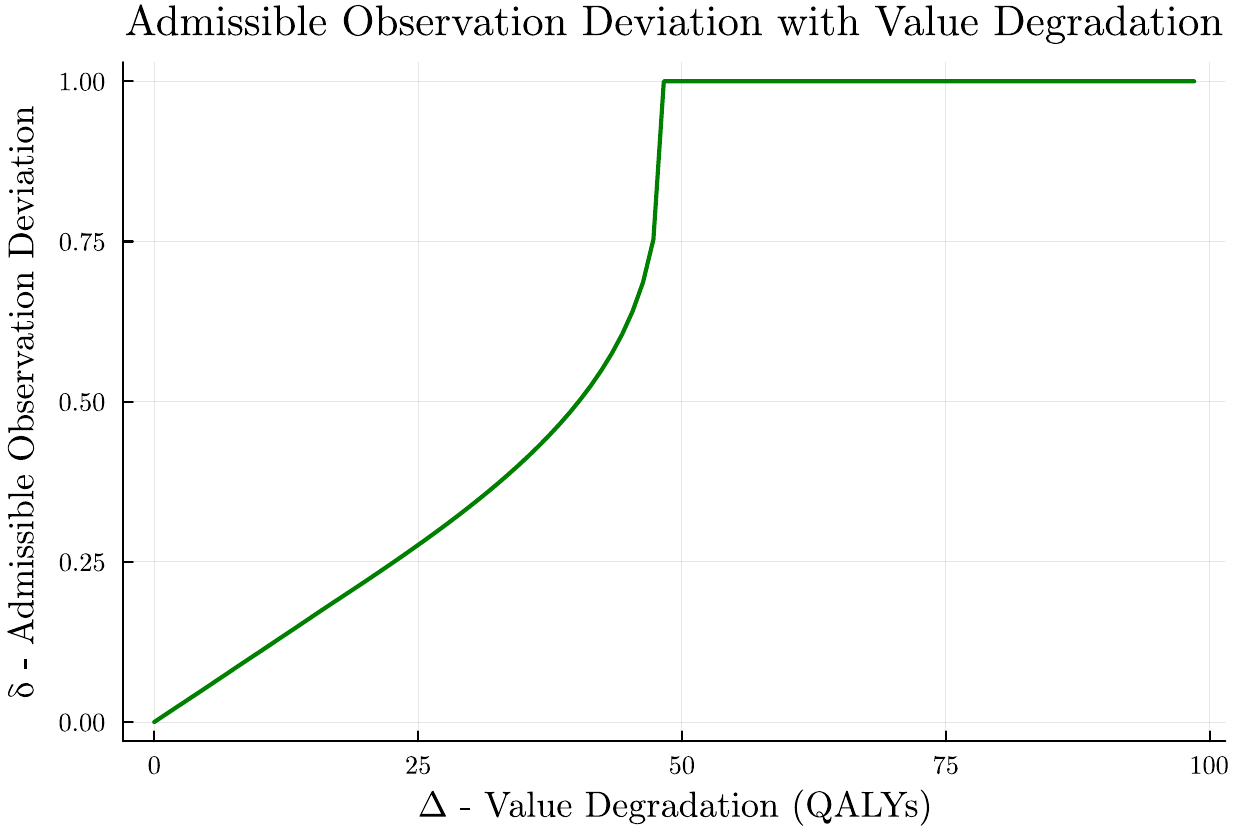}
        \caption{Change in admissible observation deviation $\delta$ with input of value degradation threshold $\threshold$. Note that there is a minimum, nonzero QALYs threshold, at which $\delta=1$ is achieved.}
        \label{fig:cancer_plot}
    \end{subfigure}
    \caption{Policy and results for the cancer case study.}
    \label{fig:cancer_pg_plot}
    \Description{A figure showing the finite-state controller policy and a plot of admissible observation deviation with input value threshold for the cancer case study. The finite-state controller policy takes the wait action four times before taking the testing action until two consecutive observations of the same type occur. If both are positive, the treatment action is taken. If both are negative, another four wait actions are taken. The plot shows a curve which increases from zero admissible deviation roughly linearly to delta of roughly 0.3 at 30 QALYs, at which point it increases exponentially to delta of 1.0 at roughly 48 QALYs.}
\end{figure}

\Cref{fig:cancer_plot} shows the increase in admissible observation deviation $\delta$ as a function of degradation in QALYs $\threshold$.
Observe that admissible $\delta$ increases up to $1.0$ at $\threshold \approx 48$ QALYs, indicating the observation function only impacts the policy value up to this value threshold. Beyond this, we observe no impact in the admissible observation deviation. This can be attributed in part to the structure of the policy. The policy will take action "wait" before testing for cancer, which has a reward of 1 QALY, at least four times before testing. Likewise, the probability of having cancer which progresses all the way to death is $0.0012$ if untreated. Given this low probability and the reward for waiting, the value influenced by observation quality is inherently limited for this problem.

Note that \shortAlgoNameSticky{} is not run on this benchmark as the PLA implementation used does not natively support discounted, infinite-horizon problems.

\subsubsection{Part Quality Control}

Consider a part checking POMDP inspired by operations research. Here, inspection machines with some accuracy are used to determine whether a part is passing or failing required specifications. The underlying state of passing or failing is partially observable. The problem terminates when either "accept" or "reject" actions are taken. The probability of the initial state being passing is $0.5$. A full description of the \rev{POMDP} is in \cref{sec:probdefs}. Two heuristic policies are considered (\cref{fig:partPGs}): policy 1,  which takes advantage of a higher quality inspection system with better measurement accuracy ($ac_1 = 99 \%$), and policy 2, which accounts for a less accurate system ($ac_2 = 90.234\%$). Both policies initially achieve approximately the same classification error rate of approximately 1 in 100 out of specification parts being passed on their individual models, which have 99\% and 90.234\% measurement accuracy, respectively.

\begin{figure}[ht]
    \centering
    \begin{subfigure}[b]{0.45\linewidth}
        \centering
        \includegraphics[width=0.45\linewidth]{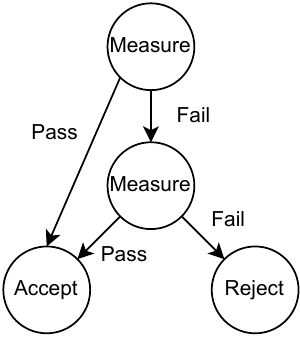}
        \caption{Policy 1}
        \label{fig:part_pg1}
    \end{subfigure}
    \;
    \begin{subfigure}[b]{0.45\linewidth}
        \centering
        \includegraphics[width=0.54\linewidth]{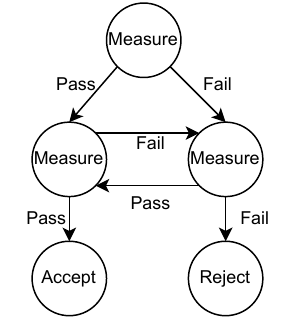}
        \caption{Policy 2}
        \label{fig:part_pg2}
    \end{subfigure}
    \caption{FSCs for the Part Checking POMDP. Note that the initial measure node is returned to after accepting or rejecting a part.}
    \label{fig:partPGs}
    \Description{A figure showing two finite-state controller policies for the part checking POMDP. Policy 1 measures once. If the pass observation is received, it takes the accept action. If the fail observation is received, it takes the measure action again and takes the accept action if pass is observed and otherwise takes the reject action. Policy 2 measures until two consecutive observations agree. If both are pass, the accept action is taken. If both are fail, the reject action is taken.}
\end{figure}

\Cref{fig:parts} shows the admissible measurement inaccuracy, defined $\min(1-ac+\delta,1)$, on each POMDP for the respective policy as a function of admissible fraction of failing parts accepted $\threshold$ for both \shortAlgoName{} with $\epsilon_p=0.0$ and \shortAlgoNameSticky{} with $\epsilon_p=0.01$. The evaluation horizon is limited to $200$ steps. Policy 2 is more robust to measurement inaccuracy, up until approximately $\threshold=0.5$, at which point both policies approach admissible inaccuracy of 1. Thus, it is shown quantitatively that policy 2 is much more robust to measurement inaccuracy than policy 1. Similar studies could be performed to select a policy which best balances robustness and efficiency.

As with the cancer case study, we observe an asymptote in the admissible inaccuracy as value degradation threshold increases. This is due in part to the underlying initial state probability, as $50\%$ of parts are passing initially. Since we are only concerned with false positives (accepting failing parts), the worst-case performance with respect to this metric is $50\%$ failure. Thus, at $50\%$ false positives, the entire probability simplex is admissible.

\begin{figure}[ht]
    \centering
    \begin{subfigure}[b]{0.49\linewidth}
        \centering
        \includegraphics[width=\linewidth]{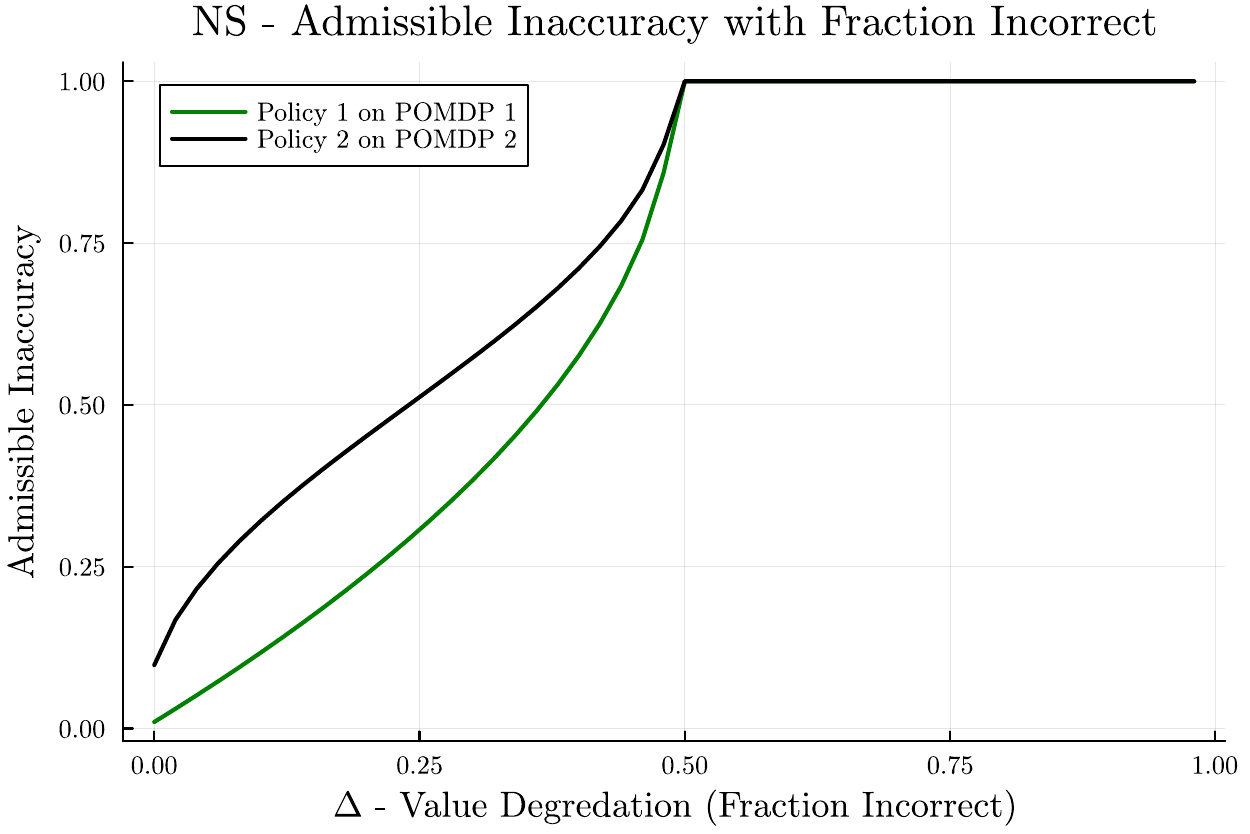}
        \caption{\shortAlgoName{}}
        \label{fig:partsNS}
    \end{subfigure}
    \;
    \begin{subfigure}[b]{0.49\linewidth}
        \centering
        \includegraphics[width=\linewidth]{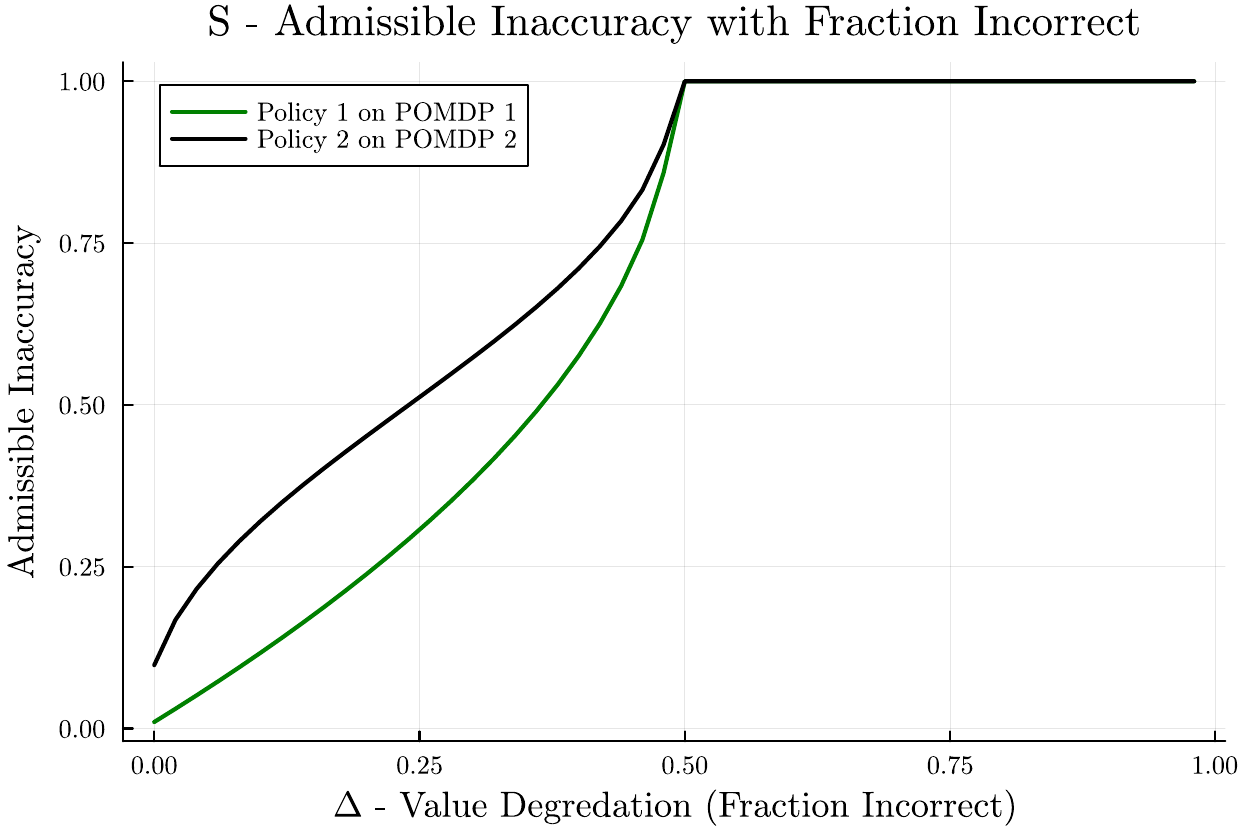}
        \caption{\shortAlgoNameSticky{}}
        \label{fig:partsS}
    \end{subfigure}
    \caption{Comparison of Policy 1 and Policy 2 for the Part Checking POMDP.}
    \label{fig:parts}
    \Description{A figure with two plots showing the admissible inaccuracy as a function of value degradation threshold for RIS-NS and RIS-S. Both the plot for RIS-NS and the plot for RIS-S follow the same trend. They both show a curve for policy 1 which increases from zero admissible observation inaccuracy at zero degradation roughly parabolically to admissible inaccuracy of 1.0 at roughly 50\% incorrect labeling of parts. They also show a curve for policy 2 which increases from roughly 0.125 admissible inaccuracy at zero degradation to delta of 1.0 at roughly 50\% incorrect, consistently remaining above the policy 1 curve.}
\end{figure}
\section{Conclusion and Future Work}
We introduce and analyze the \ProblemName{} Problem, where some admissible observation distribution model deviation $\delta$ is found given some admissible value degradation threshold $\threshold$ for both history-dependent observations in \Cref{prob:non-sticky}, the non-sticky \rev{variant}, and strictly state- and action-dependent observations in \Cref{prob:sticky}, the sticky \rev{variant}. In the non-sticky \rev{variant}, we show that it is sufficient to consider a node-dependent representation of the history if using an FSC policy.  We show that uncertain Markov chain representations provide a convenient representation for solving the inner optimization worst-case value problem and that a univariate optimization can be used in tandem with this inner loop to find $\delta$.

We introduce two algorithms, \algoName{} and \algoNameSticky{}, to solve the \rev{two variants}. Our results show the approximate-value non-sticky solution approach scales well to problem \rev{instances} with Markov chain state space sizes in the millions, while the approximate-value sticky solution approach is generally restricted to smaller problem \rev{instances}. We also demonstrate the applicability of our problem to several case studies.

There are several compelling directions for future work. In this work, our implementations use near-optimal approximations of value. While the value threshold can be met with these approximations, further analysis is needed on the convergence of Robust Interval Search to $\delta^*$ under these approximations.

In this work, our solution method for RIS-S uses parametric Markov chains for evaluation, which allow for convenient representation of constraints introduced by \cref{prob:sticky} but has limited scalability as problem size increases. Developing an algorithm for RIS-S which scales better with problem size is an area of future work.

We show theoretically that the solution to the non-sticky \rev{variant} provides a lower bound on the solution to the sticky \rev{variant}. Interestingly, in many problem \rev{instances}, the resulting solutions are very close to each other. 
Further analysis of the relationship between the sticky and non-sticky \rev{variants}, particularly the cases in which their solutions are close, may provide insights on developing more efficient algorithms for RIS-S.

Furthermore, while this work provides bounds on the admissible deviation for distributions in the observation function, it does not provide feedback on the sensitivity of the policy value with respect to each distribution in the observation model. A methodology that provides this information could have additional relevance to
practitioners deploying these policies.

\begin{acks}
    We thank Rayan Mazouz and Tyler Becker for their insightful discussions on this work.
    
    This material is based upon work supported by the National Science Foundation Graduate Research Fellowship Program under Grant No. DGE 2040434. Any opinions, findings, and conclusions or recommendations expressed in this material are those of the author(s) and do not necessarily reflect the views of the National Science Foundation. 

     Part of this research was carried out at the Jet Propulsion Laboratory, California Institute of Technology, under a contract with the National Aeronautics and Space Administration (80NM0018D0004). Qi Heng Ho was supported by a Strategic University Research Partnership (SURP) grant from the NASA Jet Propulsion Laboratory (JPL) (RSA 1716011). 
\end{acks}

\bibliography{real}

@article{givan_bounded-parameter_2000,
	title = {Bounded-parameter {Markov} decision processes},
	volume = {122},
	issn = {0004-3702},
	url = {https://www.sciencedirect.com/science/article/pii/S0004370200000473},
	doi = {10.1016/S0004-3702(00)00047-3},
	number = {1},
	urldate = {2024-11-15},
	journal = {Artificial Intelligence},
	author = {Givan, Robert and Leach, Sonia and Dean, Thomas},
	month = {Sep},
	year = {2000},
	pages = {71--109},
}

@article{junges_finite-state_nodate,
	title = {Finite-state {Controllers} of {POMDPs} via {Parameter} {Synthesis}},
	language = {en},
	author = {Junges, Sebastian and Jansen, Nils and Wimmer, Ralf and Quatmann, Tim and Winterer, Leonore and Katoen, Joost-Pieter and Becker, Bernd},
    year = 2018,
    journal = "Conference on Uncertainty in Artifical Intelligence"
}

@inproceedings{bovy_imprecise_2024,
author = {Bovy, Eline M. and Suilen, Marnix and Junges, Sebastian and Jansen, Nils},
title = {Imprecise probabilities meet partial observability: game semantics for robust POMDPs},
year = {2024},
isbn = {978-1-956792-04-1},
url = {https://doi.org/10.24963/ijcai.2024/740},
doi = {10.24963/ijcai.2024/740},
booktitle = {Proceedings of the Thirty-Third International Joint Conference on Artificial Intelligence},
articleno = {740},
numpages = {10},
location = {Jeju, Korea},
series = {IJCAI '24}
}

@incollection{galesloot_pessimistic_2024,
  title={Pessimistic Iterative Planning with RNNs for Robust POMDPs},
  author={Galesloot, Maris FL and Suilen, Marnix and Sim{\~a}o, Thiago D and Carr, Steven and Spaan, Matthijs TJ and Topcu, Ufuk and Jansen, Nils},
  booktitle={ECAI 2025},
  pages={4823--4831},
  year={2025},
  publisher={IOS Press}
}

@article{junges_parameter_2024,
	title = {Parameter synthesis for {Markov} models: covering the parameter space},
	volume = {62},
	issn = {1572-8102},
	shorttitle = {Parameter synthesis for {Markov} models},
	url = {https://doi.org/10.1007/s10703-023-00442-x},
	doi = {10.1007/s10703-023-00442-x},
	language = {en},
	number = {1},
	urldate = {2025-02-09},
	journal = {Formal Methods in System Design},
	author = {Junges, Sebastian and Ábrahám, Erika and Hensel, Christian and Jansen, Nils and Katoen, Joost-Pieter and Quatmann, Tim and Volk, Matthias},
	month = jun,
	year = {2024},
	pages = {181--259},
}

@inproceedings{cubuktepe_robust_2021,
  title={Robust finite-state controllers for uncertain POMDPs},
  author={Cubuktepe, Murat and Jansen, Nils and Junges, Sebastian and Marandi, Ahmadreza and Suilen, Marnix and Topcu, Ufuk},
  booktitle={Proceedings of the AAAI Conference on Artificial Intelligence},
  volume={35},
  number={13},
  pages={11792--11800},
  year={2021}
}

@inproceedings{bork_verification_2020,
	address = {Cham},
	title = {Verification of {Indefinite}-{Horizon} {POMDPs}},
	isbn = {978-3-030-59152-6},
	doi = {10.1007/978-3-030-59152-6_16},
	language = {en},
	booktitle = {Automated {Technology} for {Verification} and {Analysis}},
	publisher = {Springer International Publishing},
	author = {Bork, Alexander and Junges, Sebastian and Katoen, Joost-Pieter and Quatmann, Tim},
	editor = {Hung, Dang Van and Sokolsky, Oleg},
	year = {2020},
	pages = {288--304},
}

@inproceedings{chatterjee_qualitative_2015,
	title = {Qualitative analysis of {POMDPs} with temporal logic specifications for robotics applications},
	url = {https://ieeexplore.ieee.org/abstract/document/7139019},
	doi = {10.1109/ICRA.2015.7139019},
	urldate = {2025-02-09},
	booktitle = {2015 {IEEE} {International} {Conference} on {Robotics} and {Automation} ({ICRA})},
	author = {Chatterjee, Krishnendu and Chmelík, Martin and Gupta, Raghav and Kanodia, Ayush},
	month = may,
	year = {2015},
	note = {ISSN: 1050-4729},
	pages = {325--330},
}

@inproceedings{kara_robustness_2021,
	address = {Cham},
	title = {Robustness to {Approximations} and {Model} {Learning} in {MDPs} and {POMDPs}},
	isbn = {978-3-030-76928-4},
	doi = {10.1007/978-3-030-76928-4_9},
	language = {en},
	booktitle = {Modern {Trends} in {Controlled} {Stochastic} {Processes}:},
	publisher = {Springer International Publishing},
	author = {Kara, Ali Devran and Yüksel, Serdar},
	editor = {Piunovskiy, Alexey and Zhang, Yi},
	year = {2021},
	pages = {166--191},
}

@article{fard_variance_2008,
	title = {A {Variance} {Analysis} for {POMDP} {Policy} {Evaluation}},
	volume = {2},
	issn = {2159-5399},
	doi = {10.1901/jaba.2008.2-1056},
	language = {eng},
	journal = {Proceedings of the ... AAAI Conference on Artificial Intelligence. AAAI Conference on Artificial Intelligence},
	author = {Fard, Mahdi Milani and Pineau, Joelle and Sun, Peng},
	month = jan,
	year = {2008},
	pmid = {20582245},
	pmcid = {PMC2892224},
	pages = {1056--1061},
}

@misc{konsta_what_2024,
	title = {What should be observed for optimal reward in {POMDPs}?},
	url = {http://arxiv.org/abs/2405.10768},
	language = {en},
	urldate = {2024-09-25},
	publisher = {arXiv},
	author = {Konsta, Alyzia-Maria and Lafuente, Alberto Lluch and Matheja, Christoph},
	month = jul,
	year = {2024},
	note = {arXiv:2405.10768 [cs]},
}

@inproceedings{kurniawati_sarsop_2008,
	title = {{SARSOP}: {Efficient} point-based {POMDP} planning by approximating optimally reachable belief spaces},
	shorttitle = {{SARSOP}},
	booktitle = {{Proc}. {Robotics}: {Science} and {Systems}},
	author = {Kurniawati, Hanna and Hsu, David and Lee, Wee Sun},
	year = {2008},
}

@inproceedings{smith2004heuristic,
  title={Heuristic search value iteration for POMDPs},
  author={Smith, Trey and Simmons, Reid},
  booktitle={Proceedings of the 20th conference on Uncertainty in artificial intelligence},
  pages={520--527},
  year={2004}
}

@article{ye_despot_2017,
	title = {{DESPOT}: {Online} {POMDP} {Planning} with {Regularization}},
	volume = {58},
	copyright = {Copyright (c)},
	issn = {1076-9757},
	shorttitle = {{DESPOT}},
	url = {https://www.jair.org/index.php/jair/article/view/11043},
	doi = {10.1613/jair.5328},
	language = {en},
	urldate = {2022-06-03},
	journal = {Journal of Artificial Intelligence Research},
	author = {Ye, Nan and Somani, Adhiraj and Hsu, David and Lee, Wee Sun},
	month = jan,
	year = {2017},
	pages = {231--266},
}

@inproceedings{silver_monte-carlo_2010,
	title = {Monte-{Carlo} {Planning} in {Large} {POMDPs}},
	volume = {23},
	url = {https://proceedings.neurips.cc/paper/2010/hash/edfbe1afcf9246bb0d40eb4d8027d90f-Abstract.html},
	urldate = {2025-02-09},
	booktitle = {Advances in {Neural} {Information} {Processing} {Systems}},
	publisher = {Curran Associates, Inc.},
	author = {Silver, David and Veness, Joel},
	year = {2010},
}

@article{ayer_or_2012,
	title = {{OR} {Forum}—{A} {POMDP} {Approach} to {Personalize} {Mammography} {Screening} {Decisions}},
	volume = {60},
	issn = {0030-364X},
	url = {https://pubsonline.informs.org/doi/abs/10.1287/opre.1110.1019},
	doi = {10.1287/opre.1110.1019},
	number = {5},
	urldate = {2025-02-10},
	journal = {Operations Research},
	author = {Ayer, Turgay and Alagoz, Oguzhan and Stout, Natasha K.},
	month = oct,
	year = {2012},
	note = {Publisher: INFORMS},
	pages = {1019--1034},
}

@inproceedings{10.5555/3535850.3535914,
author = {Gupta, Himanshu and Hayes, Bradley and Sunberg, Zachary},
title = {Intention-Aware Navigation in Crowds with Extended-Space POMDP Planning},
year = {2022},
isbn = {9781450392136},
publisher = {International Foundation for Autonomous Agents and Multiagent Systems},
address = {Richland, SC},
booktitle = {Proceedings of the 21st International Conference on Autonomous Agents and Multiagent Systems},
pages = {562–570},
numpages = {9},
location = {Virtual Event, New Zealand},
series = {AAMAS '22}
}

@article{10.1145/3359616,
author = {Chen, Min and Nikolaidis, Stefanos and Soh, Harold and Hsu, David and Srinivasa, Siddhartha},
title = {Trust-Aware Decision Making for Human-Robot Collaboration: Model Learning and Planning},
year = {2020},
issue_date = {June 2020},
publisher = {Association for Computing Machinery},
address = {New York, NY, USA},
volume = {9},
number = {2},
url = {https://doi.org/10.1145/3359616},
doi = {10.1145/3359616},
journal = {J. Hum.-Robot Interact.},
month = jan,
articleno = {9},
numpages = {23}
}

@ARTICLE{9899480,
  author={Lauri, Mikko and Hsu, David and Pajarinen, Joni},
  journal={IEEE Transactions on Robotics}, 
  title={Partially Observable Markov Decision Processes in Robotics: A Survey}, 
  year={2023},
  volume={39},
  number={1},
  pages={21-40},
  doi={10.1109/TRO.2022.3200138}}

@inbook{burden38216numerical,
  title={Numerical Analysis 7th Edition},
  author={Burden, Richard L and Faires, J Douglas},
  publisher={Thomson Learning},
  isbn= {0-534-38216-9},
  chapter=2,
  year =2001 
}

@article{Kaelbling1998pomdp,
title = {Planning and acting in partially observable stochastic domains},
journal = {Artificial Intelligence},
volume = {101},
number = {1},
pages = {99-134},
year = {1998},
author = {Leslie Pack Kaelbling and Michael L. Littman and Anthony R. Cassandra}
}

@article{egorov2017pomdps,
  title={POMDPs. jl: A framework for sequential decision making under uncertainty},
  author={Egorov, Maxim and Sunberg, Zachary N and Balaban, Edward and Wheeler, Tim A and Gupta, Jayesh K and Kochenderfer, Mykel J},
  journal={The Journal of Machine Learning Research},
  volume={18},
  number={1},
  pages={831--835},
  year={2017},
  publisher={JMLR. org}
}

@article{mathiesen2024intervalmdp,
title = {IntervalMDP.jl: Accelerated Value Iteration for Interval Markov Decision Processes},
journal = {IFAC-PapersOnLine},
volume = {58},
number = {11},
pages = {1-6},
year = {2024},
note = {8th IFAC Conference on Analysis and Design of Hybrid Systems ADHS 2024},
issn = {2405-8963},
doi = {https://doi.org/10.1016/j.ifacol.2024.07.416},
url = {https://www.sciencedirect.com/science/article/pii/S2405896324005159},
author = {Frederik Baymler Mathiesen and Morteza Lahijanian and Luca Laurenti}
}

@inbook{kochenderfer2015decision,
  title={Decision making under uncertainty: theory and application},
  author={Kochenderfer, Mykel J},
  year={2015},
  publisher={MIT press},
  chapter={6}
}

@inproceedings{spel_finding_2021,
	address = {Cham},
	title = {Finding {Provably} {Optimal} {Markov} {Chains}},
	isbn = {978-3-030-72016-2},
	doi = {10.1007/978-3-030-72016-2_10},
	language = {en},
	booktitle = {Tools and {Algorithms} for the {Construction} and {Analysis} of {Systems}},
	publisher = {Springer International Publishing},
	author = {Spel, Jip and Junges, Sebastian and Katoen, Joost-Pieter},
	editor = {Groote, Jan Friso and Larsen, Kim Guldstrand},
	year = {2021},
	pages = {173--190},
}

@inproceedings{dehnert_storm_2017,
	address = {Cham},
	title = {A {Storm} is {Coming}: {A} {Modern} {Probabilistic} {Model} {Checker}},
	isbn = {978-3-319-63390-9},
	shorttitle = {A {Storm} is {Coming}},
	doi = {10.1007/978-3-319-63390-9_31},
	language = {en},
	booktitle = {Computer {Aided} {Verification}},
	publisher = {Springer International Publishing},
	author = {Dehnert, Christian and Junges, Sebastian and Katoen, Joost-Pieter and Volk, Matthias},
	editor = {Majumdar, Rupak and Kunčak, Viktor},
	year = {2017},
	pages = {592--600},
}

@inproceedings{daws2004symbolic,
  title={Symbolic and parametric model checking of discrete-time Markov chains},
  author={Daws, Conrado},
  booktitle={International Colloquium on Theoretical Aspects of Computing},
  pages={280--294},
  year={2004},
  organization={Springer}
}

@inproceedings{mastin_loss_2012,
	title = {Loss bounds for uncertain transition probabilities in {Markov} decision processes},
	url = {https://ieeexplore.ieee.org/document/6426504/?arnumber=6426504&tag=1},
	doi = {10.1109/CDC.2012.6426504},
	urldate = {2024-09-25},
	booktitle = {2012 {IEEE} 51st {IEEE} {Conference} on {Decision} and {Control} ({CDC})},
	author = {Mastin, Andrew and Jaillet, Patrick},
	month = {Dec},
	year = {2012},
	note = {ISSN: 0743-1546},
	pages = {6708--6715},
}

@article{tan_sensitivity_2011,
	title = {Sensitivity {Analysis} in {Markov} {Decision} {Processes} with {Uncertain} {Reward} {Parameters}},
	volume = {48},
	issn = {0021-9002},
	url = {https://www.jstor.org/stable/23066436},
	number = {4},
	urldate = {2024-09-25},
	journal = {Journal of Applied Probability},
	author = {Tan, Chin Hon and Hartman, Joseph C.},
	year = {2011},
	note = {Publisher: Applied Probability Trust},
	pages = {954--967},
}

@inproceedings{robustpolsynth,
  title     = {Robust Policy Synthesis for Uncertain POMDPs via Convex Optimization},
  author    = {Suilen, Marnix and Jansen, Nils and Cubuktepe, Murat and Topcu, Ufuk},
  booktitle = {Proceedings of the Twenty-Ninth International Joint Conference on
               Artificial Intelligence, {IJCAI-20}},
  publisher = {International Joint Conferences on Artificial Intelligence Organization},
  editor    = {Christian Bessiere},
  pages     = {4113--4120},
  year      = {2020},
  month     = {7},
  note      = {Main track},
  doi       = {10.24963/ijcai.2020/569},
  url       = {https://doi.org/10.24963/ijcai.2020/569},
}

@inproceedings{meuleau_solving_2013,
author = {Meuleau, Nicolas and Kim, Kee-Eung and Kaelbling, Leslie Pack and Cassandra, Anthony R.},
title = {Solving POMDPs by searching the space of finite policies},
year = {1999},
isbn = {1558606149},
publisher = {Morgan Kaufmann Publishers Inc.},
address = {San Francisco, CA, USA},
booktitle = {Proceedings of the Fifteenth Conference on Uncertainty in Artificial Intelligence},
pages = {417–426},
numpages = {10},
location = {Stockholm, Sweden},
series = {UAI'99}
}

@article{white_parameter_1986,
	title = {Parameter {Imprecision} in {Finite} {State}, {Finite} {Action} {Dynamic} {Programs}},
	volume = {34},
	issn = {0030-364X},
	url = {https://www.jstor.org/stable/170676},
	number = {1},
	urldate = {2024-11-15},
	journal = {Operations Research},
	author = {White, Chelsea C. and El-Deib, Hany K.},
	year = {1986},
	note = {Publisher: INFORMS},
	pages = {120--129},
}

@inproceedings{jonsson_specification_1991,
	title = {Specification and refinement of probabilistic processes},
	url = {https://ieeexplore.ieee.org/document/151651},
	doi = {10.1109/LICS.1991.151651},
	urldate = {2025-05-14},
	booktitle = {[1991] {Proceedings} {Sixth} {Annual} {IEEE} {Symposium} on {Logic} in {Computer} {Science}},
	author = {Jonsson, B. and Larsen, K.G.},
	month = jul,
	year = {1991},
	pages = {266--277},
}

@phdthesis{hauskrecht1997planning,
  title={Planning and control in stochastic domains with imperfect information},
  author={Hauskrecht, Milos},
  year={1997},
  school={Massachusetts Institute of Technology}
}

@InProceedings{10.1007/978-3-319-46520-3_4,
author="Quatmann, Tim
and Dehnert, Christian
and Jansen, Nils
and Junges, Sebastian
and Katoen, Joost-Pieter",
editor="Artho, Cyrille
and Legay, Axel
and Peled, Doron",
title="Parameter Synthesis for Markov Models: Faster Than Ever",
booktitle="Automated Technology for Verification and Analysis",
year="2016",
publisher="Springer International Publishing",
address="Cham",
pages="50--67",
isbn="978-3-319-46520-3"
}

@article{MADANI20035,
title = {On the undecidability of probabilistic planning and related stochastic optimization problems},
journal = {Artificial Intelligence},
volume = {147},
number = {1},
pages = {5-34},
year = {2003},
note = {Planning with Uncertainty and Incomplete Information},
issn = {0004-3702},
doi = {https://doi.org/10.1016/S0004-3702(02)00378-8},
url = {https://www.sciencedirect.com/science/article/pii/S0004370202003788},
author = {Omid Madani and Steve Hanks and Anne Condon}
}

@article{c785ceef-c26f-38c4-b878-e08088df3cc3,
 ISSN = {0364765X, 15265471},
 URL = {http://www.jstor.org/stable/3689975},
 author = {Christos H. Papadimitriou and John N. Tsitsiklis},
 journal = {Mathematics of Operations Research},
 number = {3},
 pages = {441--450},
 publisher = {INFORMS},
 title = {The Complexity of Markov Decision Processes},
 urldate = {2025-06-20},
 volume = {12},
 year = {1987}
}

@article{cubuktepe2021convex,
  title={Convex optimization for parameter synthesis in MDPs},
  author={Cubuktepe, Murat and Jansen, Nils and Junges, Sebastian and Katoen, Joost-Pieter and Topcu, Ufuk},
  journal={IEEE Transactions on Automatic Control},
  volume={67},
  number={12},
  pages={6333--6348},
  year={2021},
  publisher={IEEE}
}

@book{bellman57,
 ISBN = {9780691146683},
 URL = {http://www.jstor.org/stable/j.ctv1nxcw0f},
 author = {Richard Bellman and Stuart Dreyfus},
 publisher = {Princeton University Press},
 title = {Dynamic Programming},
 urldate = {2025-06-22},
 volume = {33},
 year = {2010}
}

@incollection{kroening_prophesy_2015,
	address = {Cham},
	title = {{PROPhESY}: {A} {PRObabilistic} {ParamEter} {SYnthesis} {Tool}},
	volume = {9206},
	isbn = {978-3-319-21689-8 978-3-319-21690-4},
	shorttitle = {{PROPhESY}},
	url = {http://link.springer.com/10.1007/978-3-319-21690-4_13},
	language = {en},
	urldate = {2025-07-05},
	booktitle = {Computer {Aided} {Verification}},
	publisher = {Springer International Publishing},
	author = {Dehnert, Christian and Junges, Sebastian and Jansen, Nils and Corzilius, Florian and Volk, Matthias and Bruintjes, Harold and Katoen, Joost-Pieter and Ábrahám, Erika},
	editor = {Kroening, Daniel and Păsăreanu, Corina S.},
	year = {2015},
	doi = {10.1007/978-3-319-21690-4_13},
	note = {Series Title: Lecture Notes in Computer Science},
	pages = {214--231},
}

@article{mazouz2024piecewise,
  title={Piecewise stochastic barrier functions},
  author={Mazouz, Rayan and Mathiesen, Frederik Baymler and Laurenti, Luca and Lahijanian, Morteza},
  journal={arXiv preprint arXiv:2404.16986},
  year={2024}
}

@book{puterman2014markov,
  title={Markov decision processes: discrete stochastic dynamic programming},
  author={Puterman, Martin L},
  year={2014},
  publisher={John Wiley \& Sons}
}

@article{sondik1978,
 ISSN = {0030364X, 15265463},
 URL = {http://www.jstor.org/stable/169635},
 author = {Edward J. Sondik},
 journal = {Operations Research},
 number = {2},
 pages = {282--304},
 publisher = {INFORMS},
 title = {The Optimal Control of Partially Observable Markov Processes Over the Infinite Horizon: Discounted Costs},
 urldate = {2025-09-08},
 volume = {26},
 year = {1978}
}

@article{iyengar2005robust,
  title={Robust dynamic programming},
  author={Iyengar, Garud N},
  journal={Mathematics of Operations Research},
  volume={30},
  number={2},
  pages={257--280},
  year={2005},
  publisher={INFORMS}
}

@article{nilim2005robust,
  title={Robust control of Markov decision processes with uncertain transition matrices},
  author={Nilim, Arnab and El Ghaoui, Laurent},
  journal={Operations Research},
  volume={53},
  number={5},
  pages={780--798},
  year={2005},
  publisher={INFORMS}
}

@article{hutschenreiter_parametric_2017,
	title = {Parametric {Markov} {Chains}: {PCTL} {Complexity} and {Fraction}-free {Gaussian} {Elimination}},
	volume = {256},
	issn = {2075-2180},
	shorttitle = {Parametric {Markov} {Chains}},
	url = {http://arxiv.org/abs/1709.02093},
	doi = {10.4204/EPTCS.256.2},
	urldate = {2025-09-15},
	journal = {Electronic Proceedings in Theoretical Computer Science},
	author = {Hutschenreiter, Lisa and Baier, Christel and Klein, Joachim},
	month = sep,
	year = {2017},
	note = {arXiv:1709.02093 [cs]},
	pages = {16--30},
}

@INPROCEEDINGS{9081758,
  author={Owen, Michael P. and Panken, Adam and Moss, Robert and Alvarez, Luis and Leeper, Charles},
  booktitle={2019 IEEE/AIAA 38th Digital Avionics Systems Conference (DASC)}, 
  title={ACAS Xu: Integrated Collision Avoidance and Detect and Avoid Capability for UAS}, 
  year={2019},
  volume={},
  number={},
  pages={1-10},
  doi={10.1109/DASC43569.2019.9081758}}

@article{BAIER2020104504,
title = {Parametric Markov chains: PCTL complexity and fraction-free Gaussian elimination},
journal = {Information and Computation},
volume = {272},
pages = {104504},
year = {2020},
note = {GandALF 2017},
issn = {0890-5401},
doi = {https://doi.org/10.1016/j.ic.2019.104504},
url = {https://www.sciencedirect.com/science/article/pii/S0890540119301208},
author = {Christel Baier and Christian Hensel and Lisa Hutschenreiter and Sebastian Junges and Joost-Pieter Katoen and Joachim Klein}
}

@inproceedings{pineau2003point,
  title={Point-based value iteration: An anytime algorithm for POMDPs},
  author={Pineau, Joelle and Gordon, Geoff and Thrun, Sebastian and others},
  booktitle={Ijcai},
  volume={3},
  pages={1025--1032},
  year={2003}
}

@article{HELDMANN2016308,
title = {Site selection and traverse planning to support a lunar polar rover mission: A case study at Haworth Crater},
journal = {Acta Astronautica},
volume = {127},
pages = {308-320},
year = {2016},
issn = {0094-5765},
doi = {https://doi.org/10.1016/j.actaastro.2016.06.014},
url = {https://www.sciencedirect.com/science/article/pii/S0094576515301934},
author = {Jennifer L. Heldmann and Anthony Colaprete and Richard C. Elphic and Ben Bussey and Andrew McGovern and Ross Beyer and David Lees and Matt Deans}
}

@inproceedings{colaprete2019overview,
  title={An overview of the volatiles investigating polar exploration rover (viper) mission},
  author={Colaprete, Anthony and Andrews, Daniel and Bluethmann, William and Elphic, Richard C and Bussey, Ben and Trimble, Jay and Zacny, Kris and Captain, Janine E},
  booktitle={AGU fall meeting abstracts},
  volume={2019},
  pages={P34B--03},
  year={2019}
}

@INPROCEEDINGS{VIPERMDP,
  author={Banerjee, Somrita and Balaban, Edward and Shirley, Mark and Bradner, Kevin and Pavone, Marco},
  booktitle={2024 IEEE Aerospace Conference}, 
  title={Contingency Planning Using Bi-level Markov Decision Processes for Space Missions}, 
  year={2024},
  volume={},
  number={},
  pages={1-9},
  doi={10.1109/AERO58975.2024.10521281}}

@INPROCEEDINGS{ICAA,
  author={Kraske, Benjamin D. and Saksena, Anshu and Buczak, Anna L. and Sunberg, Zachary N.},
  booktitle={2023 IEEE International Conference on Assured Autonomy (ICAA)}, 
  title={Explanation Through Reward Model Reconciliation using POMDP Tree Search}, 
  year={2023},
  volume={},
  number={},
  pages={137-140},
  doi={10.1109/ICAA58325.2023.00027}}

@article{sharma2024risk,
  title={Risk-Aware markov decision process contingency management autonomy for uncrewed aircraft systems},
  author={Sharma, Prashin and Kraske, Benjamin and Kim, Joseph and Laouar, Zakariya and Sunberg, Zachary and Atkins, Ella},
  journal={Journal of aerospace information systems},
  volume={21},
  number={3},
  pages={234--248},
  year={2024},
  publisher={American Institute of Aeronautics and Astronautics}
}

@inproceedings{ijcai2025p958,
  title     = {A Finite-State Controller Based Offline Solver for Deterministic POMDPs},
  author    = {Schutz, Alex and You, Yang and Mattamala, Matías and Caliskanelli, Ipek and Lacerda, Bruno and Hawes, Nick},
  booktitle = {Proceedings of the Thirty-Fourth International Joint Conference on
               Artificial Intelligence, {IJCAI-25}},
  publisher = {International Joint Conferences on Artificial Intelligence Organization},
  editor    = {James Kwok},
  pages     = {8617--8625},
  year      = {2025},
  month     = {8},
  note      = {Main Track},
  doi       = {10.24963/ijcai.2025/958},
  url       = {https://doi.org/10.24963/ijcai.2025/958},
}

@Inbook{Bai2011,
author="Bai, Haoyu
and Hsu, David
and Lee, Wee Sun
and Ngo, Vien A.",
editor="Hsu, David
and Isler, Volkan
and Latombe, Jean-Claude
and Lin, Ming C.",
title="Monte Carlo Value Iteration for Continuous-State POMDPs",
bookTitle="Algorithmic Foundations of Robotics IX: Selected Contributions of the Ninth International Workshop on the Algorithmic Foundations of Robotics",
year="2011",
publisher="Springer Berlin Heidelberg",
address="Berlin, Heidelberg",
pages="175--191",
isbn="978-3-642-17452-0",
doi="10.1007/978-3-642-17452-0_11",
url="https://doi.org/10.1007/978-3-642-17452-0_11"
}

@article{ou2025sequential,
  title={Sequential decision-making under uncertainty: a robust MDPs review},
  author={Ou, Wenfan and Bi, Sheng},
  journal={Annals of Operations Research},
  volume={353},
  number={3},
  pages={1239--1285},
  year={2025},
  publisher={Springer}
}

@article{norman2017verification,
  title={Verification and control of partially observable probabilistic systems},
  author={Norman, Gethin and Parker, David and Zou, Xueyi},
  journal={Real-Time Systems},
  volume={53},
  number={3},
  pages={354--402},
  year={2017},
  publisher={Springer}
}

@article{demirci2025sensitivity,
  title={Sensitivity of Filter Kernels and Robustness Bounds to Transition and Measurement Kernel Perturbations in Partially Observable Stochastic Control},
  author={Demirci, Yunus Emre and Kara, Ali Devran and Y{\"u}ksel, Serdar},
  journal={arXiv preprint arXiv:2508.10658},
  year={2025}
}

@article{doi:10.1137/100808976,
author = {Y\"{u}ksel, Serdar and Linder, Tam\'{a}s},
title = {Optimization and Convergence of Observation Channels in Stochastic Control},
journal = {SIAM Journal on Control and Optimization},
volume = {50},
number = {2},
pages = {864-887},
year = {2012},
doi = {10.1137/100808976},
URL = { https://doi.org/10.1137/100808976},
eprint = {https://doi.org/10.1137/100808976}
}

\appendix
\newpage

\section{Proof of Theorem \ref{thm:history-node}} \label{sec:history-node}

\equivalence*

\begin{proof}
    As the value of $\pi$ on the POMDP model $\mathcal{M}$ is a constant and $\threshold$ is an input, it is sufficient to show that \cref{eq:inner} and \cref{eq:innernode} have the same objective value, referred to as worst-case value here.
    
   We first consider the finite-horizon case and then extend to the infinite-horizon case.

    \noindent \textbf{Finite Horizon}    

    Formally, we will show

    \begin{equation} \label{eq:nodeHistorySame}
        \min_{\ZH\in \prball{\ZH_0, \delta}} \: V^{\pi,d}_{\ZH}(h_0) = \min_{\Z^d \in \prball{\Z_0, \delta}, \ ...\ , \ \Z^1 \in \prball{\Z_0, \delta}} \: V^{\pi,d}_{\Z}(h_0).
    \end{equation}

    Consider an arbitrary $t$-step history, $h$, in $H(n)$. The minimum non-sticky $d$-step value for history from state $s$ is given by the following Bellman equation:

    \begin{equation}\label{eq:historyBellman}
        V^{\pi,d}_{\ZH \min}(s,h) = R(s,\alpha(n_h(h)))+ \gamma \min_{\ZH^d(h)
        \in D_{h}} \sum_{s'\in S} T(s'|s,\alpha(n_h(h))\sum_{o \in O} \ZH^d(o|\alpha(n_h(h)),s',h)V^{\pi,d-1}_{\ZH \min}(s',h\cup(\alpha(n_h(h)),o)).
    \end{equation}
      
    where $n_h(h)$ is the FSC node corresponding to $h$ and the domain of minimization is
    \begin{equation}\label{eq:zhatdomain}
        D_{h} = \left\{\prball{\ZH_0(\alpha(n_h(h)),s',h), \delta} \forall s' \in S \right\},
    \end{equation}
    and the probability ball defined by $\delta$ around $\ZH$ is defined 
     \begin{equation}
        \begin{aligned}
            \prball{\ZH_0(\alpha(n),s',h), \delta} &= \left\{\ZH'(o|\alpha(n),s',h) \; \Big| \; \left|\ZH_0(o|\alpha(n),s',h)-\ZH'(\alpha(n),s',h)\right| \leq \delta,\right.\\
            &\ZH_0(o|\alpha(n),s',h)>0 \implies \ZH'(o|\alpha(n),s',h) \geq \epsilon_p,\\
            &\ZH_0(o|\alpha(n),s',h)=0 \implies \ZH'(o|\alpha(n),s',h)=0, \; \forall o \in O, \left. \sum_{o\in O} \ZH'(o|\alpha(n),s',h) = 1  \right\}.
        \end{aligned}
    \end{equation}

    Now, consider the $d$-step minimum non-sticky value from some node $n$ and state $s$, 
    \begin{equation} \label{eq:NodeMinVal}
        V^{\pi,d}_{\Z \min}(s,n_h(h)) = R(s,\alpha(n))+ \gamma \min_{\Z^d(n) \in D_n} \sum_{s'\in S} T(s'|s,\alpha(n))\sum_{o \in O} \Z^d(o|\alpha(n),s',n)V^{\pi,d-1}_{\Z \min}(s',T_{FSC}(n, o))
    \end{equation}
    where the domain of minimization for $\Z$ is
    \begin{equation}
        D_{n} = \left\{\prball{\Z_0(\alpha(n),s',n), \delta} \forall s' \in S \right\},
    \end{equation}
    and the probability ball defined by $\delta$ around $\Z$ is
    \begin{equation}
        \begin{aligned}
        \prball{\Z_0(\alpha(n),s',n), \delta} &= \left\{\Z'(\alpha(n),s',n) \; \Big| \; \left|\Z_0(\alpha(n),s',n)-\Z'(\alpha(n),s',n)\right| \leq \delta,\right.\\
            &\Z_0(o|\alpha(n),s',n)>0 \implies \Z'(o|\alpha(n),s',n) \geq \epsilon_p,\\
            &\left.\Z_0(o|\alpha(n),s',n)=0 \implies \Z'(o|\alpha(n),s',n)=0 \; \forall o \in O ,\sum_{o\in O} \Z'(o|\alpha(n),s',n) = 1 \right\}.            
        \end{aligned}
    \end{equation}

    By induction, we will show
    \begin{equation} \label{eq:allHNhavesameval}
        V^{\pi,d}_{\ZH \min}(s,h) = V^{\pi,d}_{\Z \min}(s,n)  \quad \forall h \in H(n), \forall n \in \pi, \forall s \in S, \forall d.
    \end{equation}
    Note, we are only concerned with $h \in H(n) \; \forall n \in \pi$ as these are the only histories reachable under $\pi$ from $b_0$.
    
    For base case $d=1$
    \begin{equation}\label{eq:baseCase}
        V^{\pi,1}(s,h)_{\ZH \min}  = R(s,\alpha(n_h(h))) = R(s,\alpha(n)) = V^{\pi,1}(s,n)_{\Z \min} \; \forall h \in H(n), \forall n \in \pi, \forall s \in S,
    \end{equation}
    and \cref{eq:allHNhavesameval} holds.

    We now prove inductively that if the $d-1$-step value is the same in the history- and node-dependent cases, then the $d$-step value will also be the same in both cases, that is, if
    \begin{equation} \label{eq:inductive_assumption}
        V^{\pi,d-1}_{\ZH \min}(s,h_{t+1}) = V^{\pi,d-1}_{\Z \min}(s,n_h(h_{t+1}))  \quad \quad \forall h_{t+1}, \forall s \in S,
    \end{equation}
    then
    \begin{equation} \label{eq:inductive_result}
        V^{\pi,d}_{\ZH \min}(s,h_{t}) = V^{\pi,d}_{\Z \min}(s,n_h(h_{t}))  \quad \quad \forall h_{t}, \forall s \in S.
    \end{equation}\\
   
    We will prove this by relating a history-dependent Bellman equation to a node-dependent Bellman equation and showing that the optimization domains are identical.

    Since $n_h(h) = n$; $n_h\left(h \cup (\alpha(n_h(h)), o)\right) = T_{FSC}(n, o)$; and \cref{eq:inductive_assumption} holds, we can substitute the node-dependent value function for $d-1$ into the Bellman equation (\ref{eq:historyBellman}), yielding
    \begin{equation}\label{eq:historyMinVal}
       V^{\pi,d}_{\ZH \min}(s,h) = R(s,\alpha(n))+ \gamma \min_{\ZH^d(h) \in D_{h}} \sum_{s'\in S} T(s'|s,\alpha(n))\sum_{o \in O} \ZH^d(o|\alpha(n),s',h)V^{\pi,d-1}_{\Z \min}(s',T_{FSC}(n, o)).
    \end{equation}
    
    By construction initially, 
    \begin{equation}\label{eq:sharedObsProbsNew}
        \ZH_0(o|a,s',h)=\Z_0(o|a,s',n) =Z(o|a,s') \quad \forall \; o \in O, a \in A,s' \in S, h \in H, n \in \pi.
    \end{equation}
    Observe that $\delta$ is the same for all elements of the observation function, then given \cref{eq:sharedObsProbsNew}, we have the following corollary.
    \begin{restatable}[]{corollary}{BoundsSame}
        \label{cor:boundssame}
        Given $\delta$, the upper and lower bounds on $\ZH$, $\Z$, and $Z$ established by $\prball{\cdot,\delta}$ for a given $s,a,o$ are the same initially across any other dimensions, i.e. $h$ and $n$.
    \end{restatable}
    
    From \cref{cor:boundssame} we have that
   
    \begin{align*}
        D_h = D_n.
    \end{align*}
    
    Then, it follows that
    \begin{equation}
        \begin{aligned}
            V^{\pi,d}_{\ZH \min}(s,h) &= R(s,\alpha(n))+ \gamma \min_{\ZH^d(h) \in D_h} \sum_{s'\in S} T(s'|s,\alpha(n)\sum_{o \in O} \ZH^d(o|\alpha(n),s',h)V^{\pi,d-1}_{\Z \min}(s',T_{FSC}(n, o))\\
            &= R(s,\alpha(n))+ \gamma \min_{\Z^d(n) \in D_n} \sum_{s'\in S} T(s'|s,\alpha(n)\sum_{o \in O} \Z^d(o|\alpha(n),s',n)V^{\pi,d-1}_{\Z \min}(s',T_{FSC}(n, o))\\
            &= V^{\pi,d}_{\Z \min}(s,n(h)).
        \end{aligned}
    \end{equation}

    Thus, the inductive step (\cref{eq:inductive_assumption,eq:inductive_result}) has been proven, and, by induction, \cref{eq:allHNhavesameval} holds for all $d$.

    Then, since each history-dependent observation distribution in \cref{eq:inner} is assigned a unique value and \cref{cor:boundssame} holds,
    \begin{equation}\label{eq:min_bell_equiv}
         \min_{\ZH\in \prball{\ZH_0, \delta}} \: V^{\pi,d}_{\ZH}(h_0) = \sum_{s\in S} b_0(s)V^{\pi,d}_{\ZH \min}(s,h_0).
    \end{equation}

    Likewise since each node-dependent distribution in \cref{eq:innernode} is assigned a unique value at each $d$,
    \begin{equation} \label{eq:min_bell_equiv2}
        \min_{\Z^d \in \prball{\Z_0, \delta}, \ ...\ , \ \Z^1 \in \prball{\Z_0, \delta}} \: V^{\pi,d}_{\Z}(h_0) = \sum_{s\in S} b_0(s)V^{\pi,d}_{\Z \min}(s,n_0(h_0)),
    \end{equation} 
    therefore \cref{eq:nodeHistorySame} holds since for a shared $b_0$ the right sides of \cref{eq:min_bell_equiv} and \cref{eq:min_bell_equiv2} have the same value by \cref{eq:allHNhavesameval}.
    
    \noindent \textbf{Infinite Horizon}
    
    Formally, we will show

    \begin{equation} \label{eq:nodeHistorySameInf}
        \min_{\ZH\in \prball{\ZH_0, \delta}} \: V^{\pi,\infty}_{\ZH}(h_0) = \min_{\Z\in \prball{\Z_0, \delta}} \: V^{\pi,\infty}_{\Z}(h_0).
    \end{equation}
    
    We begin by showing that the infinite-horizon extensions of \cref{eq:historyBellman} and \cref{eq:NodeMinVal} are the same for some $s,n(h)$, that is
    \begin{equation} \label{eq:allHNhavesameval_inf}
        V^{\pi,\infty}_{\ZH \min}(s,h) = V^{\pi,\infty}_{\Z \min}(s,n)  \quad \forall h \in H(n), \forall n \in \pi, \forall s \in S.
    \end{equation}

    We will first show that $V^{\pi,d}_{\Z \min} \rightarrow V^{\pi,\infty}_{\Z \min}$ and $V^{\pi,d}_{\ZH \min} \rightarrow V^{\pi,\infty}_{\ZH \min}$ as $d \rightarrow \infty$. Given $V^{\pi,d}_{\Z \min} = V^{\pi,d}_{\ZH \min}$ as shown in the finite-horizon proof, this implies that $V^{\pi,\infty}_{\Z \min} = V^{\pi,\infty}_{\ZH \min}$.
  
    We first define the node-based worst-case backup equation $B_{n}$ and then show it is a contraction mapping. Formally,
    \begin{equation}
        \left[B_{n}V^{d}\right](s,n) = R(s,\alpha(n))+ \gamma \min_{\Z^d \in D_n} \sum_{s'\in S} T(s'|s,\alpha(n))\sum_{o \in O} \Z^d(o|\alpha(n),s',n)V^{d-1}(s',T_{FSC}(n, o)).
    \end{equation}

    We now show that $B_{n}$ is a contraction mapping which causes $V$ to converge to a fixed point. Formally, we will show for $\gamma = [0,1)$

    \begin{equation}
        \left\|B_{n}V_1-B_{n}V_2 \right\|_\infty \leq  \gamma \left\| V_1-V_2 \right\|_\infty \forall V_1,V_2.       
    \end{equation}

    For any $s,n$,
    \begin{equation}
        \begin{aligned}
            &\left| B_{n}V^d_1(s,n)-B_{n}V^d_2(s,n) \right|\\
            &= \left|  \gamma \min_{\Z^d \in D_n} \sum_{s'\in S} T(s'|s,\alpha(n))\sum_{o \in O} \Z^{d}(o|\alpha(n),s',n)V^{d-1}_1(s,n) - \gamma \min_{\Z^d \in D_n} \sum_{s'\in S} T(s'|s,\alpha(n)\sum_{o \in O} \Z^d(o|\alpha(n),s',n)V^{d-1}_2(s',n) \right|\\ 
            &\leq \gamma \max_{\Z^d \in D_n} \sum_{s'\in S} T(s'|s,\alpha(n))\sum_{o \in O} \Z^d(o|\alpha(n),s',n) \left| V^{d-1}_1(s',n) - V^{d-1}_2(s',n) \right|\\  
            &\leq \gamma \left\| V^{d-1}_1 - V^{d-1}_2 \right\|_\infty.\\  
        \end{aligned}
    \end{equation}

    Then, \cref{eq:NodeMinVal} is a contraction mapping and repeated application will lead to a unique value by the Banach fixed-point theorem and $V^{d,\pi}_{\Z \min}(s,n) \rightarrow V^{\infty,\pi}_{\Z \min}(s,n)$ as $d \rightarrow \infty$ for any $s,n$. This then implies that there is a stationary $\Z$ which is a minimizer with respect to $V^{\infty,\pi}_{\Z \min}$, that is
    \begin{equation}
        V^{\pi,\infty}_{\Z \min}(s,n(h)) = R(s,\alpha(n))+ \gamma \min_{\Z \in D_n} \sum_{s'\in S} T(s'|s,\alpha(n)\sum_{o \in O} \Z(o|\alpha(n),s',n)V^{\pi,\infty}_{\Z \min}(s',T_{FSC}(n, o)).
    \end{equation}
    
    We now define the history-based worst-case backup equation $B_h$ and then show it is a contraction mapping. Formally,
    \begin{equation}
        \left[B_{h} V^d\right](s,h) = R(s,\alpha(n_h(h)))+ \gamma \min_{\ZH^d 
        \in D_{h}} \sum_{s'\in S} T(s'|s,\alpha(n_h(h)))\sum_{o \in O} \ZH^d(o|\alpha(n_h(h)),s',h)V^{d-1}(s',h\cup(\alpha(n_h(h)),o)).
    \end{equation}
    
    We will show for $\gamma = [0,1)$

    \begin{equation}
        \left\|B_{h}V_1-B_{h}V_2 \right\|_\infty \leq  \gamma \left\| V_1-V_2 \right\|_\infty \forall V_1,V_2.
    \end{equation}

    For any $s$ and finite $h$, with $h' = h\cup(\alpha(n_h(h)),o)$, $a=\alpha(n(h))$
    \begin{equation}
        \begin{aligned}
            &\left| B_{h}V^d_1(s,h)-B_{h}V^d_2(s,h) \right|\\
            &= \left|  \gamma \min_{\Z^d \in D_h} \sum_{s'\in S} T(s'|s,a)\sum_{o \in O} \Z^{d}(o|a,s',h)V^{d-1}_1(s,h') - \gamma \min_{\Z^d \in D_h} \sum_{s'\in S} T(s'|s,a)\sum_{o \in O} \Z^d(o|a,s',h)V^{d-1}_2(s',h') \right|\\ 
            &\leq \gamma \max_{\Z^d \in D_h} \sum_{s'\in S} T(s'|s,a)\sum_{o \in O} \Z^d(o|a,s',n) \left| V^{d-1}_1(s',h') - V^{d-1}_2(s',h') \right|\\  
            &\leq \gamma \left\| V^{d-1}_1 - V^{d-1}_2 \right\|_\infty.\\  
        \end{aligned}
    \end{equation}

    Then, \cref{eq:historyMinVal} is a contraction mapping and repeated application will lead to a unique value by the Banach fixed-point theorem and $V^{\pi,d}_{\ZH \min}(s,h) \rightarrow V^{\pi,\infty}_{\ZH \min}(s,h)$ as $d \rightarrow \infty$ for any $s,h$.

    By the finite horizon proof, $V^{\pi,d}_{\ZH \min}(s,h) = V^{\pi,d}_{\Z \min}(s,n_h(h)) \; \forall s \in S, h \in H$. Then, since both $V^{\pi,d}_{\ZH \min}(s,h)$ and $V^{\pi,d}_{\Z \min}(s,n_h(h))$ converge, their limit must be the same, i.e. $V^{\pi,\infty}_{\ZH \min}(s,h) = V^{\pi,\infty}_{\Z \min}(s,n_h(h)) \; \forall s \in S, h \in H$.

    Then, since each history-dependent observation distribution in \cref{eq:inner} is assigned a unique value and \cref{cor:boundssame} holds,
    \begin{equation}\label{eq:min_bell_equiv_inf}
         \min_{\ZH\in \prball{\ZH_0, \delta}} \: V^{\pi,\infty}_{\ZH}(h_0) = \sum_{s\in S} b_0(s)V^{\pi,\infty}_{\ZH \min}(s,h_0).
    \end{equation}

    Likewise,
    \begin{equation}
        \min_{\Z \in \prball{\Z_0, \delta}} \: V^{\pi,\infty}_{\Z}(h_0) = \sum_{s\in S} b_0(s)V^{\pi,\infty}_{\Z \min}(s,n_0(h_0)),
    \end{equation} 
    therefore \cref{eq:nodeHistorySameInf} holds.
    
\end{proof}

\section{Proof of Proposition \ref{lm:bisect_sound}}\label{sec:bisect_sound}
\bisectsound*

\begin{proof}
    Modified bisection search is initialized with $\delta_a = 0$ and $\delta_b = 1$. 
    In the case where for initial $\delta_a,\delta_b$ $sign(f(\delta_a)) \neq sign(f(\delta_b))$, \modifiedBisectionShort{} will always return $\delta_a$. By construction, $\delta_a$ will only be assigned $\delta_c$ such that $sign(f(\delta_c)) = sign(f(\delta_a))$ or $f(\delta_c) = 0$. Then, $f(\delta_a)$ at any iteration cannot be greater than 0.  In the case where for initial $\delta_a,\delta_b$ $sign(f(\delta_a)) = sign(f(\delta_b))$ and $f(\delta_b)$ is returned, $f(\delta_a)<0 \implies f(\delta_b)<0$.
\end{proof}

\section{Proof of Proposition \ref{lm:bisectconverge} }\label{sec:bisect_converge}
\bisectconverge*

\begin{proof}
    By definition, $f(\delta^*) \leq 0$. This gives two possible cases for initial $f(\delta_a)$ and $f(\delta_b)$:
    \begin{itemize}
        \item  $sign(f(\delta_a)) \neq sign(f(\delta_b))$: By construction, $\delta_a \in \{\delta | f(\delta)\leq 0\}$. Then, by definition, $\delta_a \leq \delta^*$. There are two cases for $f(\delta_b)$:
    \begin{itemize}
        \item $f(\delta_b) > 0$: Since $f(\delta_b) > 0$ and $f(\delta_*) \leq 0$, $f(\delta_*) < f(\delta_b)$. Since $f$ is monotone non-decreasing, $f(\delta_*) < f(\delta_b) \implies \delta^*<\delta_b$. Thus, $\delta_a \leq \delta^* < \delta_b$. Since the interval will half every iteration, $\exists N \text{ st } \delta_b-\delta_a \leq \epsilon_{\modifiedBisectionShort{}}= \frac{1}{2^{N-1}}  \implies \delta^*-\delta_a < \epsilon_{\modifiedBisectionShort{}}$.
        \item $f(\delta_b) = 0$: Then, trivially $f(\delta_*) = f(\delta_b) = 0$ and only $\delta_a$ will change as $f(\delta_c) \leq 0$ by monotonicity of $f$. Thus, $\delta_a \leq \delta^* = \delta_b=1$. Since the interval will half every iteration, $\exists N \text{ st } \delta_b-\delta_a \leq \epsilon_{\modifiedBisectionShort{}}= \frac{1}{2^{N-1}} \implies \delta^*-\delta_a < \epsilon_{\modifiedBisectionShort{}}$.
    \end{itemize}
    \item $sign(f(\delta_a)) = sign(f(\delta_b))$: By assumption $f(\delta_a)\leq0$, since the signs are the same, $f(\delta_b)\leq0$. By assumption $f$ is monotone non-decreasing in $\delta$, i.e., $f(\delta_a) \leq f(\delta_b)$. Then, $f(\delta_b)$ is an upper bound on $f$ on the interval $[0,1]$ and since $f(\delta_b)\leq0$, $\delta_b=1$ is the largest admissible $\delta$ on $[0,1]$, i.e., $\delta^*=\delta_b=1$. The algorithm terminates immediately as defined on lines 2-4.
\end{itemize}
\end{proof}

\section{Case Study \rev{POMDP Definitions}}\label{sec:probdefs}

\subsection{Toy Rover}
The Toy Rover POMDP is a tuple $(S,A,O,T,Z,R,\gamma,b_0)$ where $S=\{\text{terminal} , \{ \text{large},\text{small}\} \times \{\text{smooth},\text{angular}\} \}$, $A=\{\text{measure size},\text{measure texture},\text{go through},\text{go around}\}$, $O= \{\text{true},\text{false} \}$,
\begin{equation*} T(s'|s,a)=
    \begin{aligned}
        \begin{cases} 
            1.0  &\text{ if } s' = s,  a \notin \{\text{go through},\text{go around}\}\\
            1.0 &\text{ if } s' = \text{terminal},  a \in \{\text{go through},\text{go around}\}\\
            1.0 &\text{ if } s'=\text{terminal}, s=\text{terminal}\\
            0.0 &\text{ o.w.}
        \end{cases},
    \end{aligned}
\end{equation*}
\begin{equation*} Z(o|a,s') = 
    \begin{aligned}
    \begin{cases}
         0.99 &\text{ if } o=\text{true}, s'=\text{large}, a=\text{measure size}\\
         0.01 &\text{ if } o=\text{false}, s'=\text{large}, a=\text{measure size}\\
         0.01 &\text{ if } o=\text{true}, s'=\text{small}, a=\text{measure size}\\
         0.99 &\text{ if } o=\text{false}, s'=\text{small}, a=\text{measure size}\\
         0.99 &\text{ if } o=\text{true}, s'=\text{smooth}, a=\text{measure texture}\\
         0.01 &\text{ if } o=\text{false}, s'=\text{smooth}, a=\text{measure texture}\\
         0.01 &\text{ if } o=\text{true}, s'=\text{angular}, a=\text{measure texture}\\
         0.99 &\text{ if } o=\text{false}, s'=\text{angular}, a=\text{measure texture}\\
        \end{cases},
    \end{aligned}
    \end{equation*}
\begin{equation*} R = 
    \begin{aligned}
    \begin{cases}
            0.0 &\text{ if } s=\text{terminal}\\
            1.0 &\text{ if } s\in \{\text{(large, smooth)},\text{(small, angular)}\},a=\text{go through}\\
            0.9 &\text{ if }  s\neq \text{terminal},a=\text{go around}\\
            0.0 &\text{ o.w.}
    \end{cases},
    \end{aligned}
\end{equation*}
$\gamma=0.99$, and $b_0 = \mathcal{U}(\{ \text{large},\text{small}\} \times \{\text{smooth},\text{angular}\})$.

\subsection{Rover Navigation}

The Rover POMDP is a tuple $(S,A,O,T,Z,R,\gamma,b_0)$ where $S=\{\text{terminal} , \{ \text{large},\text{small}\} \times \{\text{smooth},\text{angular}\} \times \{1..3\} \times \{1..5\}\} \}$, $A=\{\text{measure size},\text{measure texture},\text{up},\text{down},\text{right},\text{left}\}$, $O= \{\text{true},\text{false} \}$,

\begin{equation*} T((s'.sand,s'.pos)|(s.sand,s.pos),a)=
    \begin{aligned}
    \begin{cases} 
            1.0  &\text{ if } s'.sand=s.sand,  s'.pos=s.pos+dir(a,s) , s.pos \notin \{(1,3),(1,4)\}\\
            1.0  &\text{ if } s'.sand=s.sand,  s'.pos=s.pos+dir(a,s) , s.pos \in \{(1,3),(1,4)\}, \\
            &\quad \; s.sand \in \{\text{(large, smooth)},\text{(small, angular)}\} \\
            1.0  &\text{ if } s'=\text{terminal} , s.pos \in \{(1,3),(1,4)\},\\
            &\quad \;  s.sand \notin \{\text{(large, smooth)},\text{(small, angular)}\} \\
            1.0  &\text{ if } s'=\text{terminal} , s.pos = (1,5) \\
            1.0 &\text{ if } s'=\text{terminal}, s=\text{terminal}\\
            0.0 &\text{ o.w.}
    \end{cases},
    \end{aligned}
\end{equation*}
    
where 
\begin{equation*}
	dir(a,s) = 
\begin{aligned}
\begin{cases}
        (0,0) &\text{ if } a \in \{\text{measure size}, \text{measure texture}\}\\
        (0,1) &\text{ if } a=\text{up}, s.pos+(0,1) \in  \{1..3\}, \times \{1..5\}\\ 
        (0,-1) &\text{ if } a=\text{down},s.pos+(0,-1) \in  \{1..3\}, \times \{1..5\}\\ 
        (-1,0) &\text{ if } a=\text{left},s.pos+(-1,0) \in  \{1..3\}, \times \{1..5\}\\ 
        (1,0) &\text{ if } a=\text{right},s.pos+(1,0) \in  \{1..3\}, \times \{1..5\}\\ 
\end{cases},
\end{aligned}
\end{equation*}

\begin{equation*} Z(o|a,s') = 
    \begin{aligned}
        \begin{cases}
             0.99 &\text{ if } o=\text{true}, s'.sand=\text{large}, a=\text{measure size}\\
             0.01 &\text{ if } o=\text{false}, s'.sand=\text{large}, a=\text{measure size}\\
             0.01 &\text{ if } o=\text{true}, s'.sand=\text{small}, a=\text{measure size}\\
             0.99 &\text{ if } o=\text{false}, s'.sand=\text{small}, a=\text{measure size}\\
             0.99 &\text{ if } o=\text{true}, s'.sand=\text{smooth}, a=\text{measure texture}\\
             0.01 &\text{ if } o=\text{false}, s'.sand=\text{smooth}, a=\text{measure texture}\\
             0.01 &\text{ if } o=\text{true}, s'.sand=\text{angular}, a=\text{measure texture}\\
             0.99 &\text{ if } o=\text{false}, s'.sand=\text{angular}, a=\text{measure texture}\\
        \end{cases},
    \end{aligned}
\end{equation*}

\begin{equation*} R = 
    \begin{aligned}
    \begin{cases}
        0.0 &\text{ if } s=\text{terminal}\\
        1.0 &\text{ if } s=(3,5)\\
        -0.1 &\text{ if } a=\text{left}\\
        0.0 &\text{  o.w.}
    \end{cases},
    \end{aligned}
\end{equation*}

$\gamma=0.99$, and $b_0 = \mathcal{U}(\{(3,1)\}\times \{\text{smooth},\text{angular}\}\times \{\text{large},\text{small}\})$.

\subsection{Cancer Diagnosis}

The Cancer Diagnosis POMDP is a tuple $(S,A,O,T,Z,R,\gamma,b_0)$ where $S=\{\text{healthy},\text{in-situ}, \text{invasive},\text{death} \}$, $A=\{\text{wait},\text{test},\text{treat}\}$, $O= \{\text{positive}, \text{negative},\text{dead} \}$,

\begin{equation*} T(s'|s,a)=
    \begin{aligned}
    \begin{cases} 
        0.02  &\text{ if } s'= \text{in-situ}, s=\text{healthy} \\
        0.98  &\text{ if } s'= \text{healthy}, s=\text{healthy} \\
        0.1  &\text{ if } s'= \text{invasive}, s=\text{in-situ}, a\in
        \{\text{wait},\text{test}\}\\
        0.9  &\text{ if } s'= \text{in-situ}, s=\text{in-situ}, a \in\{\text{wait},\text{test}\}\\
        0.6  &\text{ if } s'= \text{death}, s=\text{invasive},  a\in\{\text{wait},\text{test}\}\\
        0.4  &\text{ if } s'= \text{invasive}, s=\text{invasive},  a\in\{\text{wait},\text{test}\}\\
        0.4  &\text{ if } s'= \text{in-situ}, s=\text{in-situ}, a =\text{treat}\\
        0.6  &\text{ if } s'= \text{healthy}, s=\text{in-situ}, a =\text{treat}\\
        
        0.2  &\text{ if } s'= \text{healthy}, s=\text{invasive}, a =\text{treat}\\
        0.2  &\text{ if } s'= \text{death}, s=\text{invasive}, a =\text{treat}\\
        0.6  &\text{ if } s'= \text{invasive}, s=\text{invasive}, a =\text{treat}\\
        1.0 &\text{ if } s'=\text{death}, s=\text{death}\\
         0.0 &\text{ o.w.}
    \end{cases},
    \end{aligned}
\end{equation*}

\begin{equation*} Z(o|a,s') = 
    \begin{aligned}
    \begin{cases}
        1.0  &\text{ if }  o=\text{negative}, a=\text{wait}\\
        
         0.05 &\text{ if } o=\text{positive}, a=\text{test}, s=\text{healthy} \\
         0.95 &\text{ if } o=\text{negative}, a=\text{test}, s=\text{healthy} \\
         0.8 &\text{ if } o=\text{positive}, a=\text{test}, s=\text{in-situ} \\
         0.2 &\text{ if } o=\text{negative}, a=\text{test}, s=\text{in-situ} \\
         1.0 &\text{ if } o=\text{positive}, a=\text{test}, s=\text{invasive} \\
        
        1.0  &\text{ if }  o=\text{negative}, a=\text{treat},s=\text{healthy}\\
        1.0  &\text{ if }  o=\text{positive}, a=\text{treat},s\in\{\text{in-situ},\text{invasive}\} \\	
        1.0 &\text{ if } o=\text{dead}, s=\text{dead}\\
         0.0 &\text{ o.w.}
    \end{cases},
    \end{aligned}
\end{equation*}

\begin{equation*} R =
    \begin{aligned}
    \begin{cases}
         0.0 &\text{ if } s = \text{dead}\\
         1.0 &\text{ if } a = \text{wait}\\
         0.8 &\text{ if } a = \text{test}\\
         0.1 &\text{ if } a = \text{treat}\\
        0.0 &\text{  o.w.}
    \end{cases},
    \end{aligned}
    \end{equation*}
$\gamma=0.999$, and $b_0 = 
\begin{aligned}
\begin{cases}
    1.0 &\text{ if }s=\text{healthy}\\
    0.0 &\text{ o.w.}
\end{cases}
\end{aligned}$.

\subsection{Part Quality Control}
The Part Quality Control POMDP is a tuple $(S,A,O,T,Z,R,\gamma,b_0)$ where $S=\{\text{terminal},\text{passing},\text{failing}\}$, \newline
$A=\{\text{measure},\text{accept},\text{reject}\}$, $O= \{\text{pass}, \text{fail} \}$,

\begin{equation*} T(s'|s,a)=
    \begin{aligned}
    \begin{cases} 
         1.0  &\text{ if } a=\text{measure}, s'= s\\
         1.0 &\text{ if } a\in\{\text{accept},\text{reject}\}, s'=\text{terminal}\\
         0.0 &\text{ o.w.}
    \end{cases},
    \end{aligned}
\end{equation*}

\begin{equation*} Z(o|a,s') = 
    \begin{aligned}
    \begin{cases}
        ac  &\text{ if }  o=\text{pass}, s=\text{passing}, a=\text{measure}\\
        1-ac &\text{ if }  o=\text{fail}, s=\text{passing}, a=\text{measure}\\
        ac  &\text{ if }  o=\text{fail}, s=\text{failing}, a=\text{measure}\\
        1-ac  &\text{ if }  o=\text{pass}, s=\text{failing}, a=\text{measure}\\
         0.0 &\text{ o.w.}
    \end{cases},
    \end{aligned}
\end{equation*}
where $ac$ is the accuracy of the inspection,

\begin{equation*} R = 
    \begin{aligned}
    \begin{cases}
         -1.0 &\text{ if } a = \text{accept}, s=\text{failing}\\
        0.0 &\text{ o.w.}
    \end{cases},
    \end{aligned}
\end{equation*}

$\gamma=1.0$, and $b_0 = \mathcal{U}(\{\text{passing},\text{failing}\})$.

\end{document}